\documentclass[11.5pt]{article}
\usepackage{appendix}
\usepackage{jmlr2e}
\usepackage{lastpage}
\ShortHeadings{Small Transformers Compute Universal Metric Embeddings}{Anastasis Kratsios, Valentin Debarnot, Ivan Dokmani\'{c}}

\usepackage{silence,etoolbox}
\makeatletter
\patchcmd{\wrong@fontshape}{\@gobbletwo}{}{}{}
\makeatother
\usepackage[table]{xcolor}
\usepackage{natbib}
\usepackage{hyperref}
\hypersetup{
	colorlinks = true,
	linkcolor = blue,
	anchorcolor = blue,
	citecolor = blue,
	filecolor = blue,
	urlcolor = blue
}

\DeclareMathAlphabet{\mathsfit}{\encodingdefault}{\sfdefault}{m}{sl}
\SetMathAlphabet{\mathsfit}{bold}{\encodingdefault}{\sfdefault}{bx}{n}

\newcommand{\softmax}{\mathrm{\bm{\sigma}}}

\usepackage{soul}
\usepackage{float}
\usepackage{graphicx}
\usepackage{booktabs}
\usepackage{nicefrac}
\usepackage{times}
\usepackage{enumitem}

\usepackage{footnote}
\makesavenoteenv{tabular}
\makesavenoteenv{table}

\usepackage{etoolbox}

\AfterEndEnvironment{equation}{\reseteqhint}
\AfterEndEnvironment{equation*}{\reseteqhint}
\AfterEndEnvironment{align}{\reseteqhint}
\AfterEndEnvironment{align*}{\reseteqhint}
\newcounter{hints}
\renewcommand{\thehints}{\alph{hints}}
\newcommand{\eqhint}[2][]{%
  \stepcounter{hints}%
  \if\relax\detokenize{#1}\relax\else\csxdef{hint@#1}{\thehints}\fi
  \mathrel{\overset{\textrm{(\thehints\hspace{0.01em})}}{\vphantom{\le}{#2}}}%
}
\newcommand{\reseteqhint}{\setcounter{hints}{0}}

\usepackage{graphicx}

\usepackage[utf8]{inputenc}
\usepackage{appendix}

\usepackage{mathtools}
\usepackage{bbm}
\usepackage{booktabs}
\usepackage{xparse}
\usepackage{sidecap}
\usepackage{float}
\usepackage{mdwlist}
\usepackage[bbgreekl]{mathbbol}
\usepackage{amsfonts}
\usepackage{amsmath}
\usepackage{amssymb}
\usepackage{amscd}
\usepackage{mathrsfs}
\usepackage{bm}
\usepackage{latexsym}
\usepackage{mleftright}
\usepackage{color}
\usepackage{xifthen}
\usepackage{makeidx}

\usepackage{mathrsfs}

\newcommand{\rr}{{\mathbb{R}}}
\newcommand{\pp}{{\mathbb{P}}}
\newcommand{\qq}{{\mathbb{Q}}}
\newcommand{\rrflex}[1]{{\ensuremath{\rr^{#1}
}}}

\newcommand{\rrd}{{\rrflex{d}}}

\newcommand{\xxx}{\mathcal{X}}

\newcommand{\nn}{{\mathbb{N}}}

\NewDocumentCommand{\mw}{}{\mathcal{MW}}
\NewDocumentCommand{\ppp}{mo}{
                \ensuremath{
                            \mathcal{P}\IfValueT{#2}{_{{#2}}}\IfValueF{#2}{_1}
                            \left(
                                {#1}
                            \right)
                }
                    }

\newcommand{\zzz}{{\mathcal{Z}}}

\NewDocumentCommand{\NN}{oo}{
    \ensuremath{
        \mathcal{NN}
        \IfValueT{#1}{_{#1}}\IfValueF{#1}{_{d:D}}
        \IfValueT{#2}{^{#2}}\IfValueF{#2}{^{\sigma}}
    }
}

\usepackage{marginnote}
\definecolor{WowColor}{rgb}{.75,0,.75}
\definecolor{MildlyAlarming}{rgb}{0.85,0.25,0.1}
\definecolor{SubtleColor}{rgb}{0,0,.50}
\newcounter{margincounter}

\usepackage[normalem]{ulem}
\usepackage{wasysym}

\definecolor{darkcerulean}{rgb}{0.03, 0.27, 0.49}
\definecolor{darkmidnightblue}{rgb}{0.0, 0.2, 0.4}
\definecolor{darkcyan}{rgb}{0.0, 0.55, 0.55}
\definecolor{darkgreen}{rgb}{0.0, 0.2, 0.13}
\definecolor{deepjunglegreen}{rgb}{0.0, 0.29, 0.29}
\definecolor{darkcandyapplered}{rgb}{0.64, 0.0, 0.0}
\definecolor{darkred}{rgb}{0.55, 0.0, 0.0}
\definecolor{darkscarlet}{rgb}{0.34, 0.01, 0.1}
\definecolor{jasper}{rgb}{0.84, 0.23, 0.24}
\definecolor{darkjazzberryjam}{rgb}{0.45, 0.04, 0.37}
\definecolor{MidnightBlue}{RGB}{25,25,112}
\definecolor{ValBlue}{RGB}{51, 155, 200}
\definecolor{ValIIBlue}{RGB}{29, 88, 114}
\definecolor{MidnightBlueComplementingGreen}{RGB}{25,112,25}
\definecolor{MidnightBlueComplementingPurple}{RGB}{112,25,112}
\definecolor{MidnightBlueComplementingRed}{RGB}{112,25,69}

\definecolor{MidnightBlue}{RGB}{25,25,112}
\definecolor{MidnightBlueComplementingGreen}{RGB}{25,112,25}
\definecolor{MidnightBlueComplementingPurple}{RGB}{112,25,112}
\definecolor{MidnightBlueComplementingRed}{RGB}{112,25,69}
\definecolor{WowColor}{rgb}{.75,0,.75}
\definecolor{MildlyAlarming}{rgb}{0.85,0.25,0.1}
\definecolor{SubtleColor}{rgb}{0,0,.50}
\definecolor{antiquefuchsia}{rgb}{0.57, 0.36, 0.51}
\definecolor{fashionfuchsia}{rgb}{0.96, 0.0, 0.63}
\definecolor{jade}{RGB}{0,122,84}
\definecolor{darkjade}{RGB}{0, 98, 67}
\definecolor{caribbeangreen}{rgb}{0.0, 0.8, 0.6}
\definecolor{aquamarine}{rgb}{0.5, 0.8, 0.85}

\NewDocumentCommand{\Val}{mo}{
    \IfValueF{#2}{
                        {{\scriptsize
                            \textcolor{ValBlue}{ 
                            $\boldsymbol{\mathbb{V}}$\textbf{: }
                            \textit{{#1}}
                            }
                        }}
        }
    \IfValueT{#2}{
                        \marginnote{{\scriptsize
                            \textcolor{ValIIBlue}{ 
                            $\boldsymbol{\mathbb{V}}$\textbf{: }
                            \textit{{#1}}
                            }
                        }}
        }
                    }
\NewDocumentCommand{\Ivan}{mo}{
    \IfValueF{#2}{
                        {{\scriptsize
                            \textcolor{darkjazzberryjam}{ 
                            \textbf{$\mathcal{I}$:}
                            \textit{{#1}}
                            }
                        }}
        }
    \IfValueT{#2}{
                        \marginnote{{\scriptsize
                            \textcolor{darkjazzberryjam}{ 
                            \textbf{$\mathcal{I}$:}
                            \textit{{#1}}
                            }
                        }}
        }
                    }
\NewDocumentCommand{\Anastasis}{mo}{
    \IfValueF{#2}{
                        {{
                            \textcolor{jade}{ 
                            {$\boldsymbol{\mathfrak{A}}$:}
                            \textit{{#1}}
                            }
                        }}
        }
    \IfValueT{#2}{
                        \marginnote{{\scriptsize
                            \textcolor{darkjade}{ 
                           \textbf{{$\boldsymbol{\mathfrak{A}}$:} - Note:}
                            \textit{{#1}}
                            }
                        }}
        }
                    }

\newcommand{\eqdef}{\ensuremath{\stackrel{\mbox{\upshape\tiny def.}}{=}}}
\def\R{\mathbb{R}}

\def\M{\mathcal{M}}

\usepackage{wrapfig}

\usepackage{algorithm}
\usepackage{algpseudocode}

\definecolor{Window1}{RGB}{92,150,31}%
    \definecolor{Window1dark}{RGB}{41,67,13}%
\definecolor{Window2}{RGB}{255,168,28}
    \definecolor{Window2dark}{RGB}{114,75,12}
\definecolor{Window3}{RGB}{255,96,33}
    \definecolor{Window3dark}{RGB}{97,36,12}
\definecolor{InputColor}{RGB}{20,255,177}
    \definecolor{InputColorlight}{RGB}{222,237,229}
    
\definecolor{MidnightBlue}{RGB}{25,25,112}
\definecolor{MidnightBlueComplementingGreen}{RGB}{25,112,25}
\definecolor{MidnightBlueComplementingPurple}{RGB}{112,25,112}
\definecolor{MidnightBlueComplementingRed}{RGB}{112,25,69}
\definecolor{WowColor}{rgb}{.75,0,.75}
\definecolor{MildlyAlarming}{rgb}{0.85,0.25,0.1}
\definecolor{SubtleColor}{rgb}{0,0,.50}
\definecolor{Fuchsia}{RGB}{178, 18, 178}
\definecolor{antiquefuchsia}{rgb}{0.57, 0.36, 0.51}
\definecolor{fashionfuchsia}{rgb}{0.96, 0.0, 0.63}
\definecolor{jade}{rgb}{0.0, 0.66, 0.42}
\definecolor{caribbeangreen}{rgb}{0.0, 0.8, 0.6}
\definecolor{aquamarine}{rgb}{0.5, 0.8, 0.85}
\definecolor{attentioncolor}{RGB}{152,90,81}
\definecolor{burgred}{RGB}{40,3,22}
\definecolor{AnnieGreen}{RGB}{17,123,92}
\definecolor{Turquoise}{RGB}{64,224,208}
\definecolor{darkjade}{RGB}{0,122,84}
\definecolor{landmarkgreen}{RGB}{6,59,6}

\definecolor{mycolor1}{rgb}{0.00000,0.44700,0.74100}%
\definecolor{mycolor2}{rgb}{0.85000,0.32500,0.09800}%

\usepackage{subcaption}
\usepackage{tabularx}
\usepackage{multirow}

\usepackage{cleveref}
\newcommand{\ra}[1]{\renewcommand{\arraystretch}{#1}}
\usepackage{lscape}

\begin{document}

\title{Small Transformers Compute Universal Metric Embeddings}




	\author{\name Anastasis Kratsios 
	\email kratsioa@mcmaster.ca\\
	\addr McMaster University\\
	Department of Mathematics\\
    1280 Main Street West, Hamilton, Ontario, L8S 4K1, Canada
	\AND
	\name Valentin Debarnot
	\email valentin.debarnot@unibas.ch\\
    \addr Universit\"{a}t Basel\\
    Department of Computer Science\\
    Basel, 4051, Switzerland
    \AND
    \name Ivan Dokmani\'{c}
	\email ivan.dokmanic@unibas.ch\\
    \addr Universit\"{a}t Basel\\
    Department of Computer Science\\
    Basel, 4051, Switzerland
	}
	
	\editor{}
	
\maketitle

\begin{abstract}
We study representations of data from an arbitrary metric space $\mathcal{X}$ in the space of univariate Gaussian mixtures with a transport metric (Delon and Desolneux 2020).  We derive embedding guarantees for feature maps implemented by small recent neural networks called \emph{probabilistic transformers}.  Our guarantees are of memorization type: we prove that a probabilistic transformer of depth about $n\log(n)$ and width about $n^2$ can bi-H\"older embed any $n$-point dataset from $\mathcal{X}$ with low metric distortion, thus avoiding the curse of dimensionality.  We further derive probabilistic bi-Lipschitz guarantees, which trade off the amount of distortion and the probability that a randomly chosen pair of points embeds with that distortion.  If $\mathcal{X}$'s geometry is sufficiently regular, we obtain stronger bi-Lipschitz guarantees for all points in the dataset.  As applications, we derive neural embedding guarantees for datasets from Riemannian manifolds, metric trees, and certain types of combinatorial graphs.  When instead embedding into multivariate Gaussian mixtures, we show that probabilistic transformers can compute bi-H\"{o}lder embeddings with arbitrarily small distortion.  
\end{abstract}
\hfill\\




\section{Introduction}
	In machine learning practice we face a daunting variety of data sources. All of them, however, come equipped with some domain-specific notion of (dis)similarity. The successes of modern machine learning models such as deep neural networks may be in part attributed to the fact that they (explicitly or implicitly) learn representations of data that are compatible with these domain-specific notions. A key task in a successful machine learning pipeline is thus to identify a \textit{representation space} $\mathcal{R}$ and a \textit{feature map} $\phi:\mathcal{X}\rightarrow \mathcal{R}$ which faithfully encodes $\mathcal{X}$. It is often the case that the data space $\mathcal{X}$ and the dissimilarity $d_{\mathcal{X}}$ form a metric space $(\mathcal{X}, d_{\mathcal{X}})$; ``faithfully'' then means with minimal distortion to $d_\mathcal{X}$. Once data is represented in $\mathcal{R}$, we can learn and predict from $\phi(\mathcal{X})$ using downstream classification and regression models.

	For a long time, the default choice for $\mathcal{R}$ has been a Euclidean space or a (reproducing kernel) Hilbert space. Yet there is now a body of work both in machine learning \cite{JostNature1,JostNature2,muscoloni2017machine,ganea2018hyperbolic,digiovanni2022heterogeneous} and in metric embedding theory \cite{Bourgain_1985_HilbertSpaceisaNogo,Matouvsek_1996_Embeddings,naor2006markov} showing that Euclidean and Hilbert spaces are suboptimal for many practically-relevant notions of dissimilarity. This is because Euclidean spaces are ``too small'' and Hilbert spaces ``too flat'' to represent the great variety of geometries in modern datasets.  
	
    An archetypal situation where low-dimensional non-Euclidean representations outperform high-dimensional Euclidean is when working with hierarchical data, organized on trees or graphs with power-law degree distributions \cite{ravasz2003hierarchical}, and equipped with the usual combinatorial (geodesic) metric. Any finite tree $\xxx$ can be efficiently embedded into the two-dimensional \textit{hyperbolic space} \citep{sarkar2011low} while embeddings in $d$-dimensional Euclidean space require much greater distortion \cite{Gupta_Trees_Quantitiativefinitedimembeddings_ACM}. Since hierarchical data is ubiquitous, this favorable property of hyperbolic spaces has inspired a booming development of ``hyperbolic'' machine learning models. Examples include clustering \citep{chami2020trees,tabaghi2021procrustes,tabaghi2020hyperbolic}, PCA-type methods \citep{chami2021horopca}, classification \cite{lopez2020fully}, hyperbolic analogues of feedforward networks \citep{ganea2018hyperbolic,shimizu2021hyperbolic}, and several Python packages supporting this Riemannian geometry \cite{GEOMSTATS_JMLR_2020,kochurov2020geoopt}. These advances contributed to state-of-the-art performance when learning from natural language \cite{Zipf1949,dhingra2018embedding,le2019inferring}, knowledge graphs \cite{kolyvakis2020hyperbolic}, social networks \cite{krioukov2010hyperbolic,muscoloni2017machine}, directed graphs \cite{munzner1997h3}, scenario generation for stochastic phenomena \cite{pflug2015dynamic}, and combinatorial trees \cite{nickel2017poincare}.
    	
	Hyperbolic representations perform well when $\xxx$ possesses a ``tree-like'' geometric prior, but many datasets lack such structure or exhibit a more complex one. This has motivated the search for representations of combinatorial graphs in  Riemannian manifolds beyond constant-curvature space forms, including products of space forms (mixed-curvature spaces) \cite{tabaghi2021linear,JostNature1,JostNature2} and representations that match discrete curvature in heterogeneous rotationally-symmetric manifolds \cite{digiovanni2022heterogeneous}. While these proposals show great promise in addressing the emergent shortcomings of hyperbolic spaces, each of them introduces additional parameters (such as the choice of space forms or their dimensions and curvatures) without obvious nominal values on real-world datasets, and each requires a different set of downstream tools.

    We are thus motivated to look for a ``universal'' representation space. Beyond the compelling geometric considerations in the previous paragraphs, even when we work with data from a nice space $\mathcal{X}$, there may exist no machine learning models on $\mathcal{X}$ with favourable approximation, generalization, and optimization properties. When good machine learning models \emph{do} exist, it often happens that we work with datasets which belong to some \emph{unknown} low-dimensional subset of $\mathcal{X}$ and that this potentially advantageous latent structure is not leveraged by the off-the-shelf models. It is then opportune to design representations that allow us to optimally relay $\xxx$'s structure to downstream tasks. An important practical advantage of a single canonical representation space $\mathcal{R}$ is that it would admit a single ``standardized'' set of machine learning tools.

	\subsection{``Good'' Representation Spaces and Learnable Feature Maps, with Guarantees}
	\label{sec:good-representations}

	We consider the problem of \textit{representing} a general metric space $(\xxx, d_\mathcal{X})$ in a representation space $\mathcal{R}$, using a \textit{feature map} $\phi:\xxx\rightarrow \mathcal{R}$. We are interested in ``good'' representation spaces and computationally tractable feature maps. 

	Let us take a moment to reflect on what a ``good''  $\mathcal{R}$ and $\phi$ should entail.  Empirically successful non-Euclidean embeddings \cite{LopezPozzettiTrettelStrubeWienhard_2021__Symmetric_Spaces_for_Graph_Embeddings,digiovanni2022heterogeneous} and Bourgain's impossibility result for metric embeddings in Hilbert spaces \cite{Bourgain_1985_HilbertSpaceisaNogo} suggest that $\mathcal{R}$ must not be \emph{flat}. On the other hand, we show in Section~\ref{s_Main_ss_Supporting_Results} that for any complete Riemannian manifold $\mathcal{R}$ (including all manifolds of positive curvature) and for any positive integer $n$, there exists an $n$-point metric space which cannot be embedded in $\mathcal{R}$ with metric distortion smaller than $O(\log n)$. Finite-dimensional manifolds are thus \emph{too small} to embed datasets from arbitrary metric spaces, so we require that $\mathcal{R}$ be infinite-dimensional. 

	Given a finite dataset from $\mathcal{X}$, its representation in $\mathcal{R}$ is often computed by solving an optimization problem. One example is Laplacian eigenmaps when $\mathcal{X}$ is a Riemannian manifold and $d_\mathcal{X}$ the geodesic distance \cite{belkin2003laplacian,belkin2006convergence}. The feature map $\phi$ is thus defined implicitly via an optimization problem. A drawback of this strategy is that if a new sample is added to the dataset, computing its representation may require recomputing representations of all points. To avoid this, we are interested in efficient direct parameterizations of the feature map by deep neural networks with controlled complexity. In machine learning, this ``amortization" of static optimization-based procedures is referred to as (out-of-sample) generalization \cite{bartlett2017spectrally,AgarwalNiyogi_2009_GeneralizationRankingAlgosJMLR,BartlettFailuresGeneralizationJMLR2021,MeiTheodorMontanari_2022_Generalization_RFMs,MeiMontanari2022}. To make it possible, $\mathcal{R}$ should be a familiar representation space with well-understood geometry, finitely parameterized points, and admitting efficient numerical algorithms. Since only finite-dimensional representations can be practically implemented, we ask that the feature map $\phi:\xxx\rightarrow \mathcal{R}$ implemented by our deep learning model take values in a finite-dimensional subspace of $\mathcal{R}$ whenever $\xxx$ has a suitable geometric prior (which we make precise shortly). We informally refer to the dimension of the image of $\phi$ as the \textit{effective dimension} of the representation $\phi$ of $\xxx$.

	We show that the requirements for $\mathcal{R}$ are met by a certain space of probability measures equipped with an optimal transport metric, and that the corresponding feature map $\phi$ can be \textit{exactly implemented} (in the sense of memorization capacity) by a deep neural network with controlled complexity. Concretely, we prove that a deep neural network with about $n\log(n)$ layers of width $n^2$ can embed any dataset of $n$ points in $\mathcal{X}$ with low metric distortion.  
	Thus, unlike neural networks studied by most universal approximation theorems \cite{BurgerNeubauer_2001,YAROTSKYSobolev,kratsios2021_GDL,acciaio2022metric,kidger2019universal,puthawala2020globally,ShenYangZhang2021}, the PT does not face the curse of dimensionality. Its number of parameters is a (small-degree) polynomial in the embedded number of points which is independent of $\xxx$'s dimension. In contrast, the networks constructed in universal approximation theorems usually require a number of parameters exponential in $\xxx$'s dimension.  

	The particular $n$-point dataset (a subset $\xxx_n$ of $\xxx$) encodes information about general points in $\xxx$. Its points act as reference points or \emph{landmarks} for the remainder of $\xxx$. If $\xxx$ is a polytope (for example, $[0,1]^d$) then $\xxx_n$ could be the set of extremal points in $\xxx$ (for example, $\xxx_n= \{0,1\}^d$). If $\xxx$ is a connected compact Riemannian manifold then $\xxx_n$ could be any maximal $\delta$-dense subset; for  a sufficiently small%
	\footnote{%
	    Let $d_g$ be the geodesic distance on $\xxx$.  
        The main result of \cite{KatzUsadiKatz_GeomDedicata_2011__BiLipEmbeddingFinManifolds} shows that if $\delta$ is less than $10$ times the length of the smallest non-contractible loop in $\xxx$ (i.e. $\xxx$'s systole) then, map $x\mapsto (d_g(x,x_l))_{l=1}^n$ is a high-dimensional Euclidean bi-Lipschitz embedding of the Riemannian manifold $\xxx$; where $\xxx_n:=\{x_l\}_{l=1}^n$.  A key point here is the trade-off between the embedding dimension $n$, which grows exponentially in $\delta^{-1}$ (via an elementary covering argument) and the Lipschitz constant of the map $x\mapsto (d_g(x,x_l))_{l=1}^n$ which is large whenever $\delta$ is small.%
	} $\delta>0$ this set of $n$ points encodes most of the metric information in $(\xxx, d_{\mathcal{X}})$. In machine learning, $\xxx_n$ is the dataset to be analyzed or the training set which is randomly drawn from $\xxx$. In metric embedding theory suitable finite subsets play an analogous role to Schauder bases in the theory of Banach spaces (see \cite{Naor_SnapshotRibe_ICM_2018}).

    \begin{wrapfigure}{l}{0.3\textwidth}
    \centering
    \includegraphics[width=1\linewidth]{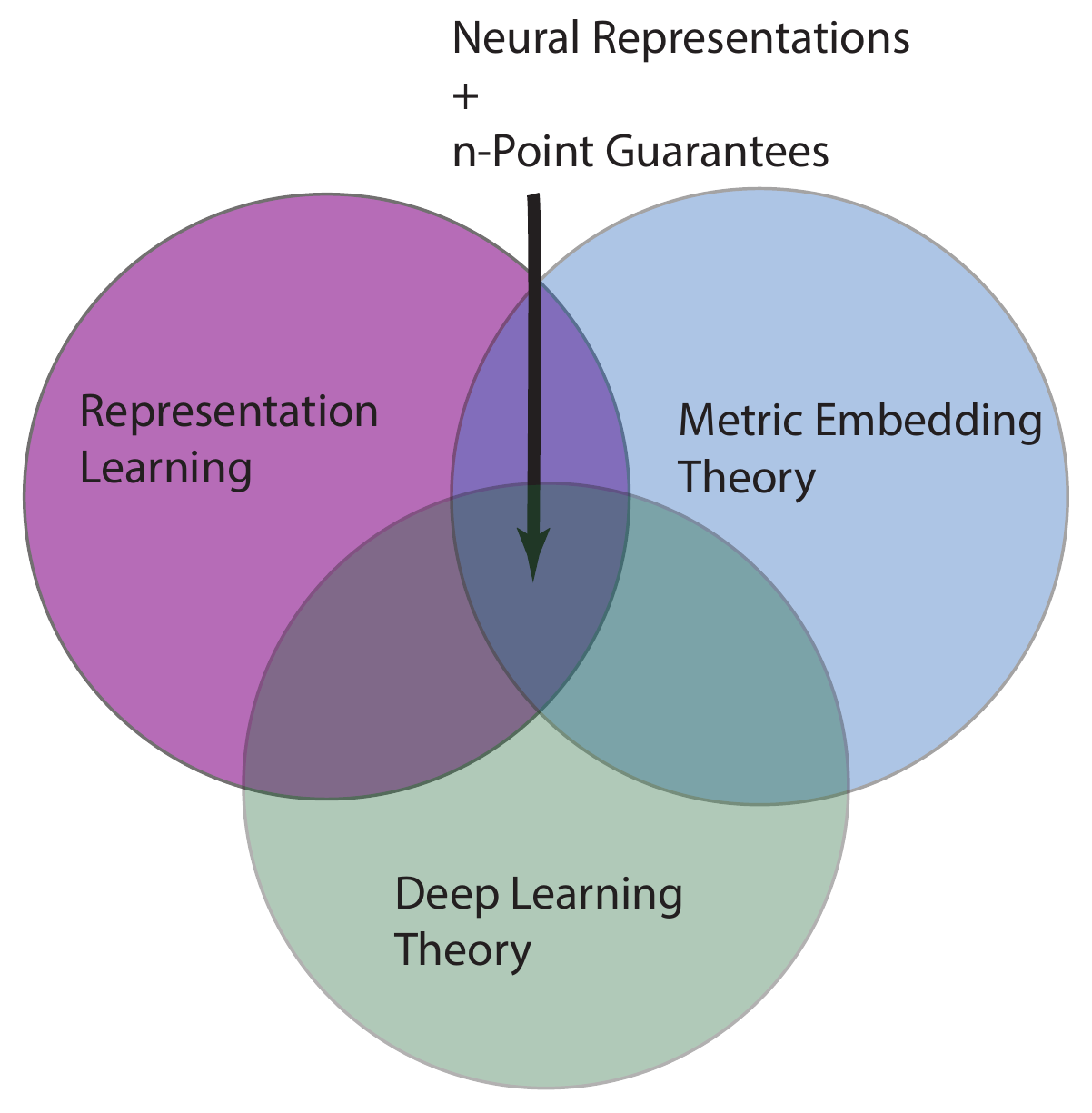}
    \end{wrapfigure}
	This places our work at the junction between \textit{metric embedding theory} and \textit{representation learning}.  The former aims at proving the existence of low-dimensional and low-distortion embeddings of finite metric spaces (e.g.\ $\xxx_n$), while the latter seeks computationally tractable representations of non-linear data (e.g.\ $\xxx$) which are performant in downstream learning tasks. Our theoretical guarantees concern $n$-point datasets from $\mathcal{X}$ because one cannot expect to embed general infinite metric spaces into any single ``reasonable''\footnote{By ``reasonable'' we mean a separable complete geodesic metric space.} space. 
	Finally, we complement the theory by stylized computer experiments which illustrate low metric distortion of our proposed representations (compared to Euclidean and hyperbolic) as well as out-of-sample generalization.	

	\subsection{The Representation Space} Our search for a ``good'' representation space begins by examining the Wasserstein space, denoted by $(\mathcal{P}_2(\mathbb{R}^d),\mathcal{W}_2)$, ($d\geq 1$), which is both infinite-dimensional and exhibits \textit{``positive curvatures on all scales''} \citep[Section 4.1.1]{kloeckner2016geometric}. Furthermore, even when $d=1$, $(\mathcal{P}_2(\mathbb{R}^d),\mathcal{W}_2)$ contains an isometric copy of any finite-dimensional bounded Euclidean ball.  Thus, it has plenty of room and a complex-enough geometry to make it a viable candidate for a universal representation space. A fortiori, this intuition has been confirmed by Andoni, Naor, and Neiman \cite{NaorNeimanAndoniSnowflake2018AnnalsSENSQ} who show that $(\mathcal{P}_2(\mathbb{R}^d),\mathcal{W}_2)$ embeds any finite dataset with little or no distortion to $\xxx$'s metric geometry when $d>2$, with $d=2$ still an open problem in geometric analysis.
	
	A computational drawback of $(\mathcal{P}_2(\mathbb{R}^d),\mathcal{W}_2)$ is that its elements cannot be \textit{exactly implemented} by a computer (capable of processing real values) but rather only approximated by an expressive class of probability measures such as Gaussian mixtures or empirical distributions (atomic probability measures) \cite{chevallier2018uniform}. This makes them ill-suited for sharp embedding guarantees since the optimal approximation rate using, for example, atomic measures with $M$ atoms in $\mathcal{P}_2(\mathbb{R}^d)$ with respect to the Wasserstein distance is $\mathcal{O}(M^{-1/d})$ \cite{FoundationsofQuantization2000GrafLuschgy,KloecknerQuantizationAhlforsRegular2012,LiuPages2020Quantization}. Moreover, restricting the Wasserstein-$2$ distance to either the set of Gaussian mixtures or empirical (atomic) probability measures no longer yields a complete geodesic space, resulting in an ill-behaved geometry.
	
	We sidestep these issues by instead working on the smaller optimal transport-theoretic space $(\mathcal{GM}_2(\mathbb{R}),$ $\mw_2)$, introduced by \cite{WassersteinGaussianMixtures}, whose elements are (finitely parameterized) \textit{univariate Gaussian mixtures} and whose distance is a strengthening of the Wasserstein distance obtained by considering only those couplings which are themselves Gaussian mixtures. This space, metrized by restricting the optimal couplings, is part of a broader line of recent work which encodes additional structure into the transport distance by constraining the couplings; we highlight adapted optimal transport \cite{allAdaptedTopsareEqual,AcciaioBackhoffZalashki_CausalOT_2020,Gudi2022_Ann}, martingale optimal transport \cite{MeteDolinsky_2014_PTRF,MathiasNizarMarcel_2017_AOP,JanGuo_2019_AAP,GudiJulio_2022_AAP}, and semi-martingale optimal transport \cite{ChongAriel_2019}.  The marked advantage of $(\mathcal{GM}_2(\mathbb{R}),\mw_2)$ over $(\mathcal{P}_2(\mathbb{R}^d),\mathcal{W}_2)$ is that its elements can  be exactly implemented (up to numerical precision) which eliminates the quantization error of $\mathcal{O}(M^{-1/d})$.  Geometrically, $(\mathcal{GM}_2(\mathbb{R}),\mw_2)$ shares many of the appealing properties of the classical Wasserstein space $(\mathcal{P}_2(\mathbb{R}^d),\mathcal{W}_2)$ (e.g., it is a complete separable \textit{geodesic space}); thus, there are no geometric drawbacks of working with $(\mathcal{GM}_2(\mathbb{R}),\mw_2)$ instead of with $(\mathcal{P}_2(\mathbb{R}^d),\mathcal{W}_2)$.  As byproduct, since the distance $\mw_2$ is strictly stronger than the classical Wasserstein distance $\mathcal{W}_2$, our deep neural embedding results into $(\mathcal{GM}_2(\mathbb{R}),\mw_2)$ automatically imply embeddings in $(\mathcal{P}_2(\mathbb{R}),\mathcal{W}_2)$.  
	
	To summarize, we move to the space of univariate Gaussian mixtures with transport distances of \cite{WassersteinGaussianMixtures} for which there is a well-developed machine learning machinery (including non-linear dimension reduction \cite{bigot2017geodesic}, clustering \cite{mi2018variational}, and regression \cite{chen2021wasserstein}) and which is neither ``too small'' nor ``too flat''.  Furthermore, as we show, the feature maps are computationally tractable and can be parameterized by deep neural networks. Finally, several maintained Python libraries streamline the implementation of deep learning algorithms in these spaces \cite{PyOpt_2021_FlamaryCoutryetAl,GaussMixGithub}.  
	
	\subsection{Universal Feature Maps}	

	We implement the feature maps using an extension of the \textit{probabilistic transformer} (PT) model \cite{AB_2022,kratsios2021_GCDs}. PT extends the classical transformer network used in natural language processing\footnote{And now in many other places in deep learning.} \cite{vaswani2017attention}; in particular, the PT implements the classical transformer ``on average''. It can be shown to universally approximate any continuous function while simultaneously \emph{exactly} implementing arbitrary compact (possibly non-convex) constraints \cite{AB_2022}.
	
	\begin{figure}[ht!]%
		\centering	
        \includegraphics[width=1\linewidth]{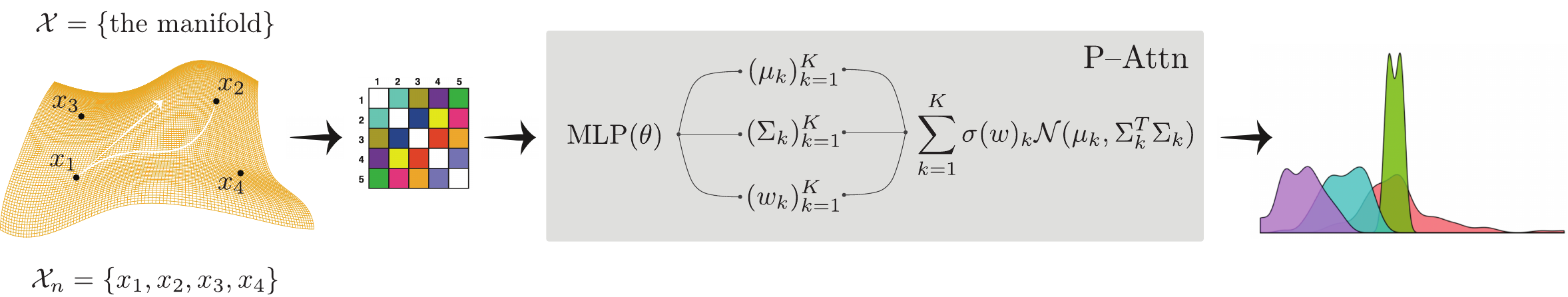}
		\caption{Schematic view of a probabilistic transformer ({\color{gray}{gray}}) representing a metric space $\xxx$ ({\color{orange}{orange}}) while bi-Lipschitz embedding the distinguished points (landmarks, data) $\xxx_n$ ({\color{landmarkgreen}{black}}).}
		\label{fig:transformer_general_intro}
	\end{figure}
	
	We adapt the more general probabilistic transformer model of \cite{kratsios2021_GCDs} by composing it with a non-parameterized input layer which translates metric information into vectorial data by mapping a point in $\xxx$ to a vector of its distances to a set of \emph{landmarks} $\{x_l\}_{l=1}^n\subseteq \xxx$ (cf. Figure~\ref{fig:transformer_MG}). The resulting process of representing a metric space $\xxx$ with the PT with embedding guarantees for any distinguished $n$-point dataset $\xxx_n \subseteq \xxx$ is illustrated in Figure~\ref{fig:transformer_general_intro}.  
    	
	\subsection{Summary of Contributions}
    \begin{figure}
		\centering
		\includegraphics[width=0.35\linewidth]{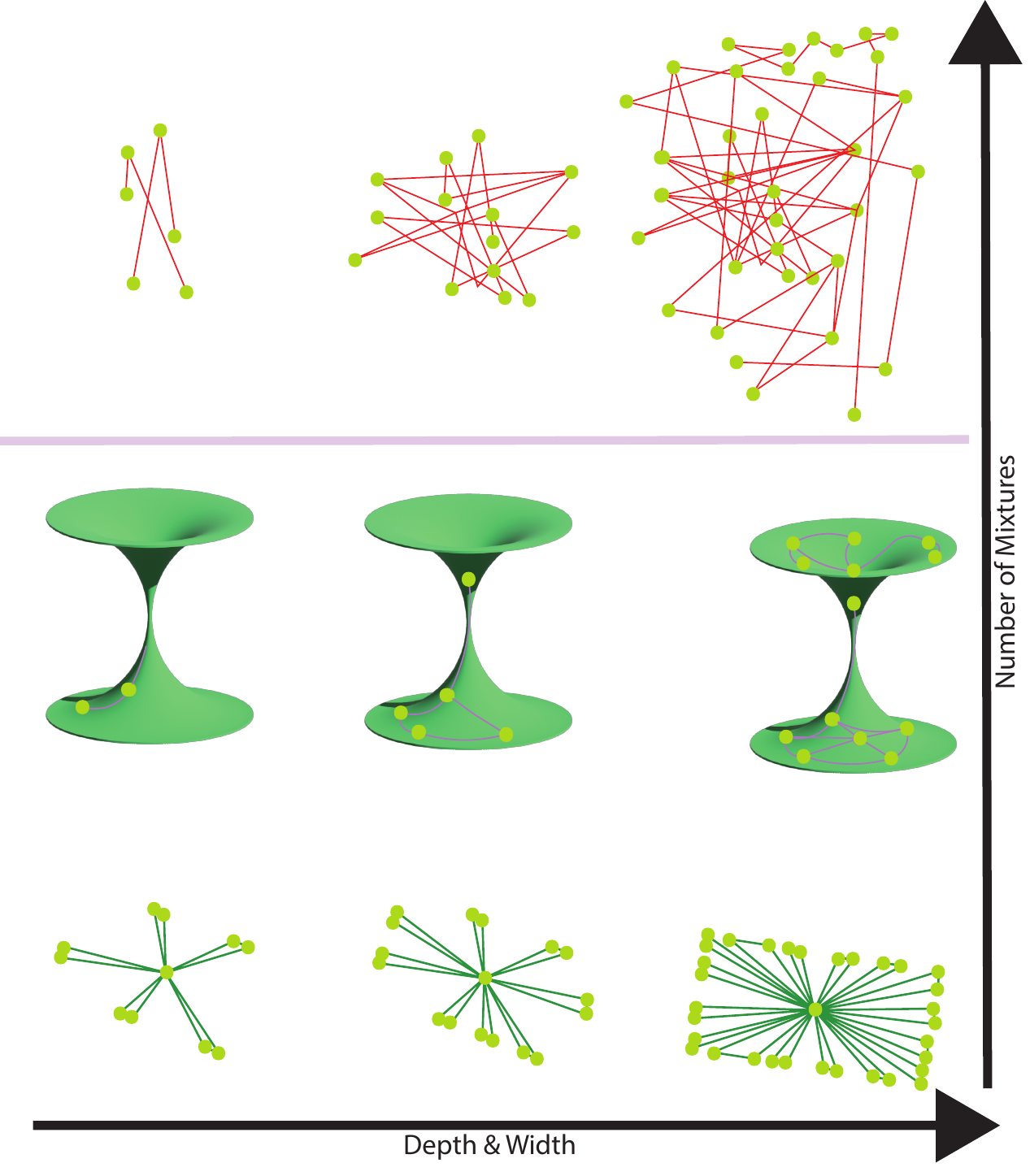}
		\caption{Embedding landscape of probabilistic transformers. Data from structured spaces such as trees or manifolds is ``easiest'' to embed in terms of feature map complexity. Data from unstructured spaces (general metric graphs, expanders) requires more complex (but still uncursed) neural feature maps.}
		\label{fig_paper_summary}
    \end{figure}
    
	Our results can be summarized with reference to the PT's embedding landscape in Figure~\ref{fig_paper_summary}.  Our first two \textit{universal representation theorems} relate to the first row, where no ``geometric prior'' is assumed on the metric space $\xxx$.  The second and third rows relate to the sharper embedding guarantees we obtain when $\xxx$ does have a regular geometry. In particular, this regularity allows us to explicitly bound the embedding's effective dimension which becomes independent of $n$, and depends only on $\xxx$'s latent geometry.  
	
	We begin by showing that, for any finite dataset $\xxx_n$ in $\xxx$ and any geometric perturbation parameter $\alpha\in (1/2,1)$ (H\"{o}lder coefficient), there is a PT which bi-$\alpha$-H\"{o}lder embeds $\xxx_n$ into $\mathcal{R}$ with metric distortion equal to that of the quantitative version of Assouad's embedding theorem \cite{Assouad} as derived by Naor and Neiman \cite{naor2012assouad}. We note that the result of Naor and Neiman is an existence theorem, while we prove that the embedding can be implemented by a concrete deep neural network (the PT) with explicitly controlled width and depth.  
	Of course, it is desirable to transition from $\alpha < 1$ to a bi-Lipschitz embedding with $\alpha=1$. One reason is that Lipschitz feature maps are easiest to learn from \cite{kratsios2021_GDL}. Our second main result shows that for any $D>1$, still without geometric assumptions on $\xxx$, there exists a PT which represents $\xxx_n$ in $\mathcal{R}$ and for which a pair of uniformly randomly sampled points are bi-Lipschitz embedded with metric distortion at most $D$ with probability at least $\mathcal{O}(n^{-4e/(1+D)})$. 
	Thus the PT exhibits a natural trade-off between embedding quality and satisfiability of the embedding.  
	
	It is natural to expect that if $\xxx$ has a regular geometric prior, it should be possible to obtain a deterministic small-distortion bi-Lipschitz guarantee for all points in $\xxx_n$.  This is indeed the case: we show that if $\xxx$ is a compact $d$-dimensional Riemannian manifold with controlled Ricci curvature, then there is a PT which bi-Lipschitz embeds $\xxx_n$ into $\mathcal{R}$ with effective dimension $3d+6$.  A similar fact holds for combinatorial trees, with PT's metric distortion coinciding with that of Gupta's embeddings of trees in Euclidean space \cite{Gupta_Trees_Quantitiativefinitedimembeddings_ACM}.  In these cases, we also obtain explicit depth, width, and effective dimension estimates. It is worth emphasizing that the depth and the width only depend on the number of points $n$ for which we seek a memorization guarantee (the number of landmarks) and the effective dimension only depends on the regularity and geometry of $\mathcal{X}$ (cf. Figure~\ref{fig_paper_summary}).  
	Lastly, we show that if one embeds into Gaussian mixtures on Euclidean space of dimension at-least three, then there are probabilistic transformers which can implement bi-H\"{o}lder embeddings of arbitrarily low distortion.
	
	Our results are the first theoretical results on embeddings of datasets from metric spaces using deep neural feature maps. They are also the first guarantees that the discrete geometries covered by our results can be embedded into the space of Gaussian mixtures on $\mathbb{R}$ with an optimal transport-type distance, independently of how the embeddings are implemented. Finally, we develop new proof techniques that opportunely combine metric embedding theory \cite{heinonen2003geometric,KrauthgameLeeNaor2004,naor2012assouad,ostrovskii2013metric,eriksson2018quantitative,NaorNeimanAndoniSnowflake2018AnnalsSENSQ}, memorization theory of deep feedforward networks \cite{sontag1997shattering,Park,NEURIPS2020_1e14bfe2,NEURIPS2020_662a2e96,Memorization_ReluThreshfold_2020_SIAM,NEURIPS2020_34609bdc,vardi2022on}, and computational optimal transport \cite{WassersteinGaussianMixtures}. Our results immediately imply embeddings into the usual Wasserstein-2 space over $\mathbb{R}$.

	\section{Geometric Background}
	\label{s_Intro_Background}
	We briefly overview some of the relevant terminologies from metric embedding theory.  For further details we recommend the book \cite{ostrovskii2013metric} or the lecture notes \cite{matouvsek2013lecture,naor2015metric}.  
	
	A metric space is a pair $(\xxx,d_{\xxx})$ of a set $\xxx$ and a \textit{distance function} $d_{\xxx}:\xxx\times \xxx\rightarrow[0,\infty)$ which is symmetric in its arguments, is $0$ if and only if $x=\tilde{x}$ (thus can uniquely identify points), and satisfies the triangle inequality: for every three points $x_1,x_2,x_3$ in $\xxx$
	$
	d_{\xxx}(x_1,x_3)\leq d_{\xxx}(x_1,x_2) + d_{\xxx}(x_2,x_3)
	$. 
	The prototypical metric space is the Euclidean space; i.e. $\rr^n$ with Euclidean distance $d^2(x,\tilde{x})\eqdef \sum_{i=1}^n (x_i-\tilde{x}_i)^2$.  We now review three other examples which are central to our analysis.
	
	\subsection{Combinatorial Graphs}
	\label{s_Intro_Background__ssCombinatorialGraphs}
		A \textit{(combinatorial) graph} is a pair $G\eqdef (V,E)$ of a set of \textit{vertices} (or vertices) $V$ and a set of unordered pairs%
		\footnote{For a set $V$, $\binom{V}{2}$ denotes the quotient of the Cartesian product $V\times V$ under the equivalence relation $\sim$ defined on any $(v_1,v_2),(v_3,v_4)\in V\times V$ by $(v_1,v_2)\sim (v_3,v_4)$ if and only if $(v_1,v_2)=(v_3,v_4)$ or $(v_1,v_2)=(v_4,v_3)$.}
		$E \subseteq \binom{V}{2}$, called \textit{edges}, such that $\{u, v\} \in E$ if there is an edge between $u,v\in V$.  We will always assume $G$ to be \textit{connected}, which means that for any $u,v\in V$ there is a sequence of edges $\{u,v_1\}$, $\dots$, $\{v_N,v\}$ in $E$ ``linking $u$ to $v$.  
		Every graph $G = (V,E)$ induces a metric space $\xxx_G\eqdef (V,d_G)$ whose \textit{graph geodesic} distance $d_G$ is defined for any $u,v\in V$ by
		\[
        		d_G(v,u) 
    		\eqdef 
                d_G(u,v) 
                    \eqdef 
                \inf\big\{N : \exists\,\{v, v_1\},\dots\{v_{N-1},u\}\in E\big\}
		.
		\]
    We only consider \textit{simple graphs}, meaning that any two vertices can be connected by at most one edge.  
	
	\subsection{Riemannian Manifolds with Lower-Bounded Ricci Curvature}
	\label{s_Intro_Background__ssRiemanninan_manifold_LBRicci_CurvatureRiemanninan_manifold_LBRicci_Curvature}
	We overview the ``smooth geometric prior'' considered in this paper; namely, we review the notion of a Riemannian manifold.  
	For a more detailed exposition on Riemannian geometry, we refer the reader to \citep[Chapter 1]{jost2008riemannian} on the topic.  
	
	A $d$-dimensional Riemannian manifold $(M,g)$ is a pair of a $d$-dimensional \textit{smooth manifold} $M$ and smoothly varying family of positive-definite inner-products $g\eqdef (g_x)_{x\in M}$, with each $g_x$ mapping pairs of vectors in the tangent space $T_x(M)$ above the point $x\in M$ to $\rr$.  Equivalently, by \citep[Theorem 2]{Nash_1954_c1_Embeddings} we may describe any $d$-dimensional Riemannian manifold ``extrinsically'' by viewing it as a Riemannian submanifold of the Euclidean space $\rr^{2d+1}$.
	
    We say that $(M,g)$ is complete if every two points in $M$ can be reached by a (continuous) curve and if the \textit{geodesic distance}, defined for $x,\tilde{x}\in M$ by
		\begin{equation}
			\label{s_Intro_Background__ssRiemanninan_manifold_LBRicci_CurvatureRiemanninan_manifold_LBRicci_Curvature___definition_Riemannian_metric}
			d_g(x,\tilde{x})
			\eqdef
			\inf_{\gamma}\,
			\int_0^1 \, 
			\langle \gamma'(t),\gamma'(t)\rangle_{\gamma(t)}
			dt
			,
		\end{equation}
		makes $(M,d_g)$ into a complete metric space; where the infimum in~\eqref{s_Intro_Background__ssRiemanninan_manifold_LBRicci_CurvatureRiemanninan_manifold_LBRicci_Curvature___definition_Riemannian_metric} is taken over all piece-wise continuously differentiable curves $\gamma:[0,1]\rightarrow \rr$ starting at $x$ and ending at $\tilde{x}$; i.e. $\gamma(0)=x$ and $\gamma(1)=\tilde{x}$.  By the main result of \cite{Hopf_Rinow_1931}, the completeness of $(M,d_g)$ is equivalent to the existence of unique smooth curves minimizing the objective function~\eqref{s_Intro_Background__ssRiemanninan_manifold_LBRicci_CurvatureRiemanninan_manifold_LBRicci_Curvature___definition_Riemannian_metric}; these curves are called \textit{geodesics}.  
		
		In this paper, we focus on Riemannian manifolds satisfying a lower-bounded \textit{Ricci curvature condition}, which we express using the \textit{curvature-dimension inequality} of \cite{BaudoinBonnefontGarofalo_2014_SubRiemannianCurvatureDoubling}: for every infinitely differentiable map $f:M\rightarrow \rr$ the following holds
		\begin{equation}
			\label{eq_cdi_beaudoin}
			\frac1{2} \Delta \| \nabla f\|_g^2 - \langle \nabla f,\nabla (\Delta f)\rangle_g
			\geq 
			\frac1{d} \|\Delta f\|_g^2 + r  \|\nabla f\|_g
			,
		\end{equation}
		where $\Delta$ is the Laplacian on $(M,g)$ (which describes heat flow on $(M,g)$) and $\|\cdot\|_g\eqdef \langle \cdot,\cdot\rangle_g$.  To make matters concrete, one can show that a $d$-dimensional sphere of radius $\rho>0$ satisfies~\eqref{eq_cdi_beaudoin} with $r=\rho^{-2}d(d-1)>0$, the hyperbolic plane satisfies~\eqref{eq_cdi_beaudoin} with $r=-1<0$, and the Euclidean space satisfies~\eqref{eq_cdi_beaudoin} with $r=0$.  Since $\Delta$ dictates the heatflow on $(M,g)$, then $r$ in~\eqref{eq_cdi_beaudoin} can be interpreted as describing the rate at which heat dissipates on $(M,g)$ with small $r$ being fast (e.g.\ the hyperbolic plane) and large $r$ being slow (e.g.\ the sphere).  
	
	\subsection{Mixed-Gaussian Optimal Transport}
	\label{s_Intro_Background__ssWasserstein_Space}
    We rely on one more example of a non-Euclidean metric space, which is a variant of the space of probability measures introduced by \cite{vaserstein1969markov_OG}, in the context of optimal transport theory.  This is the optimal transport-theoretic space of Gaussian mixtures introduced in \cite{WassersteinGaussianMixtures}, which specializes in the classical optimal transport constructions to the class of Gaussian mixtures.  This metric space of probability measures is part of a broader area of contemporary research encoding structure into transport plans by restricting the classes of admissible transport plans; see \citep{MonotoneTransport_BeiglbrockLabodereTouzi_2017,NAdalTalay_Transport_Refined_SDE,AcciaioBackhoffZalashki_CausalOT_2020,Backhodd_Bartl_Beiglbock__AlladaptedTopequal_2020,BackhoddPammer_2022_StabilityMartingaleAndWeakOT}).  In what follows, we use $\mathcal{P}_p(\rr^d)$ to denote the set of probability measures on the Euclidean space $\rr^d$ admitting a $p^{th}$-moment; i.e. $\pp\in \mathcal{P}_p(\rr^d)$ if $\mathbb{E}_{X\sim \pp}[\|X\|^p]<\infty$.

		We refer to the metric space $\mathcal{W}_2(\rr^d) \eqdef (\mathcal{P}_2(\rr^d),\mathcal{W}_2)$ as the Wasserstein space, where $\mathcal{W}_2$ is defined for any two $\pp$ and $\qq$ in $\mathcal{P}_2(\rrd)$ by
		\begin{equation}
			\label{eq_s_Intro_Background__ssWasserstein_Space__DefinitionWasserstein}
			\mathcal{W}_2^2(\pp,\qq)
			\eqdef 
			\inf_{\pi\in \operatorname{Cpl}(\pp,\qq)}\,
			\mathbb{E}_{(X,Y)\sim \pi}\left[
			\|X-Y\|^2
			\right]
			,
		\end{equation}
		where the set $\operatorname{Cpl}(\pp,\qq)$ consists of all probability distributions on $\rr^{2d}$ with marginals $\pp$ and $\qq$.  A key point, highlighting our interest in the case in embedding of metric spaces into univariate distributions, is that when $d=1$ then
		\[
		\mathcal{W}_2(\pp, \qq) 
		= 
		\int_0^1 \Big( Q_{\pp}(t) - Q_{\qq}(t) \Big)^2 \, dt
		,
		\]
		where $Q_{\pp}$ and $Q_{\qq}$ are the respective quantile functions of $\pp$ and of $\qq$.  Thus, for empirical distributions, unlike in the multivariate case ($d>1$) where computing $\mathcal{W}_2(\pp,\qq)$ has super-cubic complexity \cite{orlin1988faster}, in the univariate case ($d=1$) $\mathcal{W}_2$ can be easily computed via a simple sorting procedure; i.e. with near linear complexity.  
		
		In the case where $\pp$ and $\qq$ are \textit{Gaussian mixtures} \citep[Section 4.1]{WassersteinGaussianMixtures} propose a stronger distance function than the classical Wasserstein-$2$ distance $\mathcal{W}_2(\pp,\qq)$ on the set of subset of $\mathcal{P}_2(\rr)$ consisting of univariate Gaussian mixtures; we denote this set by $\mathcal{GM}_2(\rr)$.  This distance is defined by restricting the couplings in optimization problem~\eqref{eq_s_Intro_Background__ssWasserstein_Space__DefinitionWasserstein} to the smaller class couplings $\pi$ which are themselves are Gaussian mixtures (on $\rr^{2d}$).  
		Denoted by $\mw_2(\pp,\qq)$, the mixed-Gaussian Wasserstein distance between two Gaussian mixtures $\pp=\sum_{i=1}^I w_{1,i}\,N(\mu_{1,i},\Sigma_{1,i})$ and $\qq=\sum_{j=1}^J w_{2,j} \, N(\mu_{2,j},\Sigma_{2,j})$ is equivalently expressed in the following computationally tractable formulation.
		Denote the Wasserstein-2 distance between $N(\mu_{1, i}, \Sigma_{1, i})$ and $N(\mu_{2, j}, \Sigma_{2, j})$ by $\omega_{ij}$ and let $\Omega = (\omega_{ij}^2)_{i, j = 1}^{I, J}$. Then $\mw_2^2(\pp, \qq)$ is given by
		\begin{align*}
		    \mw_2^2(\pp, \qq) ~~ \eqdef \quad\quad &\min_{\mathclap{V \in [0, \infty)^{I \times J}}} \quad\quad
		    ~~~\operatorname{tr}( V^\top \Omega)\\
		    &\text{subject to}
		    \quad
		    V\boldsymbol{1}_{J} = w_{1}\\
		    &\phantom{\text{subject to}}
		    \quad
		    V^\top \boldsymbol{1}_{I} = w_{2}
		\end{align*}
		where $\boldsymbol{1}_L=(1,\dots,1)^\top$ denotes the $L$-dimensional vector of ones.  Central to the tractability of $\mw_2^2$ is that, by \cite{ClosedFormW2}, the Wasserstein-$2$ distance between two Gaussians $N(\mu_1,\Sigma_1)$ and $N(\mu_2,\Sigma_2)$ is given in closed form by
		\[
    		\mathcal{W}_2^2\big(
        		N(\mu_1,\Sigma_1)
        		    ,
        		N(\mu_2,\Sigma_2)
    		\big)
    		    =
    		\left\|
    			\mu_{1}
    			-
    			\mu_{2}
    			\right\|^2
    			+
    			\operatorname{tr}\left(
    			\Sigma_{1}
    			+
    			\Sigma_{2}
    			-
    			2 \sqrt{
    				\sqrt{\Sigma_{1}}
    				\Sigma_{2}
    				\sqrt{\Sigma_{1}}
    			}
    			\right)
		,
		\]
		where $\sqrt{A}$ denotes the square-root of a symmetric positive-definite matrix $A$.  
		
		There are several relationships between the two optimal transport distances.  First, clearly by definition one has $\mathcal{W}_2\leq \mw_2$; moreover, the inequality is often strict.  
		Conversely, it holds that
		\[
		\mw_2(\pp,\qq) \leq 
		\mathcal{W}_2(\pp,\qq)
		+
		\sqrt{2}
		\biggl(
			\Big(
    			\sum_{i=1}^I
    			w_{1,i} \operatorname{tr}\left(\Sigma_{1,i}\right)
			\Big)^{1/2}
			+
			\Big(
    			\sum_{j=1}^J
    			w_{j,i} \operatorname{tr}\left(\Sigma_{2,j}\right)
			\Big)^{1/2}
		\biggr)
		;
		\]
		from which we see that $\mw_2$ equals to $\mathcal{W}_2$ for finitely supported measures.  
	We use $(\mathcal{GM}(\rr^d),\mw_2)$ to denote the space of Gaussian mixtures with mixed-Gaussian Wasserstein metric.  
	
	\subsection{Bi-H\"{o}lder Embeddings}
	A map $\phi:\xxx\hookrightarrow \mathcal{R}$ from a metric space $(\xxx,d)$ and to a (metric) representation space $(\mathcal{R},d_{\mathcal{R}})$ is called a bi-H\"{o}lder embedding if: there is an $0<\alpha\leq 1$, a \textit{scale} $s>0$, and a \textit{distortion} $D\geq 1$ such that for every $x,\tilde{x}\in \mathcal{X}$
    \begin{equation}
    \label{eq_biHolder_definition}
    	    sd^{\alpha}(x,\tilde{x}) \,
	\le
	    d_{\mathcal{R}}(\phi(x_1),\phi(x_2)) 
	\le
	\,
	    sD\,d^{\alpha}(x,\tilde{x})
	.
    \end{equation}
	The smallest value of $D$ for which~\eqref{eq_biHolder_definition} holds quantifies the worst-case dilation between any two points induced by the feature map $\phi$, upon re-scaling%
	\footnote{If $\mathcal{R}$ is a normed space and $\phi$ is invariant under linear maps then we can always absorb $s$ into the model.  However, for general representation spaces resealing is not meaningful without the scale $s$.  }
	by $s$.  As an example, if $\mathcal{R}=\rr^n$, $\xxx\subseteq \rr^m$, and if $\phi$ is smooth then $D$ is roughly $\max_{x}\,\|\nabla \phi(x)\|$; however, for general metric spaces these quantities are meaningless.  Thus, a large distortion $D$ only \textit{globally} distorts $\xxx$'s geometry by stretching all of it.  
    As illustrated in Figure~\ref{fig_alpha_bigdeal}, a small value of $\alpha$ relatively distorts the space's geometry by placing more importance (illustrated by bright colours) on nearby pairs of points and less importance on distant pairs of points (relative to the {\ color {red}{red}} vertex in the center of Figure~\ref{fig_alpha_bigdeal}).  This has the effect that large scales become more minor and small scales become larger when $\alpha\ll 1$, with this effect disappearing as $\alpha$, approaches $0$.  
	\begin{figure}[ht]%
    \centering
	\begin{subfigure}[t]{0.49\textwidth}
        \centering
        \includegraphics[width=.5\linewidth]{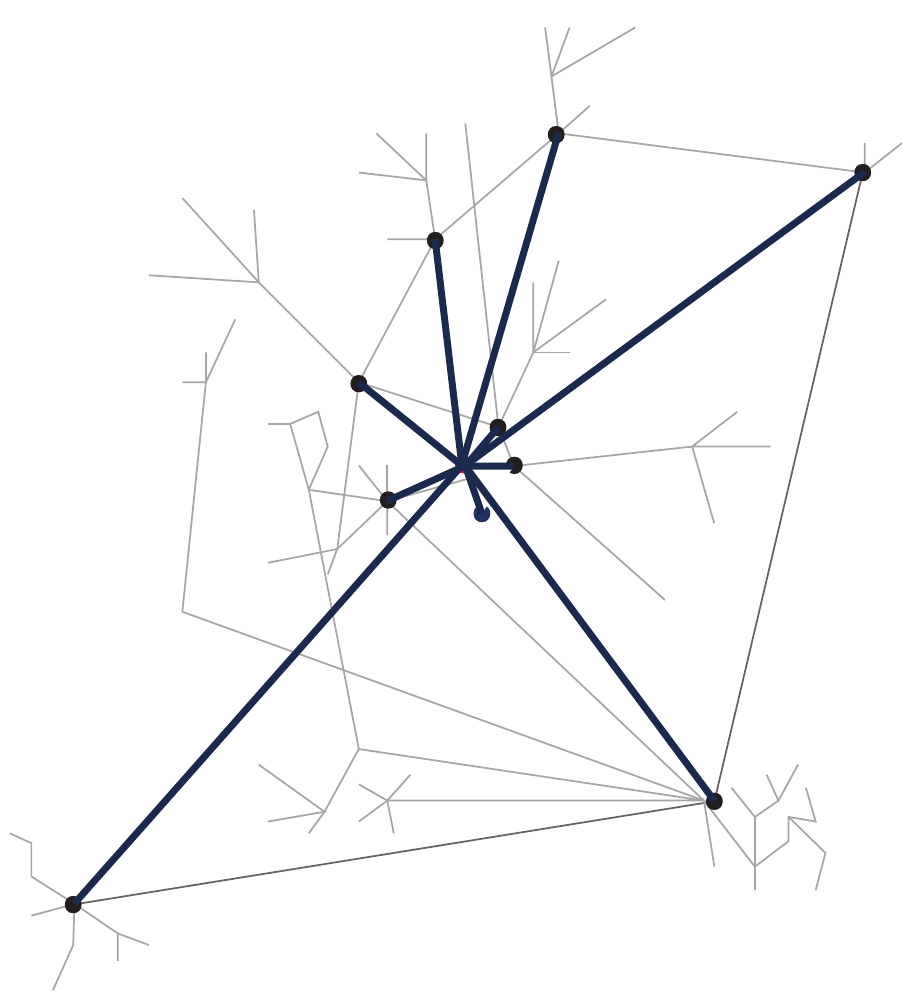}
		\caption{}\label{fig__no_softmessage_passing}
    \end{subfigure}
	\begin{subfigure}[t]{0.49\textwidth}
        \centering
        \includegraphics[width=.5\linewidth]{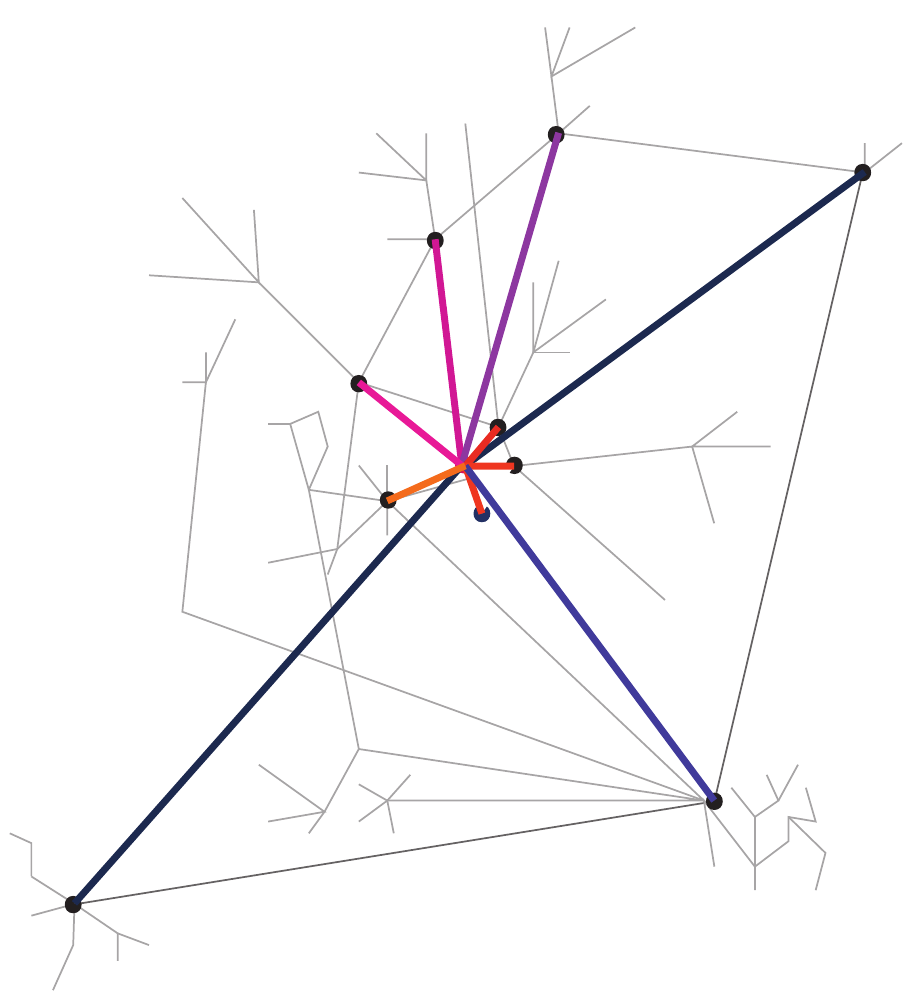}
		\caption{}\label{fig__softmessage_passing}
    \end{subfigure}
    \caption{Effect of varying $\alpha$ on the geometry of the embedding metric space.
        a) $\alpha=1$ (Lipschitz): all distances are equally distorted.
	    b) $\alpha\ll 1$ (fractional): small distances are prolonged, large distances are shortened.}
    \label{fig_alpha_bigdeal}
    \end{figure}
    Effectively $\alpha$ plays a comparable role to tuning the \textit{neighbourhood size} in classical iterations of graph attention \cite{velickovic2018graph} or the role of neighborhood size when defining the \textit{mean average precision} performance metric of \cite{cruceru2020computationally,digiovanni2022heterogeneous}.   

    The parameter $\alpha \in (0,1]$ can also be interpreted through its role in the loss function, which we use to train our model.  If $\xxx_n$ is a set of points which we want to embed into $\mathcal{MW}(\rr)$ then in our numerical experiments, we will train a probabilistic transformer to learn a low distortion bi-H\"{o}lder embedding by numerically optimizing the following loss function which is a proxy for condition~\eqref{eq_biHolder_definition}:
    \begin{equation}
    \label{eq:loss_to_minimize}
        \sum_{x,\tilde{x}\in \xxx_n} 
            \big\vert
                d_{\xxx}(x,\tilde{x})^{\alpha}
                    -
                \mw_2(\phi(x),\phi(\tilde{x})) 
            \big\vert
    .
    \end{equation}
    Here, the minimization is taken over the set of a set of models.  When $\alpha\ll 1$ then,~\eqref{eq:loss_to_minimize} over emphasizes embedding nearby pairs of points over embedding distant pairs of points.  In contrast, when $\alpha=1$, all pairs of points are given equal importance in the embedding.  Thus, through the loss function~\eqref{eq:loss_to_minimize}, $\alpha$ implicitly plays the role of the neighborhood size defining the graph attention mechanism of \cite{velickovic2018graph} or the average (distance) distortion loss function (used in e.g.\ \cite{JostNature1,digiovanni2022heterogeneous}).  
	
	In \cite{KrauthgameLeeNaor2004}, the authors introduce the notion of an aspect ratio of a measure space as the ratio of total mass over the minimum mass at any point.  Similarly, we define the \textit{aspect ratio} of a finite metric space $(\xxx_n,d_n)$ as the ratio of the maximum distance between any two points therein (its diameter) over the minimum separation between any two distinct points.  We define
	\[
	\operatorname{aspect}(\xxx_n,d_n)
	\eqdef 
	\frac{
		\max_{x,\tilde{x}\in \xxx_n}\,
		d_n(x,\tilde{x})
	}{
		\min_{x,\tilde{x}\in \xxx_n
			;\,
			x\neq \tilde{x}}\,
		d_n(x,\tilde{x})
	}
	.
	\]
	We note that, closely related concepts to the aspect ratio can be found in the (approximate) memorization results for deep feedforward networks \cite{Sublinear_Memorization}. Furthermore, requiring that the aspect ratio is positive and finite circumvents otherwise memorization by a feedforward neural network may be impossible \cite{SontagShattering_1997}.

	\begin{wrapfigure}{l}{0.25\textwidth}
		\vspace{8mm}
        \centering
        \includegraphics[width=1\linewidth]{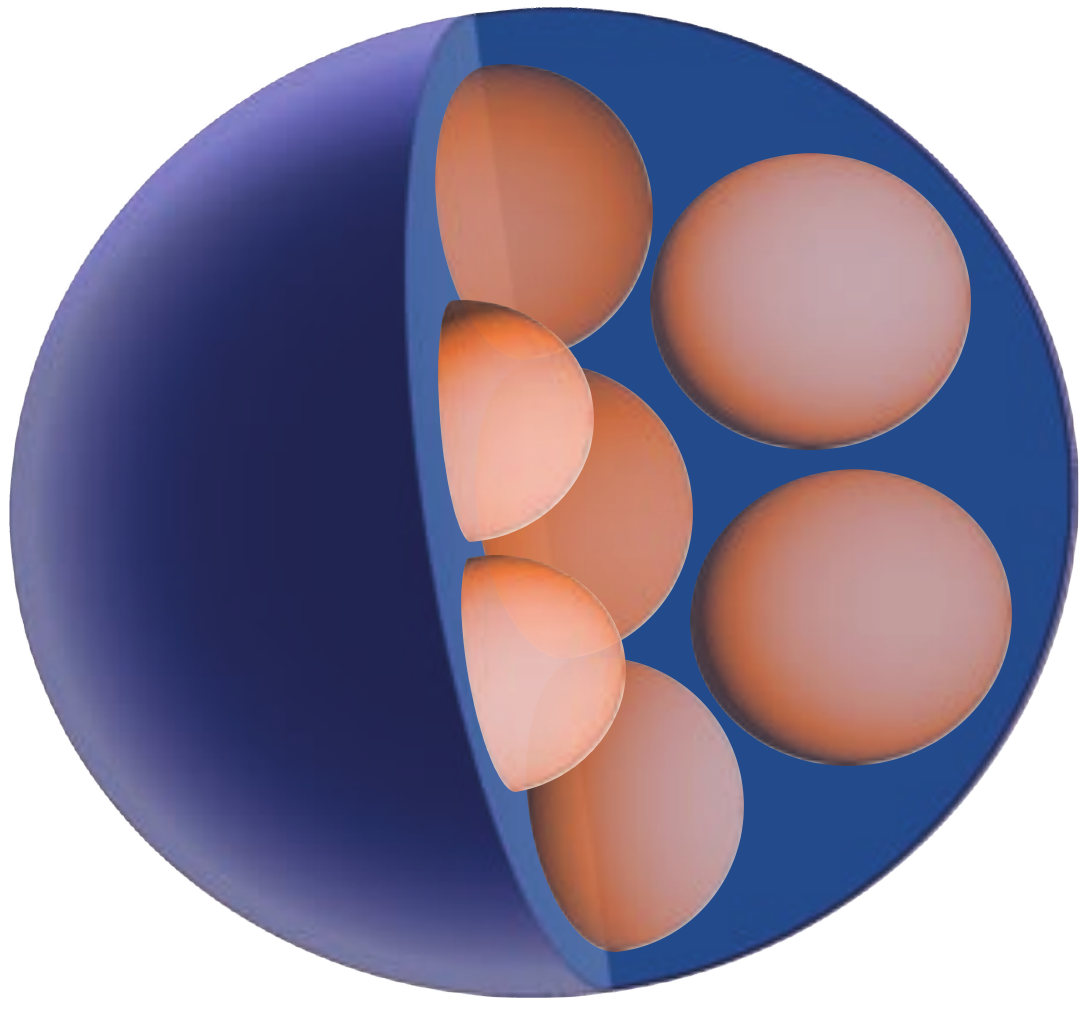}
        \caption{$\xxx$'s capacity is the number of small balls of radius $r/5$ that can fit in any ball of radius $r$.}
        \label{fig:metric_capacity}
        \vspace{-.05cm}
    \end{wrapfigure}

	A related quantity is the \textit{diameter} of the $n$-point metric space $(\xxx_n,d_n)$, defined by $\operatorname{diam}(\xxx_n,d_n)\eqdef \max_{x,\tilde{x}\in \xxx_n}\, d_n(x,\tilde{x})$.  Thus, the aspect ratio can be interpreted as $(\xxx_n,d_n)$'s diameter re-scaled by the reciprocal smallest distance in $\xxx_n$.

    Lastly, we will make use of the notion of ``metric capacity'' $\mathrm{cap}(\xxx,d_{\xxx})$ of a metric space $(\xxx,d_{\xxx})$.  
	This ``dimensional'' metric invariant is defined similarly to box or Hausdorff dimension (in fractal geometry \cite{Falconer_1986_GeometryofFractalSets}), is a formalization of the notion of dimension of a general metric space $\xxx$.  When $\xxx$ is the $m$-dimensional Euclidean space or when $m$-dimensional compact Riemannian manifold then, the logarithm of its capacity can be shown to be proportional to $m$.  
	Illustrated in Figure~\ref{fig:metric_capacity}, the \textit{capacity} $\mathrm{cap}(\xxx,d_{\xxx}) $ of a metric space $(\xxx,d_{\xxx})$ is equal to
    $$
    	\sup\big\{
        	    d \in \nn_+ 
        	:\,
            	\exists (x_i)_{i=0}^d \in \xxx^{d+1},\exists r>0\, \mbox{s.t. }
                	\sqcup_{i=1}^d \, B_{\xxx}(x_i,r/5)
                	    \subseteq 
                	B_{\xxx}(x_0,r)
    	\big\}
	,
	$$
	where $B_{\xxx}(x,r)\eqdef \{u \in \xxx:\, d_{\xxx}(u,z)<r\}$ and $\sqcup$ denotes the \textit{union} of disjoint sets.  We refer the reader to \citep[Chapter 10]{heinonen2001lectures} and \citep[Proposition 1.7]{Brue2021Extension} for further details on metric capacity and to \cite{le2015assouad} for an exposition of metric-theoretic notions of dimension.
	
	\section{The Probabilistic Transformer Model}
	\label{s_Transformer_Model}
    Since any good representation space should admit ``nice" feature maps which can be approximated by standard neural networks. Our main result shows that this holds for $\mathcal{MW}_2(\rr)$, with the neural network being the probabilistic transformer.  
	The PT model, introduced below and illustrated in Figure~\ref{fig:transformer_general}, processes a dataset $(\xxx, d_n)$ from a metric space $(\xxx,d_{\xxx})$ in three phases before producing a probabilistic output in $\mathcal{MG}_2(\rr^d)$.
	
	\begin{figure}[H]
		\centering	
		\includegraphics[width=0.9\linewidth]{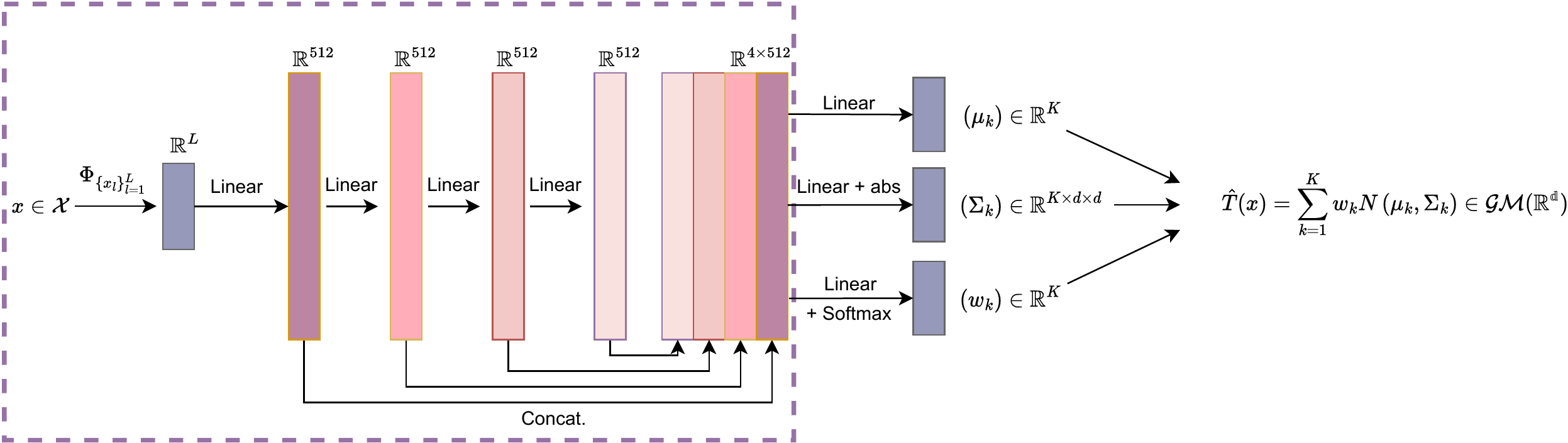}
		\caption{The architecture of a probabilistic transformer.}
		\label{fig:transformer_general}
	\end{figure}
	First, data in $(\xxx_n,d_n)$ is encoded into an $L$-dimensional vector of distances to a set of landmarks $x_1,\dots,x_L\in \xxx_n$. 
	The encoding, illustrated by the $\Phi_{\{x_l\}_{l=1}^L}$ layer in Figure~\ref{fig:transformer_general}, is reminiscent of the graph attention mechanism of \cite{velickovic2018graph}. The \textit{un-normalized graph attention} mechanism is defined as
	\begin{equation}
	    \label{eq:feature_map}
	\Phi_{\{x_l\}_{l=1}^L}
	: \xxx_n \ni x 
	\mapsto
	\left(
	d_n(x,x_{l})
	\right)_{l=1}^L \in \rr^L
	.
	\end{equation}
	The landmarks act as an intrinsic reference frame, they identify points in $\xxx_n$ by their relative positions to $\{x_{l}\}_{l=1}^L$.  This can be seen as a non-linear counterpart to a truncated orthonormal basis of a Hilbert space.
	When all points in $\xxx$ are chosen as landmarks, this un-normalized graph attention mechanism is a finite-dimensional version of the embedding of $(\xxx_n,d_n)$ into $(\rr^{n},\ell^{\infty})$ \cite{frechet1906quelques}.  
	
	\begin{remark}
	\label{remark_trainablelandmarks}
	    The points in $x_1,\dots,x_L$ in the definition of $\Phi_{\{x_l\}_{l=1}^L}$ are in principle trainable and could be discovered algorithmically instead of being specified by the user.
	\end{remark}
	
	Next, this vectorial data is processed by a set of \textit{independent}, \textit{parallel}
	ReLU-feedforward networks $\mathcal{NN}_w:\rr^L\rightarrow \rr^K$, and $K$ ReLU-feedforward networks $\mathcal{NN}_1:\rr^L\rightarrow \rr^{d}\times \rr^{d\times d}, \dots,\mathcal{NN}_K:\rr^L\rightarrow \rr^{d}\times \rr^{d\times d}$.  The role of the networks $\{\mathcal{NN}_k\}_{k=1}^K$ is to process inputs to a $d$-dimensional Gaussian measure-valued outputs $N_d(\mu,\Sigma)$ by outputting a mean vector $\mu$ and \textit{a} square-root of a covariance matrix $\Sigma^{\top}\Sigma$.  The role of the network $\mathcal{NN}_w$ is to generate input-dependent weights in the $K$-simplex $\{w\in [0,1]^K:\,\sum_{k=1}^K w_k=1\}$ used to mix the Gaussians into a Gaussian mixture.  
	
	\begin{remark}
	\label{rem_paralelization}
	It is theoretically sufficient for the ReLU network to be independent. In practice, this can be relaxed by merging them into a single network as we do in Section~\ref{s_Experiments}.
	\end{remark}
	
	The predictions of all networks are then combined using the probabilistic attention mechanism \cite{AB_2022} which generalizes the construction of \cite{bishop1994mixture}.  The (Gaussian) probabilistic attention mechanism is defined for any $K\in \nn_+$, any \textit{weight} $w\in \rr^K$, and any set of \textit{keys} $Y\eqdef ((\mu_1,\Sigma_1),\dots,(\mu_n,\Sigma_K))\in \rr^{K\times (d+d^2)}$ by
	\[
	\operatorname{P-Attn}(w,Y)
	\eqdef 
	\sum_{k=1}^K\,
	\softmax_K(w)_k\,N(\mu_k,
	\Sigma_k^{\top}\Sigma_k)
	\in 
	\mathcal{P}_2(\rr^d);
	\]
	where we have identified $\rr^{d+d^2}$ with $\rr^d\times \rr^{d\times d}$ (where $\rr^{N\times M}$ is the set of $N\times M$ matrices) and $\softmax_K(u)\eqdef (e^{u_k}/\sum_{i=1}^K e^{u_i})_{k=1}^K$ is the \textit{softmax function}.  The probabilistic attention mechanism's place in our model is illustrated by the last layer in Figure~\ref{fig:transformer_MG}, which synthesizes the mixtures from the $(\mu_k)_{k=1}^K$, $(\Sigma_k)_{k=1}^K$, and the $(w_k)_{k=1}^K$ parameters produced by the feedforward networks. 
	
	The parameters defining the Gaussian mixtures are implemented by deep feedforward networks with ReLU activation function $\operatorname{ReLU}(x)\eqdef \max\{0,x\}$.  Let $D\in \nn_+$. Fix a \textit{depth}
	\footnote{The depth is the total number of affine layers defining the network and not the number of ``hidden layers''; that is, the depth is one more than the number of $\operatorname{ReLU}$ activation functions applied component-wise in~\eqref{eq_definition_ffNNrepresentation_function}.}
	$J\in \nn_+$ and a multi-index $[d]\eqdef (d_0,\dots,d_J)$ of integers with $d_0=d$ and $d_J=D$.  A function $\mathcal{NN}:\rr^d\rightarrow \rr^D$ is said to be a deep feedforward network if for every $j=1,\dots, J$ there are $d_{j}\times d_{j-1}$-dimensional matrices $A^{(j)}$ called weights and $b^{(j)}\in \rr^{d_j}$ called biases, such that $\mathcal{NN}$ admits the iterative representation 
	\begin{equation}
	\label{eq_definition_ffNNrepresentation_function}
		\begin{aligned}
			\mathcal{NN}(x) 
			\eqdef &
			A^{(J)}x^{(J)} + b^{(J)}\\
			x^{(j)} &\eqdef
			\operatorname{ReLU}\bullet( A^{(j)}
			x^{(j-1)}
			+
			b^{(j)})
			& \mbox{ for }  
			j=1,\dots,J-1,
		    \\
			x^{(0)} &\eqdef x
			,
		\end{aligned}
	\end{equation}
	where $\bullet$ denotes component-wise composition.  We now formalize the probabilistic transformer in Figure~\ref{fig:transformer_MG}.
	\begin{definition}[Probabilistic Transformer (PT)]
		\label{def_transformer}
		Fix a metric space $(\xxx,d)$ and a $D\in \nn_+$.  A \textit{probabilistic transformer} is a map $T:\xxx\rightarrow \mathcal{P}_2(\rr^D)$ with representation
		\begin{equation}
			\label{eq__def_transformer}
			T(x)
			=
			\operatorname{P-Attn}\left(
			\mathcal{NN}_w
			(u)
			,
			\left(
			    \mathcal{NN}_k(u)
			\right)_{k=1}^K
			\right)
			,
			\qquad 
			u\eqdef \Phi_{\{x_l\}_{l=1}^L}(x)
			,
		\end{equation}
		where $L\in \nn_+$, $x_1,\dots,x_L\in \xxx$, and $\mathcal{NN}_w:\rr^L\rightarrow \rr^K$ and $\mathcal{NN}_1:\rr^L\mapsto \rr^{D}\times \rr^{D\times D},\dots,\mathcal{NN}_K:\rr^L\mapsto \rr^{D}\times \rr^{D\times D}$ are ReLU feedforward networks.
	\end{definition}
	The \textit{depth} of a probabilistic transformer $T$ (with representation~\eqref{eq__def_transformer}) is defined as the maximum of the depth of each of the feedforward networks $\mathcal{NN}_w$, $\mathcal{NN}_1$, $\dots$, $\mathcal{NN}_K$.  Similarly, the \textit{width} of a transformer $T$ is defined as the maximum of the widths of each feedforward network  $\mathcal{NN}_w$, $\mathcal{NN}_1$, $\dots$, $\mathcal{NN}_K$.  These definitions are natural since each of the networks $\mathcal{NN}_w$, $\mathcal{NN}_1$, $\dots$, $\mathcal{NN}_K$ defining $T$ are independent, in the sense that their layers do not pass inputs to or from one another and their parameters have no explicit relations.  
	
	The \textit{effective dimension} of a probabilistic transformer $T$, is the maximum number of mixtures output by $T$ for any input scaled by the number of parameters defining each Gaussian probability measure in
	\[
	\mathrm{effdim}(T) := 
	K
	\,
	(D+D^2)
	,
	\]
	where we use the notation of Definition~\ref{def_transformer}.
	Thus, $d_T\eqdef K$.  
	We define the number of trainable parameters in $T$, denoted by $\mathrm{par}(T)$, as the aggregate number of \textit{trainable parameters} amongst the networks $\mathcal{NN}_w,\mathcal{NN}_1,\dots,\mathcal{NN}_K$; i.e.
	$%
	\mathrm{par}(T) 
	\eqdef
	\mathrm{par}(\mathcal{NN}_w)
	+
	\sum_{k=1}^K\, 
	\mathrm{par}(\mathcal{NN}_k)
	$.  As in \cite{JLMLR_BartlessHAveyLiawMagrabian_2019_VCBoundsReLUffNN}, the number of trainable parameters encodes ``network topology'' by only counting the number of non-zero entries, thus only counting the number of weighted and biases between any two neurons with a direct connection.  In the notation of~\eqref{eq_definition_ffNNrepresentation_function}, the number of trainable parameters of a neural network $\mathcal{NN}$ is defined as
	\[
	\mathrm{par}(\mathcal{NN}) 
	\eqdef
	\|c_0\|_0 + \sum_{j=1}^{J}\, \big(\|A^{(j)}\|_0 + \|b^{(j)}\|_0\big)
	,
	\]
	where $\|\cdot\|_0$ counts the number of non-zero entries in a matrix or vector of a given size.  
	For fully connected feedforward networks of constant width $\mathrm{par}(\mathcal{NN})$ is roughly equal to the network's depth multiplied by its width squared; but in general $\mathrm{par}(\mathcal{NN})$ will be much smaller since our networks will tend to have a \textit{sparsely connected} architecture analogously to \citep{suzuki2018adaptivity} or in the convolutional neural network theory \citep{zhou2020universality}.  If one replaces $\|\cdot\|_0$ with the Fr\"{o}benius and Euclidean norms (for the weight and biases respectively) then, $\mathrm{par}(\mathcal{NN})$ becomes a path norm \cite{neyshabur2015path,Zheng_Meng_Zhang_Chen_Yu_Liu_2019,WeinanWojtowytsch_pathnorm_ReLU2022} typically used in neural network regularization.

	\section{Main Results}
	\label{s_MainResults}
	
	Our \textit{quantitative} results concerning the embedding capabilities of probabilistic transformers in mixed-Gaussian Wasserstein spaces fall into two classes.  The first set of results concerns the possibility of embedding and the complexity of the model performing an embedding when no ``latent geometric structures'' are present in the dataset, and the latter class of results concerns the embedding of datasets with a latent geometry.  
	
	Table~\ref{tab:Complexities_paperversion} summarizes the rates in our main results.  We emphasize that explicit constants are also available and we refer the reader interested in detailed constants, and precise $\mathcal{O}$ rates instead of $\tilde{\mathcal{O}}$ to Table~\ref{tab:Complexities_explicit} in Appendix~\ref{s_explicit_network_complexities}. Note that we use the following asymptotic notation. Given $g,f:\mathbb{N}\rightarrow [0,\infty)$ will say that $g$ is $\tilde{\mathcal{O}}(f)$ if $g$ is $\mathcal{O}(\tilde{f})$ where $f = \tilde{f}\times$ logarithmic terms.  
	
	\begin{table*}[ht!]
    \centering
	\ra{1.3}
    \caption{Complexity of the probabilistic transformers performing $n$-point embeddings.  }
        \resizebox{\columnwidth}{!}{%
    		\begin{tabular}{@{}lllll@{}}
    			\cmidrule[0.3ex](){1-5}
    			\textbf{Latent Geometry} & \textbf{Effective Dimension $d_{T}$} & \textbf{Depth} $\in \widetilde{\mathcal{O}}(\,\cdot\,)$ & \textbf{Width} $\in \Omega(\,\cdot\,)$ & \textbf{Result} 
    			\\    
    			\midrule
    		    General
    		    & 
    		        $
    		        \Omega\Big(
            			\frac{ \log(
            			    \mathrm{cap}(\xxx_n,d_n))
            			}{
            			    \alpha
            			   }
        			\Big)
    		        $
    		    &
    		        $
    			        n
    			        \big(
    			            1
    			                +
    			            \log(
        			            n^{5/2}
        			            \operatorname{aspect}(\xxx_n,d_n)
    			            )
    			        \big)
    		        $
    		    & 
    		        $
    		        \max\{d_T,n^2\}
    		        $
    		   &
    		   Theorem~\ref{MAINTHEOREM_DETERMINISTIC}
    		\\
    			General 
    			& 
    			    $\mathcal{O}\big(
                    (D-2)
        			^{-2}
                    \theta_D
        			\log_2(n)
        			\big)
    			$
    			&
    			    $
    			        n^{\Theta_D}
    			        \big(
    			            1
    			                +
    			            \theta_D
    			                + 
    			            \log(
    			                \operatorname{aspect}(\xxx_n,d_n)
    			            )
    			        \big)
    			    $
    			& 
    			$
    			\max\{d_T,n^{2\theta_{D}}\}
    			$
                & 
                Theorem~\ref{MAINTHEOREM_PROBABILISTIC}
    		\\
            General - Multivariate Mixtures & 
    		$
    		\mathcal{O}\big(\frac{n^6}{(D-1)} \, \operatorname{aspect}(\xxx_n,d)^2 \big)
    		$
    		& 
    		 $ n^2 + \log\big(
    			            n^{5/2}\, \operatorname{aspect}(\xxx_n,d_n)
    			          \big)\, 
    		$
    		& 
            $
            \mathcal{O}\big(\frac{n^5(n-1)}{(D-1)} \, \operatorname{aspect}(\xxx_n,d)^2 \big)
            $ 
            & 
            Theorem~\ref{theorem:high_dimension_asymptotic_no_distortion}
            \\
    		    Discrete: Trees
    		    & 
    		        $M$
    		    &
    		        $
    			        n
    			        \big(
    			            1
    			                +
    			            \log(
        			            n^{5/2}
        			            \operatorname{diam}(\xxx_n,d_n)
    			            )
    			        \big)
    		        $
    		    & 
		        $\max\{M,n^2\}$
    		   &
    		   Proposition~\ref{prop_combinatorial_tree_embedding}
    		 \\
    		    Discrete: $2$-Hop Graphs
    		    & 
    		        $\Omega\Big(
    		            \frac{
    		                \log(1+\rho(A_{G}))
    		                }{
    		                \alpha
    		                }
    		        \Big)$
    		    &
    		        $
    			        n
    			        \big(
    			            1
    			                +
    			            \log(
        			            n^{5/2}
        			            \operatorname{diam}(\xxx_n,d_n)
    			            )
    			        \big)
    		        $
    		    & 
    		        $
    		        \max\{d_T,n^2\}
    		        $
    		   &
    		   Corollary~\ref{cor_combinatorial_graphs}
    		\\
    		    Manifold: Riemannian, Bounded Curvature
    		    & 
    		        $
    		        \Omega\Big(
    		          \frac{
    		            m^{1+\alpha}
    		          }{
    		            \alpha (1-\alpha)^{1+\alpha}
    		           }
    		        \Big)
    		        $
    		    &
    		        $
    			        n
    			        \big(
    			            1
    			                +
    			            \log(
        			            n^{5/2}
        			            \operatorname{aspect}(\xxx_n,d_n)
    			            )
    			        \big)
    		        $
    		    & 
    		        $
    		        \max\{d_T,n^2\}
    		        $
    		   &
    		   Corollary~\ref{cor_Ricci_Version}
    		\\
    		    Manifold: Riemannian and Compact
    		    & 
    		        $\Omega(d)$
    		    &
    		        $
    			        n
    			        \big(
    			            1
    			                +
    			            \log(
        			            n^{5/2}
        			            \operatorname{aspect}(\xxx_n,d_n)
    			            )
    			        \big)
    		        $
    		    & 
    		        $
    		        \max\{d,n^2\}
    		        $
    		   &
    		   Proposition~\ref{proposition_manifoldprior}\\
    			\bottomrule
    		\end{tabular}
    		}
        \label{tab:Complexities_paperversion}
    \end{table*}
	
	\subsection{Embedding Guarantees for General Datasets}
	\label{s_MainResults__ss_NoStructure}
	
	Our first main result shows that any metric space $\xxx$ can be represented in $(\mathcal{MG}_2(\rr),\mw_2)$ with a small PT guaranteeing that any fixed finite set of points can be embedded with little distortion, if we allow for minor perturbations to $\xxx$'s geometry.  By a ``small PT'', we mean that the representation implemented by the PT of Theorem~\ref{MAINTHEOREM_PROBABILISTIC} does not face the \textit{the curse of dimensionality}.  
	\begin{theorem}[Deterministic Fractional Embeddings of Large Data by Small Transformers]
		\label{MAINTHEOREM_DETERMINISTIC}
		\hfill\\ 
		There is an absolute constants $C>0$ such that, for every metric space $(\xxx,d_{\xxx})$, every finite subset $\xxx_n\subseteq \xxx$ with $n\ge 2$ points, and every ``geometric perturbation parameter'' $\frac{1}{2}<\alpha<1$, 
		there exists a probabilistic transformer $T$ for which: for every $x,\tilde{x}\in \xxx$
		\[
    		    d_{\xxx}^{\alpha}(x,\tilde{x})
    		\leq
    		    \mathcal{W}_2(T(x),T(\tilde{x}))
    		\leq
    		    \mw_2(T(x),T(\tilde{x}))
    		\leq 
        		\big\lceil
        		C
        		C_{\xxx_n}
        		\big\rceil
        		d_{\xxx}^{\alpha}(x,\tilde{x})
		,
		\]
		where $C_{\xxx}\eqdef \left(
		\frac{
			12
			\log
			\left(
			\mathrm{cap}(\xxx_n)
			\right)
		}{(1-\alpha)}
		\right)^{1+\alpha}$.  
		Furthermore, $T$'s effective dimension, depth, and width are recorded in Table~\ref{tab:Complexities_paperversion} (with explicit constant given in Table~\ref{tab:Complexities_explicit}). 
	\end{theorem}

	Independently of any machine learning model, results such as \cite{Bartal2005RamseyPhenomena} confirm there are $n$-point metric spaces $\xxx_n$ with the property that \textit{any} finite subset $\tilde{\xxx_n}$ that bi-Lipschitz embeds into $\ell^2$ with distortion $D$ cannot contain more than $n^{1-\vartheta_D}$ points, for some $\vartheta_D>0$ depending only on $D$.  Equivalently, if we uniformly and independently pick a pair of points in $\xxx_n$ at random then they belong to $\tilde{\xxx}_n$ with probability at-most $\frac1{n^{2\vartheta_D}}$.  This places a fundamental limitation on the number of pairs of points which can be bi-Lipschitz embedded with distortion $D$ using \textit{any model}.  In particular, this type of limitation must translate to the PT model.  In contrast, this is not the case, when working with almost bi-Lipschitz ($\alpha$-H\"{o}lder with $\alpha \approx 1$) embeddings.  
	
	Nevertheless, the next result provides a probabilistic guarantee that a PT embeds pairs of points in the set of landmarks $\xxx_n$ with a given distortion.  As we may expect, the likelihood that any two given data points are embedded with a \textit{prescribed (maximal) distortion} is proportional to the size of the distortion.  In other words if we do not allow any perturbation to the landmark set's geometry then, with high probability, most pairs of points can be bi-Lipschitz embedded into $(\mathcal{MG}_2(\rr),\mw_2)$ with large distortion and with small probability most pairs of points in the landmark set can be embedded with small distortion. 
	\begin{theorem}[PAC-Type Embeddings of Large Data by Small Transformers]
		\label{MAINTHEOREM_PROBABILISTIC}
		Let $(\xxx,d_{\xxx})$ be a metric space, fix a finite subset $\xxx_n$ of $\xxx$ with $n \ge 2$ points, and let $\pp$ be the uniform probability measure on $\xxx_n$.  
        There is a $\frac{1}{e^{2}}<\delta_n<1$ depending only on $n$ such that for every $\delta_n< \delta<1$, there exists a ``scale'' $s>0$ and a probabilistic transformer $T:(\xxx_n,d_n)\hookrightarrow (\mathcal{MG}_2(\rr),\mw_2)$ satisfying
		\[
		\pp\biggl(
    		s
    		d_{\xxx}(x,\tilde{x})
    		\leq
    		\mathcal{W}_2(T(x),T(\tilde{x}))
    		    \,\mbox{ and }\,
    		\mw_2(T(x),T(\tilde{x}))
    		\leq 
    		s
    		D(D-1)
    		d_{\xxx}(x,\tilde{x})
		\biggr)
		\ge 
		\delta
		,
		\]
        Abbreviate $\lambda \eqdef \log_n(\sqrt{\delta})$.  The distortion parameter $D>2$ is given explicitly by
        \[
        	    D 
	        =
	        -2
	        \,
	        \frac{
        	    \big(
        	        1 + \lambda
        	    \big)^{
        	        \frac{
        	            1 + \lambda
        	        }{
        	            \lambda
        	        }
        	    }
        	 }{
        	    \lambda
        	 }
        \]
    Furthermore, $T$'s effective dimension, depth, and width are recorded in Table~\ref{tab:Complexities_paperversion} (explicit constants are in Table~\ref{tab:Complexities_explicit}).  
	\end{theorem}  
	\begin{remark}[The Minimum Valid Probability $\delta_n$ in Theorem~\ref{MAINTHEOREM_PROBABILISTIC}]
	    The quantity $\delta_n$ in Theorem~\ref{MAINTHEOREM_PROBABILISTIC} is defined by
	    $
	        \delta_n 
	    \eqdef 
	        \max\big\{
	            \frac{1}{e^2}
	                ,
	            \tilde{\delta}_n
	            \big\}
	    $ where $\tilde{\delta}_n$ is the unique $\frac1{e^2}\le \delta <1$ solving %
	  $
	     \log_n(\sqrt{\delta}) 
	    = 
    	    \big(
    	        1 + \log_n(\sqrt{\delta})
    	    \big)^{
    	        \frac{
    	            1 + \log_n(\sqrt{\delta})
    	        }{
    	            \log_n(\sqrt{\delta})
    	        }
    	    }.
	  $
	\end{remark}
	
	Even though the multivariate transport distances are more expensive to compute, it is natural to ask whether there are any advantages in representing metric spaces as mixtures of multivariate instead of univariate Gaussians. The answer is \textit{yes}, as shown by the following deep neural analogue of the main finding of \cite{NaorNeimanAndoniSnowflake2018AnnalsSENSQ}.  The result shows that if one considers multivariate Gaussian mixtures, there are probabilistic transformers that implement metric embeddings of arbitrarily low distortion on any finite subset of a given metric space; however, the effective dimension of these embeddings does become large.  We note that it still depends only polynomially on the number of points being embedded.    
	
	\begin{theorem}[Distortionless Fractional Embeddings into Multivariate Gaussian Mixtures]
	\label{theorem:high_dimension_asymptotic_no_distortion}
	Let $(\xxx,d)$ be a metric space and $\xxx_n\subseteq \xxx$ a finite subset with at least two points. For every distortion $D > 1$ there exists a probabilistic transformer $T:\xxx \rightarrow \mathcal{GM}_2(\mathbb{R}^3)$ for which there is a ``scale'' $s>0$ such that for every $x,\tilde{x}\in \xxx_n$ 
	\begin{equation}
	   \label{eq:high_dimension_asymptotic_no_distortion__guarantee}
    	        s\,d(x,\tilde{x})^{1/2}
	        \le 
    	        \mathcal{MW}_2\big(
    	            T(x)
    	        ,
    	           T(\tilde{x})
    	        \big)
    	    \le 
    	        D\,s\,d(x,\tilde{x})^{1/2}
    	   .
	    \end{equation}
	Moreover, the complexity of $T$ s given in Table~\ref{tab:Complexities_paperversion}. 
    \end{theorem}
	Theorem~\ref{theorem:high_dimension_asymptotic_no_distortion} shows that one can achieve arbitrarily low distortion with embeddings into the space of multivariate Gaussian mixtures.  In the more computationally efficient setting, where we consider embeddings into univariate Gaussian mixtures,
	Theorems~\ref{MAINTHEOREM_DETERMINISTIC} and~\ref{MAINTHEOREM_PROBABILISTIC} provide efficient embedding guarantees for arbitrary finite subsets of general metric spaces with small but non-vanishing distortion.  It is natural to ask whether we can obtain stronger guarantees if $\xxx_n$ is a subset of a well-behaved $\xxx$ with a ``geometric prior''. Indeed, we show below that in this case the guarantees in Theorem $1$ can be improved to deterministic bi-Lipschitz guarantees.
	
	\subsection{Consequences: Improved Guarantees When \texorpdfstring{$(\xxx_n,d_n)$}{The Dataset} has a Geometric Priors}
	\label{s_MainResults__ss_Latent_Sructure}
	Theorems~\ref{MAINTHEOREM_PROBABILISTIC} and~\ref{MAINTHEOREM_DETERMINISTIC} can be straightforwardly applied to any finite metric space, in the sense that there is no $(\xxx_n,d_n)$ quantities to ``plug-in'', and this is true regardless of if $d_n$ encoded some additional ``latent geometric priors''.  Nevertheless, a fully explicit use of Theorem~\ref{MAINTHEOREM_DETERMINISTIC} requires an explicit estimate on $(\xxx_n,d_n)$'s metric capacity.  Conveniently, such estimates can be derived when $(\xxx_n,d_n)$ has some latent geometry.  
	
	We consider two different types of geometries from which the $n$-point set of landmarks is drawn.  For instance, one may alternatively interpret $\xxx_n$ as a training dataset.  The first is the smooth geometries arising from suitable Riemannian manifolds and the second is discrete geometries arising from well-behaved types of combinatorial graphs.

	\subsubsection{Smooth Geometric Priors}
	\label{s_MainResults__ss_Latent_Sructure___sss_Smooth}
	In the case of a ``smooth geometric prior'', meaning that points are drawn from a Riemannian manifold with suitable curvature, the required number of mixtures and the embedding distortion can be made independent of the number of datapoints $n$.  A fortiori, the necessary number of mixtures is proportional to the latent Riemannian manifold's dimension.  
	
	\begin{proposition}[Representation of Riemannian Manifolds with bi-Lipschitz Guarantees]
		\label{proposition_manifoldprior}
		Fix $d\in \nn_+$, let $(M,g)$ be a $d$-dimensional Riemannian manifold with geodesic distance $d_g$, and let $K\subseteq M$ be compact.  For any finite subset $\xxx_n\subseteq K$ with $n\ge 2$ points, there is a constant $C_K>0$ \textit{depending only on $K$} and a PT as in Theorem~\ref{MAINTHEOREM_DETERMINISTIC} satisfying: for every $x,\tilde{x}\in \xxx$ the following holds
		\[
		    d_{g}(x,\tilde{x})
		\leq
		    \mathcal{W}_2(\varphi(x),\varphi(\tilde{x}))
		\leq 
		    \mw_2(\varphi(x),\varphi(\tilde{x}))
		\leq 
    		C_K 
    		d_{g}(x,\tilde{x})
		.
		\]
	Furthermore, $T$'s effective dimension, depth, and width are recorded%
	\footnote{Explicit constant are recorded in Table~\ref{tab:Complexities_explicit}.} in Table~\ref{tab:Complexities_paperversion}. 
	\end{proposition}
	Proposition~\ref{proposition_manifoldprior} provides bi-Lipschitz embedding guarantees for finite subsets of a general Riemannian manifold via a small PT.  If the latent Riemannian manifold has bounded Ricci curvature, then we find that the constant $C_K$ can be made explicit and independent of any fixed compact subset of the Riemannian manifold.  For this, we apply Theorem~\ref{MAINTHEOREM_DETERMINISTIC}, and we turn to representations by our PT model with $n$-point (non-Lipschitz) bi-H\"{o}lder embedding guarantees.  
	
	\begin{corollary}[Representation of Riemannian Manifolds with Bounded Ricci Curvature]
		\label{cor_Ricci_Version}
		Fix $d\in \mathbb{N}_+$, let $(M,g)$ be a complete $m$-dimensional Riemannian manifold lower-bounded Ricci curvature and with geodesic distance $d_g$.  
		There is an absolute constant $\tilde{C}>0$ such that, for every finite subset $\xxx_n\subseteq M$ with $n\ge 2$ points, there is a PT as in Theorem~\ref{MAINTHEOREM_DETERMINISTIC} satisfying
		\[
		    d_{g}^{\alpha}(x,\tilde{x})
		\leq
		    \mathcal{W}_2(T(x),T(\tilde{x}))
		\leq
		    \mw_2(T(x),T(\tilde{x}))
		\leq 
    		\left\lceil
    		C
    		\Big(
    		\frac{m}{1-\alpha}
    		\Big)^{1+\alpha}
    		\right\rceil
    		d_{g}^{\alpha}(x,\tilde{x})
		.
		\]
		Furthermore, $T$'s effective dimension, depth, and width are recorded in Table~\ref{tab:Complexities_paperversion} (with explicit constant given in Table~\ref{tab:Complexities_explicit}). 
	\end{corollary}
	As a concrete example of Corollary~\ref{cor_Ricci_Version} we obtain the following embedding guarantee for spherical data.  It is worth noting that the sphere is both a simple yet interesting example since even in a Euclidean representation space, the result of \cite{RobinsonSphereNoFlat_2006} implies that the sphere cannot be embedded into any Euclidean space with no distortion\footnote{We note this, to clarify the misconception that the ``embeddings'' of \cite{Nash_1954_c1_Embeddings} are not in the metric sense (as considered in this article) but are in a different (Riemannian sense).}.    
	\begin{example}[Embedding Spherical Data]
	\label{s_Intro_Background__ssRiemanninan_manifold_LBRicci_Curvaturesphere}
	Consider the $m$-dimensional sphere $S^m\eqdef \{x\in \rr^{m+1}:\, \|x\|=1\}$ with Riemannian geometry inherited from its inclusion into the Euclidean space $\rr^{m+1}$.  In this geometry, the distance between any two $x_1,x_2\in S^m$ is $d_g(x_1,x_2)=\arccos(x_1^{\top}x_2)$ and $S^m$ has lower-bounded Ricci curvature (see \cite[page 276]{jost2008riemannian}).  Therefore, for every $n$-point subset $\xxx_n\subseteq S^m$ ($n\in \nn_+$) there exists a PT $T$ with width, depth, and number of mixture components as in Corollary~\ref{cor_Ricci_Version} satisfying: for every $x,\tilde{x}\in S^m$ it holds that
		\[
		    d_{g}^{3/4}(x,\tilde{x})
		\leq
		    \mw_2(T(x),T(\tilde{x}))
		\leq 
    		\left\lceil
    		C
    		\Big(
    		4 m
    		\Big)^{3/4}
    		\right\rceil
    		d_{g}^{3/4}(x,\tilde{x})
		,
		\]
	and $C>0$ is independent of $n$ and of $m$.
	\end{example}
	
	\subsubsection{Combinatorial Priors}
	\label{s_MainResults__ss_Latent_Sructure____sss_combinatorialPriors}
	A particularly useful class of combinatorial graphs are \textit{trees}; i.e. graphs in which no path originating at one vertex can ever come back to itself without passing the same edge at least twice.  The simplicity of the combinatorial geometry of trees begs the question: can the transition between the deterministic embeddings of Theorem~\ref{MAINTHEOREM_DETERMINISTIC} (with geometric perturbation parameter $\alpha<1$) and the probabilistic embeddings of Theorem~\ref{MAINTHEOREM_PROBABILISTIC} (with geometric perturbation parameter $\alpha=1$) be avoided, and can the simple geometry of combinatorial tree allows us to obtain purely deterministic embedding results?  The following result gives an affirmative answer.  Furthermore, an embedding with $\mathcal{O}(n^{1/(M-1)})$ distortion is possible by using a probabilistic transformer with at most $\mathcal{O}(M)$ mixture components.  
	
	\begin{remark}[Simplified Result Formulation for Graphs]
	\label{rem_simplified_statement_graphs}
	    For simplicity, we consider the case where the graph in question is finite and where each point in $\xxx$ is a landmark.  However, one can easily see how the statement extends to the general case considered in Theorem~\ref{MAINTHEOREM_DETERMINISTIC}.  
	\end{remark}
	
	\begin{proposition}[Bi-Lipschitz Representations of Combinatorial Trees]
		\label{prop_combinatorial_tree_embedding}
		Let $G=(V,E)$ be an $n$-vertex tree and let $(V,d_G)$ be its associated metric space (as in Example~\ref{s_Intro_Background__ssCombinatorialGraphs}). 
		There is an absolute constant $C>0$ such that, for any positive integer $M$ where $M,n\geq 2$, there exists a PT $T$ with width and depth as in Theorem~\ref{MAINTHEOREM_DETERMINISTIC} and such that
		\[
    		    d_{G}(x,\tilde{x})
    		\leq
    		    \mathcal{W}_2(T(x),T(\tilde{x}))
    		\leq 
    		    \mw_2(T(x),T(\tilde{x}))
    		\leq 
        		C n^{1/(M-1)}
        		d_{G}(x,\tilde{x})
		.
		\]
		Furthermore, $T$'s effective dimension, depth, and width are recorded in Table~\ref{tab:Complexities_paperversion} (with explicit constant given in Table~\ref{tab:Complexities_explicit}). 
	\end{proposition}
	Therefore, for combinatorial graphs with simple geometric priors, the deterministic to probabilistic transition between Theorem~\ref{MAINTHEOREM_DETERMINISTIC} and~\ref{MAINTHEOREM_PROBABILISTIC} can be avoided.  Next, we consider where we wish to represent data with manifold priors using the probabilistic transformer architecture.  
	
	Proposition~\ref{prop_combinatorial_tree_embedding} was derived using the same method of proof as for Theorem~\ref{MAINTHEOREM_DETERMINISTIC}, but specialized to the geometry of combinatorial trees.  Instead, the next result is a consequence of Theorem~\ref{MAINTHEOREM_DETERMINISTIC}.
	
	Many well-studied graphs are such that each vertex is reachable by hopping over at most one other.  We refer to these types of graphs as ``$2$-hop graphs''.  Examples of $2$-hop combinatorial graphs include, friendship graphs \citep{ErdosSos_Friendship1966}, complete bipartite graphs as in Mantel's Theorem \citep{Erdos_1964TuranTheorem}, cocktail-party graphs \citep{Jungerman_cocktailgraph_1978}, wheel graphs \citep{BuckleyHarary_WheelGraphs_1988}, and various others. 
	
	To state our guarantees for $2$-hop graphs we require some additional terminology.  The adjacency matrix $A_G$ of a graph $G=(V,E)$, whose vertices we enumerate by $V=\{v_i\}_{i=1}^n$, is the $n\times n$-matrix with binary with $(A_{G})_{i,j}=1$ if and only if there is an edge between the vertices $v_i$ and $v_j$.  The spectral radius of $A_G$, denoted by $\rho(A_G)$, is the largest absolute eigenvalue of $A_G$.  
	\begin{corollary}[Representation for $2$-Hop Combinatorial Graphs]
		\label{cor_combinatorial_graphs}
		Let $G=(V,E)$ be an $n$-vertex $2$-hop graph and let $(V,d_G)$ be its associated metric space (as in Example~\ref{s_Intro_Background__ssCombinatorialGraphs}).  
		Fix $\frac{1}{2}<\alpha <1$, and let $T$ be the PT of Theorem~\ref{MAINTHEOREM_DETERMINISTIC}.  Each $T(x)$ is comprised of no more than $
		\lceil 12\,C\alpha^{-1}\log(1+\rho(A_{G}))\rceil
		$ mixture components and satisfies: for each $x,\tilde{x}\in \xxx_n$ the following holds
		\[
    		    d_{G}^{\alpha}(x,\tilde{x})
    		\leq
    		    \mathcal{W}_2(T(x),T(\tilde{x}))
    		\leq 
    		    \mw_2(T(x),T(\tilde{x}))
    		\leq 
        		\left\lceil
        		C
        		\left(
        		\frac{
        			12\log(1+\rho(A_{G})
        		}{
        			1-\alpha
        		}
        		\right)^{1+\alpha}
        		\right\rceil
        		d_{G}^{\alpha}(x,\tilde{x})
		.
		\]
			Furthermore, $T$'s effective dimension, depth, and width are recorded in Table~\ref{tab:Complexities_paperversion} (with explicit constant given in Table~\ref{tab:Complexities_explicit}). 
	\end{corollary}
	The embedding guarantee in Corollary~\ref{cor_combinatorial_graphs} can be further approximated using additional information about $G$'s connectivity.  Specifically, define $G$'s \textit{maximum degree} to be the maximum number of edges connected to any of its vertices, denoted by $\operatorname{deg}^+(G)$, and its \textit{minimum degree} to be the smallest number of edges connected to any of its vertices, denoted by $\operatorname{deg}^-(G)$.  
	
	We use the following example to showcase the explicit constants and PT complexity estimates which are derived through our analysis (also recorded Appendix~\ref{s_explicit_network_complexities}).  
	
	\begin{example}[Embedding of $2$-Hop Graphs in Terms of Connectivity Information]
	\label{s_Intro_Background__ssRiemanninan_manifold_LBRicci_Curvaturefurther_Connectivity}
	Consider the setting of Corollary~\ref{cor_combinatorial_graphs}, with $\alpha=\frac{3}{4}$, $n\ge 12$, and let $T$ be the probabilistic transformer described by that result.  The upper-bound on the spectral radius $\rho(A_G)$ given in \citep[Theorem 2.7]{DasKumar_Spectral_Radii} implies the following embedding guarantee
	\[
	    d_{G}^{3/4}(x,\tilde{x})
	\leq 
		\mw_2(T(x),T(\tilde{x}))
	\leq 
		\left\lceil
		C
		C_G
		\right\rceil
		d_{G}^{3/4}(x,\tilde{x})
	,
	\]
	where $
	    C_G
	=
	    \left(
			48\log\left(1
			    +
			   \sqrt{2 (\#E) (n-1) \operatorname{deg}^+(G) + (\operatorname{deg}^+(G)-1)\operatorname{deg}^-(G)}
			   \right)
		\right)^{7/4}.
	$
	Furthermore, the ``complexity'' of $T$ can be estimated in terms of $G$'s degree, its number of edges, and its diameter; namely,
	\begin{align*}
	\mathrm{effdim}(T) &=  
	\mathcal{O}\biggl(
    			    \log\Big(
    			        1+\sqrt{2 (\#E) (n-1) \operatorname{deg}^+(G) + (\operatorname{deg}^+(G)-1)\operatorname{deg}^-(G)}
    			      \Big)
    			\biggr) \\
	\mathrm{depth}(T) &=
	\mathcal{O}\left(
	n
	\Biggl\{
        1+
    \sqrt{n\log(n)}
        \,
     \Bigg[
        1
            +
       \frac{\log(2)}{\log(n)}\,
       \left(
            c
                +
         \frac{
	         \log\big(
	            n^{5/2}\, 
	            \operatorname{diam}(\xxx_n,d_n)
	          \big)
         }{
            \log(2)
         }
       \right)_+
     \Bigg]
		\Biggr\}
	\right).
	\end{align*}
	\end{example}

	We conclude our theoretical analysis by illustrating the fundamental limits of what can be achieved when embedding \textit{arbitrary} finite metric spaces.  Our results highlight the limitations of our model when globally embedding an $(\xxx_n,d_n)$ with no geometric prior into $(\mathcal{MG}_2(\rr),\mw_2)$, as well as the mathematical limitations of any model embedding such a space into any smooth and finite-dimensional representation space.  
	\subsection{Limitations to Global Embedding Capacity for Datasets with no Geometric Priors}
	\label{s_Main_ss_Supporting_Results}
	It is natural to ask if the transition between the probabilistic bi-Lipschitz embedding of Theorem~\ref{MAINTHEOREM_PROBABILISTIC} and the deterministic bi-H\"{o}lder embedding of Theorem~\ref{MAINTHEOREM_DETERMINISTIC} is just an artifact of our analysis.  The answer is no.  The next result confirms as the phenomenon is generic and persists for any ``regular'' finite-dimensional embedding.

	We begin by considering the case of complete Riemannian manifold $\mathcal{R}$ of bounded negative curvature for which the frequently used hyperbolic spaces in non-Euclidean representation learning are prototypical.  Our next result combines  the results of \cite{naor2006markov} and \cite{biggs1983sextet} to construct an explicit sequence of bipartite combinatorial graphs (namely the sextet graphs introduced by \cite{biggs1983sextet}) which cannot be embedded into $\mathcal{R}$ with non-diverging distortion.  
	
	\begin{proposition}[Impossibility of Manifold Embedding of Fixed Finite Dimension]
	\label{prop_impossibility_theorem_negativelyCurved}
	Let $\mathcal{R}$ be a complete Riemannian manifold with negative sectional curvature bounded in $[-C,-c]$ where $0<c\le C$.  Then there exists a constant $c_{\mathcal{R}}>0$ depending only on $\mathcal{R}$ such that, for every prime integer $n\ge 2$ satisfying $n \bmod 16 \in \{\pm 3, \pm 5, \pm 7\}$ there is a combinatorial graph $(\xxx,d_{\xxx})$ with $n$-vertices and degree exactly $3$ such that the every bi-Lipschitz embedding $f:\xxx \rightarrow \mathcal{R}$ must have distortion $D$ bounded below by
\begin{equation}
    \label{eq:prop_impossibility_lower_bound_explicit}
    	    	D
		        \ge 
		    c_{\mathcal{R}}
		        \sqrt{
		        \frac{4}{3} \log_2(n) - 2
		   }
    .
\end{equation}
	\end{proposition}
	
	\begin{remark}[Improved Lower Bound Via Expander Graphs]
	One can explicitly obtain larger lower bounds than in~\ref{eq:prop_impossibility_lower_bound_explicit} by using the result of \cite{margulis1982explicit,margulis1988explicit} (both discovered independently and improving a result of \cite{erdregukre}) stating that: for every $\delta \ge 3$ there is a diverging sequence $\{n_k\}_{k\in \nn}$ and graphs $\{G_{n_k}\}_{n\in \nn}$ each with $n_k$-vertices such that each $G_{n_k}$ has girth at-least 
	$
	\frac{3}{4}\log_{\delta-1}\big(n_k\big)
	.
	$  
	Using this sequence of expander graphs, instead of the sequence of Sextet graphs used in the proof of Proposition~\ref{prop_impossibility_theorem_negativelyCurved}, implies that there is no bi-Lipschitz embedding of $G_{n_k}$ into $\mathcal{R}$ with distortion less that
	\[
	\mathcal{O}\Big( 
	    (\delta - 2)\, \sqrt{\frac{4}{3}\log_{\delta -1}(n)+2}
	\Big)
	.
	\]
	The advantage of the construction in Proposition~\ref{prop_impossibility_theorem_negativelyCurved} is a simple explicit class of graphs.  
	\end{remark}
	
	We complement Proposition~\ref{prop_impossibility_theorem_negativelyCurved} by considering the case of compact Riemannian manifolds or Euclidean spaces.   We again find that there is a sequence of pathological finite metric spaces which cannot be bi-Lipschitz embedded into such representation spaces with non-diverging distortion.

	\begin{proposition}[Impossibility of Manifold Embedding of Fixed Finite Dimension]
		\label{prop_impossibility_theorem_positivecurvature}
	Let $\mathcal{R}$ be a Riemannian manifold which is bi-Lipschitz embedded in the Hilbert space $\ell^2$.  There is a constant $c_{\mathcal{R}}>0$ (depending only on $\mathcal{R})$ such that for every positive integer $n$, there is an $n$-point metric space $(\xxx,d_{\xxx})$ for which every bi-Lipschitz embedded $f:\xxx\rightarrow \mathcal{R}$ has distortion $D$ satisfying
	\[
	    D
	        \ge
	    c_{\mathcal{R}}
	    \frac{\log(n)}{\log\big(\log(n)\big)}
	.
	\]
	\end{proposition}
	
	\begin{example}[Compact Manifolds Satisfy Proposition~\ref{prop_impossibility_theorem_positivecurvature}]
    If $\mathcal{R}$ is a compact Riemannian manifold then Whitney's embedding theorem implies that there is a smooth diffeomorphism (onto its image) $\varphi:\mathcal{R}\rightarrow \mathbb{R}^{2d}$.  Therefore, Rademacher's theorem and the compactness of $\mathcal{R}$ implies that $\varphi$ is a bi-Lipschitz embedding.  Thus, $\mathcal{R}$ satisfies the conditions of Proposition~\ref{prop_impossibility_theorem_positivecurvature} by the Whitney embedding theorem.  
	\end{example}
	
	The PAC-type bound of Theorem~\ref{MAINTHEOREM_PROBABILISTIC} can be reformulated as a function of the distortion $D$ (instead of the probability $\delta$).  Doing so allows us to compare our bi-Lipschitz embedding results with the above two impossibility results.  We now record this inverted formulation.    
	
	\begin{remark}[Inverted Form of Theorem~\ref{MAINTHEOREM_PROBABILISTIC} where Distortion is Given Instead of $\delta$]
	\label{remark_PAC_explicit_lowerbound}
	The relation ship between the distortion $D$ and the ``probabilistic level'' $0<\delta<1$ in Theorem~\ref{MAINTHEOREM_PROBABILISTIC} can be inverted.  In that case, our proof implies that for any prespecified distortion level $D>2$ there is a PT $T$ as in Theorem~\ref{MAINTHEOREM_PROBABILISTIC} satisfying:
	\[
    		\pp
    		\Big(
        		s
        		d_{\xxx}(x,\tilde{x})
        		\leq
        		\mathcal{W}_2(T(x),T(\tilde{x}))
        		    \,\mbox{ and }\,
        		\mw_2(T(x),T(\tilde{x}))
        		\leq 
        		s
        		D(D-1)
        		d_{\xxx}(x,\tilde{x})
    		\Big)
		\ge 
    		n^{-2+2{
    		\theta_D
    		}}
		,
		\]
		where $\theta_D$ is the unique solution to
		$
		    \frac{2}{D}
		        =
		    (1-\theta)^{\theta/(1-\theta)}
		$ over $[1-2 e/
        D,1)$. 
	\end{remark}
	
	Juxtaposing either of the impossibility results in Proposition~\ref{prop_impossibility_theorem_positivecurvature} or Proposition~\ref{prop_impossibility_theorem_negativelyCurved} with Theorem~\ref{MAINTHEOREM_DETERMINISTIC} (as formulated in Remark~\ref{remark_PAC_explicit_lowerbound}), we find that, even if there are $n$-point metric spaces which simply cannot be well-represented in most Riemannian representation spaces, we can always identify a subset of \textit{any such} space which \textit{can} be arbitrarily well-represented in $\mathcal{MW}_2(\rr)$.  Furthermore, the embeddings are explicitly implemented by PT models whose required number of parameters we now know.  The distortion $D>2$ controls the size of this subspace we are willing to accept.  For instance when comparing against~\eqref{prop_impossibility_theorem_negativelyCurved} we can select any distortion $D\in \big( 2, c_{\mathcal{R}}\sqrt{\frac{4}{3} \log_2(n) - 2}\big]$ and always deduce the existence of an embedding on a subset of those sextet graphs (this is not necessarily the case for a complete Riemannian representation space with pinched negative curvature).

		\section{Experiments}
	\label{s_Experiments}

	In this section we complement the theoretical embedding guarantees from Section \ref{s_MainResults} by preliminary computer experiments on synthetic data. We show that the proposed feature maps can indeed be trained in a standard deep learning framework, that the theoretical advantages of PT mixture-Wasserstein embeddings over Euclidean and hyperbolic carry over to practice, and that the PT-based feature maps generalize beyond $\mathcal{X}_n$.	We compare our proposed representation maps to learned representation maps in Euclidean and hyperbolic space when representing trees, random graphs and graphs ``sampled from'' Riemannian manifolds.\footnote{More precisely, on neighborhood graphs constructed over points sampled from manifolds.}${}^,$\footnote{The Python codes used to produce the results of this section are available at \url{https://github.com/swing-research/Universal-Embeddings}.}  
        
    Our first experiment compares the capacity of deep neural embedding in the space $(\mathcal{MG}_2(\mathbb{R}), \mathcal{WG}_2(\mathbb{R}))$ to embed combinatorial trees against the state of the art; namely embeddings into the hyperbolic plane.  
    %
    Our second experiment examines the efficiency with which PTs utilize dimension.  We compare how well the PT embeds points on a simple high-dimensional compact Riemannian manifold, namely the sphere, to how well a standard feed-forward network can embed those same points into a Euclidean space of equal dimension.  
    

    
    The probabilistic transformer we implement is illustrated in Figure \ref{fig:transformer_MG}. 
    The weights of the network are computed following \cite{WassersteinGaussianMixtures}, by minimizing
    	\begin{equation}
    	\label{eq_objective_MixtureGaussian}
    	\tag{Loss}
    	\mathcal{L}_{\mathrm{PT}}(\theta) = \sum_{\mathclap{x,y\in\mathcal{X}_{\mathrm{train}}}} \left( \mw_ 2^2(T_\theta(x), T_\theta(y)) - d_\mathcal{X}(x,y)^{2\alpha}  \right)^2.
    \end{equation}
    In practice we set $\alpha=1$ as it does not have a strong influence on empirical performance.  
    We conjecture that this is due to the influence of sampling and optimization which result in an ``error floor''. 
    
     \paragraph{Numerical parameters} All networks are trained by the Adam optimizer in \texttt{pytorch}, with weight decay parameter $10^{-6}$, initial learning rate $10^{-4}$ and final learning rate $10^{-6}$.
    
    \subsection{Hyperbolic vs.\ Gaussian Mixtures Geometry: Embedding Trees} 
    \label{s_Numerics__ss_Trees}
    
    Hyperbolic embeddings are the state-of-the-art when embedding metric trees.  Indeed, results such as \cite{sarkar2011low} show that two dimensions is enough to embed any metric tree with low distortion whereas finite trees can only be embedded into $d$-dimensional Euclidean space with low distortion if $d$ is large \cite{Gupta_Trees_Quantitiativefinitedimembeddings_ACM}.  
    
    We compare the ability of $(\mathcal{MG}_2(\rr),\mw_2)$'s geometry to accommodate deep neural embeddings of metric trees against that of the hyperbolic plane. We implement the embedding maps using (universal \cite{kratsios2021_GDL,acciaio2022metric}) overparamterized neural network models, in an effort to minimize the chance that differences in expressivity obscure the geometric effects. As we show below, the additional flexibility of $(\mathcal{MG}_2(\rr),\mw_2)$ over the hyperbolic plane is manifested even when embedding regular binary trees.  


    \subsubsection*{Embedding into the Hyperbolic Plane}

    We first consider the hyperbolic plane \cite{ganea2018hyperbolic,Ganea_embeddin_2021_AAAI} which theoretically embeds trees with arbitrarily small distortion \cite[Theorem 6]{sarkar2011low}. In the next section we look at a higher-dimensional hyperbolic space of the same dimension as the number of the degrees of freedom of our mixture-Wasserstein embedding. 
    
    The hyperbolic space has several isometric (up to a scalar) representations. In this section, we first rely on its canonical (in the sense of \cite{ChenstovTheorem_ExponentialFamilies2018}) \textit{information geometric} representation\footnote{Because it shares many similarities with our embedding space $(\mathcal{MG}_2(\rr),\mw_2)$}. The set of \textit{non-degenerate Gaussian} probability measures on $\rr$ is defined by 
    \[
        \mathcal{G}_1^+ \eqdef  \left\{ \mathcal{N}(\mu,\sigma)\in \mathcal{P}_1(\rr) :\, \frac{N(\mu,\sigma)}{dx} \propto \exp\left( -\frac{(x - \mu)^2}{2\sigma^2}\right),  \mu \in \rr, \sigma \in (0,\infty) \right\}.
    \]
    This set can be made into a Riemannian manifold with the Riemannian metric defined by the \textit{Fisher information matrix} \cite{amari2016information}.
    As shown in \citep[Equation (7)]{COSTA201559}, the geodesic distance (called the Fisher--Rao distance) between two Gaussian probability measures $N(\mu_1,\sigma_1),N(\mu_2,\sigma_2) \in \mathcal{G}_1^+$ in this Riemannian geometry is
    \[
    	d_F\left( N(\mu_1,\sigma_1) , N(\mu_2,\sigma_2) \right) =  \sqrt{2}\ln\left(\frac{\|(\frac{\mu_1}{\sqrt{2}},\sigma_1)-(\frac{\mu_2}{\sqrt{2}},-\sigma_2)\|+\|(\frac{\mu_1}{\sqrt{2}},\sigma_1)-(\frac{\mu_2}{\sqrt{2}},\sigma_2)\|}{\|(\frac{\mu_1}{\sqrt{2}},\sigma_1)-(\frac{\mu_2}{\sqrt{2}},-\sigma_2)\|-\|(\frac{\mu_1}{\sqrt{2}},\sigma_1)-(\frac{\mu_2}{\sqrt{2}},\sigma_2)\|}\right).
    \]
    Moreover, importantly for our comparison, $d_F$ is (up to a constant factor of $\sqrt{2}$) equal to the hyperbolic distance on $\mathbb{H}^2\eqdef\rr\times (0,\infty)$ which is precisely the parameter space of $\mathcal{G}_1^+$ as well as the upper half-plane model of the hyperbolic space. There are two reasons to first look at the 2D hyperbolic embeddings. First, they are easy to visualize; second, it allows us to some extent to dissociate the effects of geometry and expressivity and ease of optimization of neural networks.
    We benchmark our representations against the metric transformer model of \cite[Example 11]{acciaio2022metric} (which for clarity we call the $\mathcal{G}_1^+$-transformer) since that model is universal \cite[Theorem 3.6]{acciaio2022metric}. It is a variant of hyperbolic neural networks \cite{Ganea_embeddin_2021_AAAI}, which are known to be universal \citep[Corollary 3.16]{kratsios2020non}.  
    
    The weights of the metric transformer $T_\theta$ are obtained by minimizing
    \begin{equation}
    	\label{eq_objective_hyperbolic}
    	L_{\mathcal{G}_1^+}(\theta) = \sum_{\mathclap{x,y\in\mathcal{X}_{\mathrm{train}}}} \Big( d_F( T_\theta(x), T_\theta(y))^2 - d_\mathcal{X}(x,y)^2  \Big)^2.
    \end{equation}  
    The main difference with our model is that the metric transformer makes convex combinations in the parameter space of Gaussian measures, not in the space of probability measures itself.
    
    \def\sz{3.5cm}
    \def\szsub{4cm}
    \begin{figure*}[ht]
    	\centering
    	\begin{subfigure}[t]{0.9\textwidth}
    		\centering
        	\includegraphics[origin=c,height=\szsub]{Graphics/Transformer_MG.pdf}
        	\caption{}  
        	\label{fig:transformer_MG}
    	\end{subfigure}
    	\begin{subfigure}[t]{0.6\textwidth}
    		\centering
        	\includegraphics[width=1\textwidth]{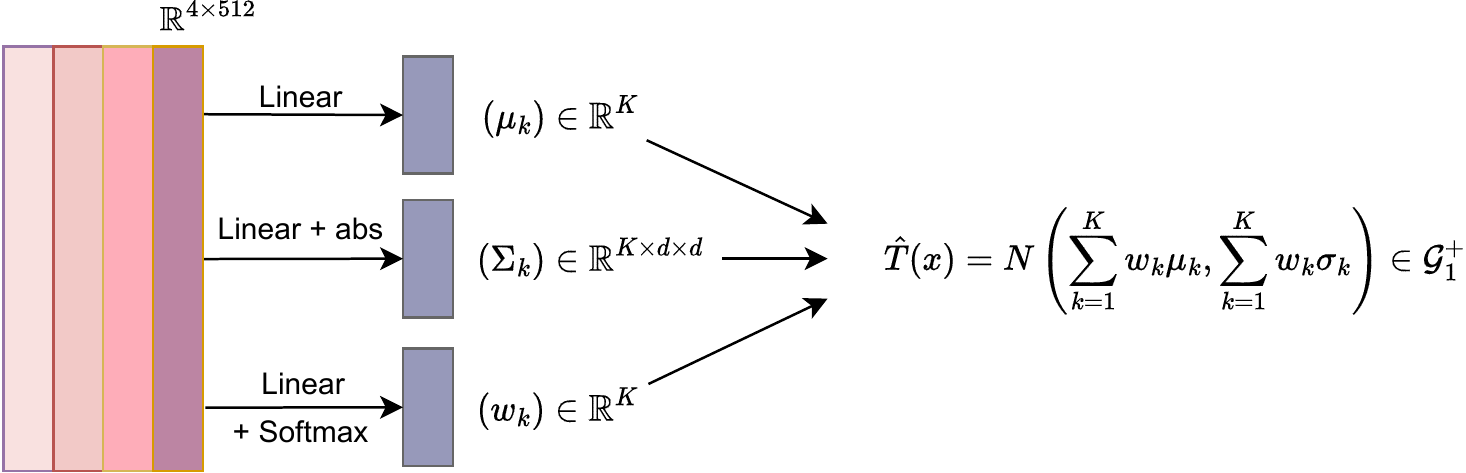}
        	\caption{}  \label{fig:transformer_Gplus}
    	\end{subfigure}
    	\begin{subfigure}[t]{0.35\textwidth}
    		\centering
        	\includegraphics[width=.95\textwidth]{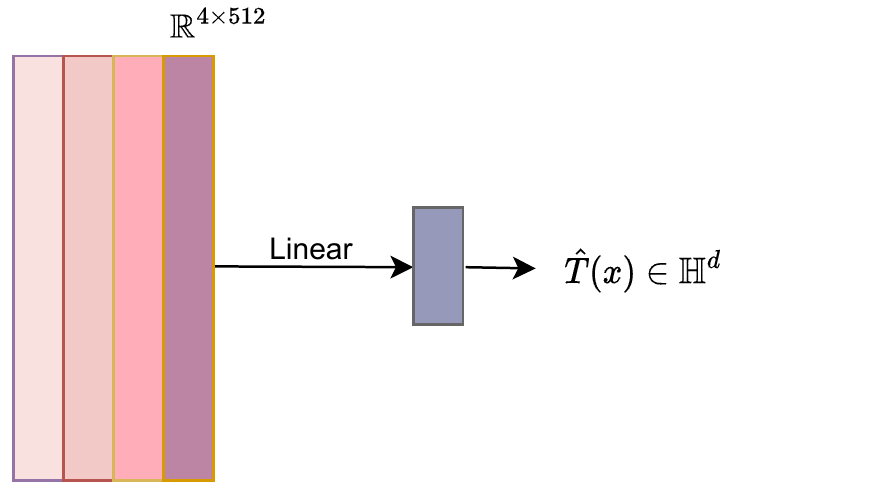}
        	\caption{}  \label{fig:transformer_Hyperbolic}
    	\end{subfigure}
    	\caption{Two different networks implementing feature (embedding) maps. The first block (violet dashed line) is common to all networks; the output blocks change to accommodate different target spaces.
    	a) Probabilistic transformer of \eqref{eq__def_transformer}.
    	b) $\mathcal{G}_1^+$-transformer.
    	c) $\mathbb{H}^d$-transformer.}
       \end{figure*}

    \subsubsection*{Embedding into the $d$-Dimensional Hyperbolic Space}
    Even though the hyperbolic plane is theoretically sufficient to embed trees with small distortion \cite{sarkar2011low}, experimental evidence indicates that neural networks representations into higher dimensional hyperbolic space perform better \cite{digiovanni2022heterogeneous}.  We therefore also compare our representations with representations in $d$-dimensional hyperbolic space represented by the Poincar\'e upper-half space model, whose underlying set it
    \begin{equation*}
        \mathbb{H}^d = \{ x\in \R^{d+1} \text{ such that } x_{d+1}>0  \},
    \end{equation*}
    and whose geodesic distance is given by
    \begin{equation*}
        d_{\mathbb{H}^d}(x,y) = \mathrm{arccosh} \left( 1+ \frac{\|x-y\|_2^2}{2x_n y_n}  \right), \forall x,y \in\mathbb{H}^d.
    \end{equation*}

    We adapt the output of the probabilistic transformer to return vectors in $\mathbb{H}^d$, by swapping its last layer for a trainable linear readout map (see Figure \ref{fig:transformer_Hyperbolic}) and modifying the embedding objective to minimize the distortion with respect to the hyperbolic distance. 
    The $d$-dimensional hyperbolic model is then trained using the following loss function:
    \begin{equation}
    	\label{eq_objective_hyperbolic_dimd}
    	L_{\mathbb{H}^d}(\theta) = \sum_{\mathclap{x,y\in\mathcal{X}_{\mathrm{train}}}} \Big( d_{\mathbb{H}^d}( T_\theta(x), T_\theta(y))^2 - d_\mathcal{X}(x,y)^2  \Big)^2.
    \end{equation}  
    
    \subsubsection*{Experimental Results}
    
    We consider a regular binary tree $\mathcal{X}=(V, E)$ (Figure \ref{fig:binaryTree})
    of depth six with a total of $\vert V \vert = 127$ vertices. 
    The distance $d_\mathcal{X}$ is the shortest path distance.
    We partition the vertices $V$ into training and testing sets, $V_{\text{train}} \cup V_{\text{test}} = V$, with $\vert V_{\text{train}} \vert = 111$ and $\vert V_{\text{test}} \vert = 16$.
    The test vertices (colored white in Figure \ref{fig:binaryTree}) are used to evaluate the quality of out-of-sample representations (that is to say, the generalization) computed by the different representation maps. We designate $L=20$ training vertices as landmarks (colored black in Figure \ref{fig:binaryTree}). The two probabilistic transformers take as input the distances between a vertex and the $L$ landmarks. 
    The neural networks contains more than $7,000$ parameters with only the last layer being adapted to produce an output in the desired space. We use $K=5$ mixture components and $d=15$ for the dimension of the hyperbolic space to ensure a fair comparison with the probabilistic transformer's effective dimension.
    
    \begin{table}[!htbp]%
    \centering
    \begin{tabular}{@{}llccc@{}}
    \toprule
    & & $\boldsymbol{\mathcal{GM}_2(\mathbb{R})}$ & $\boldsymbol{\mathbb{H}^2}$  & $\boldsymbol{\mathbb{H}^{15}}$ \\
    \midrule
    \multirow{2}{*}{\textbf{Training points}} & Mean error & 0.03 & 0.15 & 0.04\\
    & Max. error  & 0.25  & 0.44 & 0.26 \\[2mm]
    \multirow{2}{*}{\textbf{Testing points}} & Mean error & 0.02 & 0.19 & 0.05\\
    & Max. error  & 0.25  & 0.44 & 0.26\\
    \bottomrule
    \end{tabular}
        \caption{Average and maximum relative embedding errors for binary tree representations.}
        \label{tab:XP_tree}
    \end{table}

    Table \ref{tab:XP_tree} shows that our neural representations in the space of Gaussian mixtures perform better than $\mathbb{H}^2$ even when embedding simple trees, and that they perform on par or slightly better than $\mathbb{H}^{15}$. The difference in performance between $\mathbb{H}^2$ and $\mathbb{H}^{15}$ epitomizes the difficulty in verifying theoretical facts with computer experiments: even though theoretically both spaces can achieve low distortion, it is easier to approximate a good feature map for $\mathbb{H}^{15}$. We report the average relative absolute difference between the pairwise distances in the binary tree as compared to embeddings in the respective representation spaces.
    The training set evaluates the embedding quality of the given set of landmarks.  The role of the test set is to check whether the learned representation still provides a good embedding using the set of landmarks in the binary graph\footnote{Our theorems give ``memorization'' guarantees, or, phrased differently, approximation guarantees for embeddings of finite metric spaces. In computer experiments we thus deliberately explore the generalization aspect. The related theoretical questions remain open.}.
    
    We visually assess the quality of the computed representations in Figure \ref{fig:XP_binary_tree}: Figure \ref{fig:distr_MG} displays the probability density functions of the $\mathcal{GM}(\mathbb{R})$ representations and Figure \ref{fig:fr_hyperbolic} displays the information-theoretic hyperbolic representations using Gaussian measures. In both cases the networks learn representations that change continuously with the shortest path distance in $\mathcal{X}$.

    \begin{figure}[ht!]%
    	\centering
    	\begin{subfigure}[t]{1\textwidth}
    		\centering
        	\includegraphics[origin=c,width=.5\linewidth]{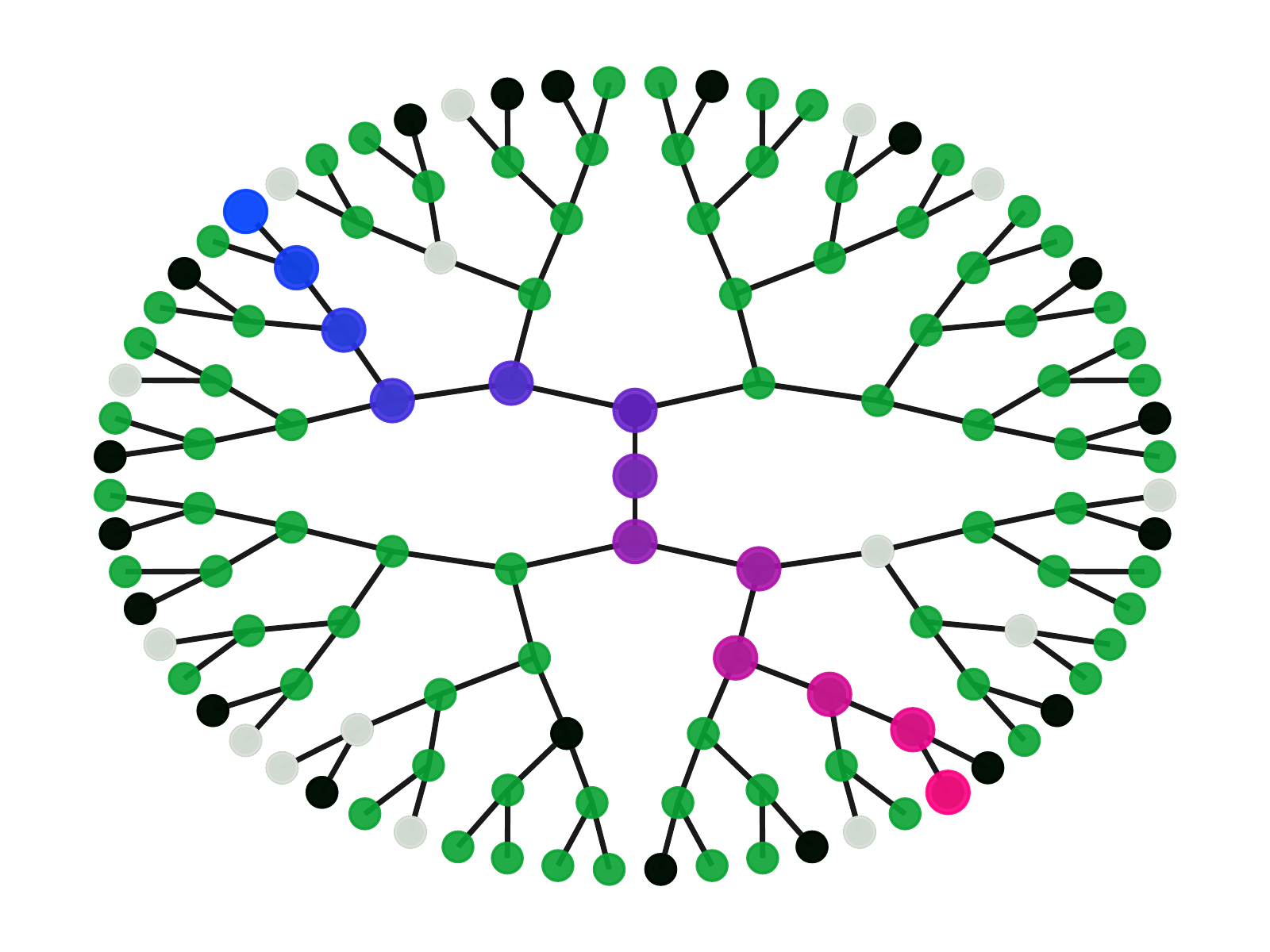}
        	\caption{\label{fig:binaryTree}}  
    	\end{subfigure}
    	\centering
    	\begin{subfigure}[t]{0.32\textwidth}
    		\centering
            \includegraphics[width=\linewidth]{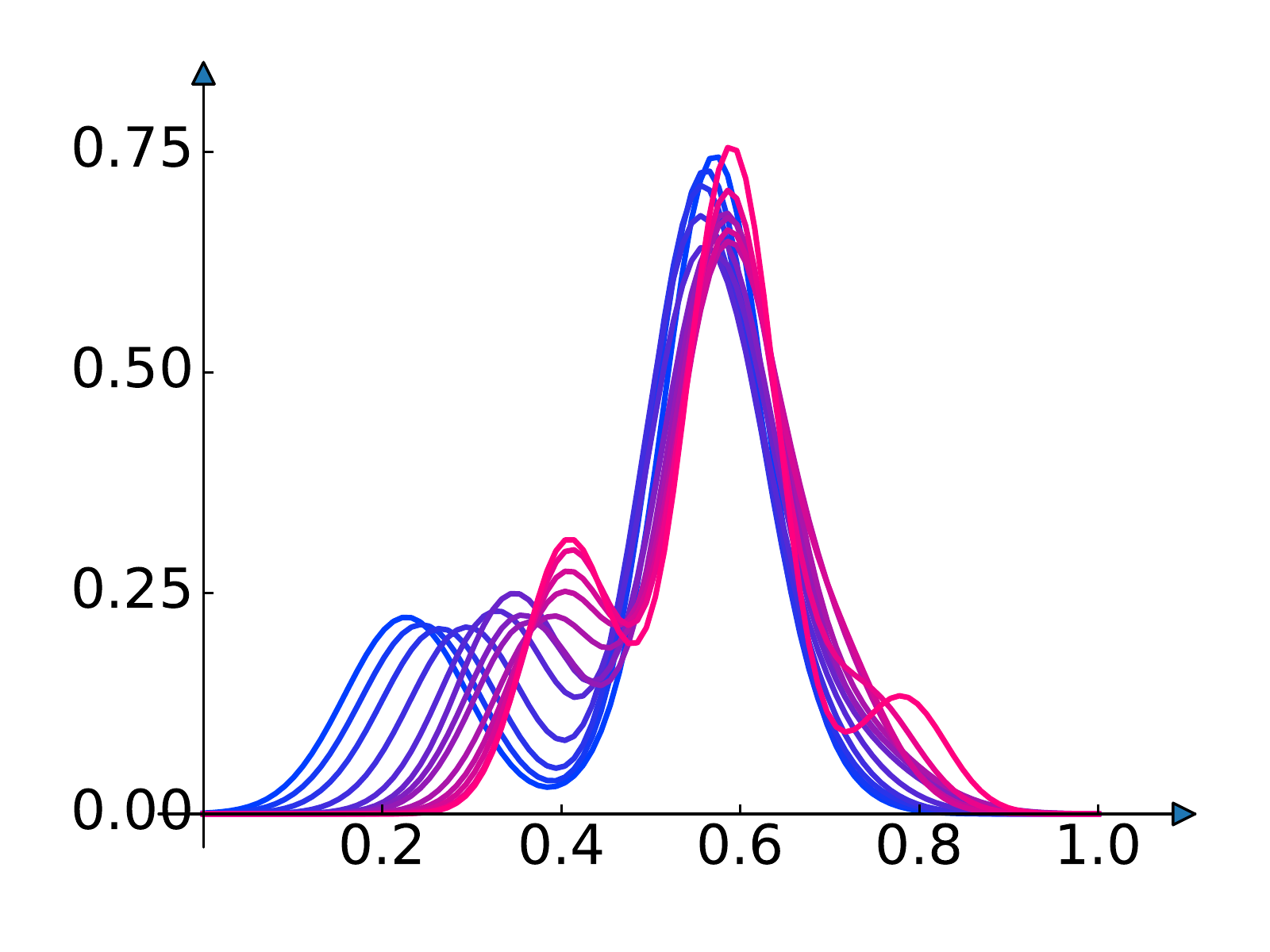}
    		\caption{\label{fig:distr_MG}} 
    	\end{subfigure}
    	\begin{subfigure}[t]{0.32\textwidth}
    		\centering
            \includegraphics[width=\linewidth]{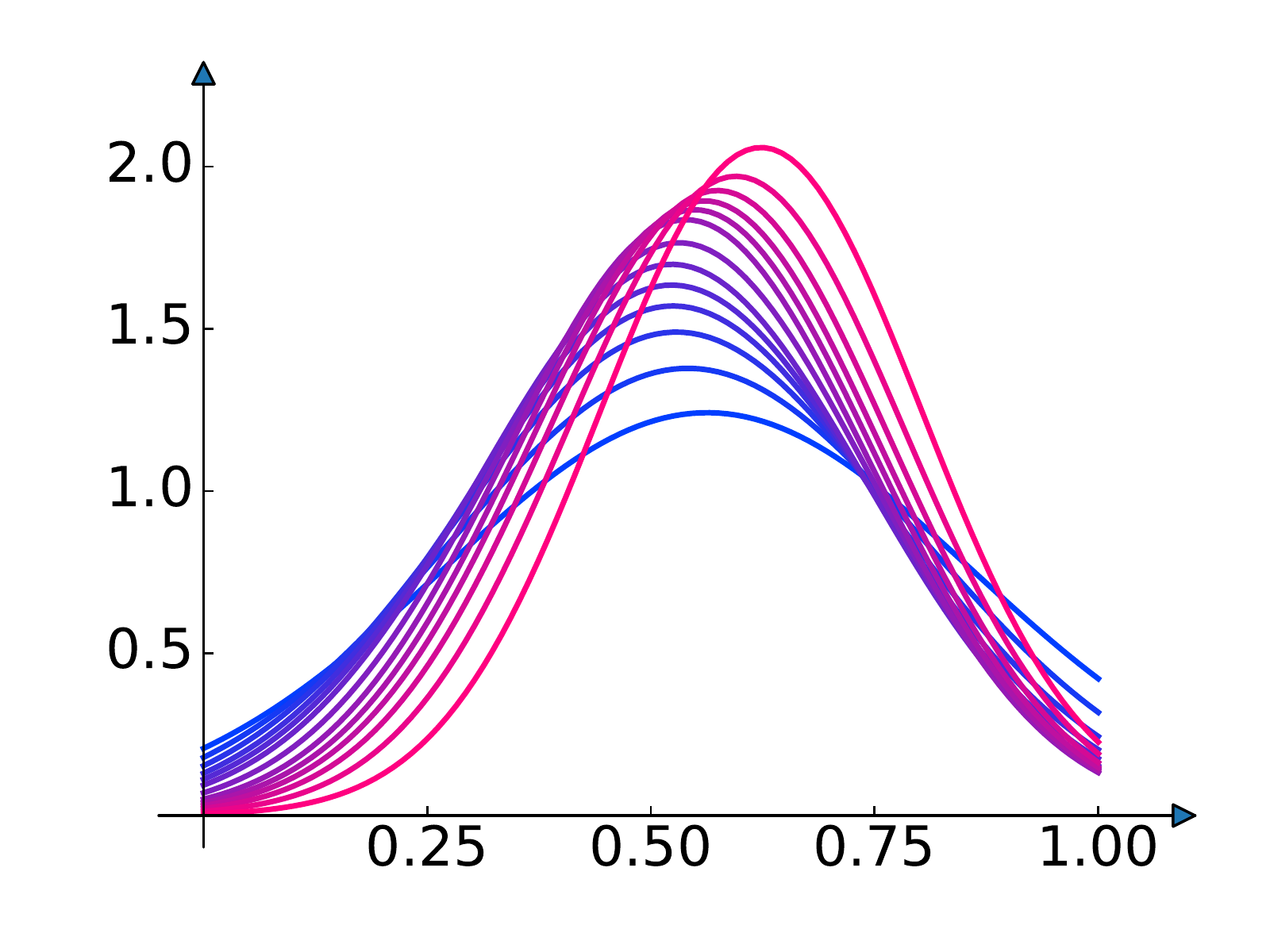}
    		\caption{\label{fig:fr_hyperbolic}}
    	\end{subfigure}
    	\caption{Two embeddings of colored vertices in the binary tree $\mathcal{X}$ shown in (a); white vertices are not seen during training; landmarks are colored black. We apply the transformation $\sigma \mapsto \log_{10}(\sigma+2.1)$ to aid visualization. (b) Mixture Gaussian Wasserstein embeddings represented by their density. (c) hyperbolic (univariate Gaussian) embeddings represented by their density. \label{fig:XP_binary_tree}}
    \end{figure}
   
   Figure~\ref{fig:XP_binary_tree} visually examines the stability of the embeddings produced by the PT against those learned by its hyperbolic counterpart.  Since both output univariate probability distributions, a visual comparison can easily be drawn.  When comparing Figure \ref{fig:fr_hyperbolic} to Figure \ref{fig:distr_MG} we see that the embedding of the path in Figure \ref{fig:XP_binary_tree} is only a bit more contorted than the embedding produced by the PT's hyperbolic counterpart.  This can explain why both embeddings generalize similarly well, but that the gap in embedding quality is explained by the richness of the geometry of $(\mathcal{GM}_2(\rr),\mw_2)$ as compared to the rigidity of the hyperbolic plane's geometry.

    \subsection{PTs Leverage Dimension Efficiently: Manifold Embeddings}
    \label{s_efficiency_EuclideanDim_EffectiveDimension}
    
    We now compare the capacity of the Euclidean space with $(\mathcal{GM}_2(\rr),\mw_2)$ to represent smooth geometries using few effective dimensions. 
    We represent datasets from $N$-dimensional spheres equipped with the geodesic distance.  We remark that, as shown in \cite{robinson2006sphere}, there exist no bi-Lipschitz embeddings of the sphere into any Euclidean space with arbitrarily low distortion.%
    \footnote{We note that this is different from positive results on Riemannian embeddings of the sphere in Euclidean spaces as exemplified by Nash's theorems \cite{Nash_1954_c1_Embeddings}. These ``embeddings'' view the sphere as a submanifold of the Euclidean space but make no guarantees that the embedding preserves the geodesic distance.}%
    ~In this case our findings corroborate known facts and also show that the PT leverages the rich geometry of $(\mathcal{GM}_2(\rr),\mw_2)$ to produce better  low-dimensional embeddings standard feedforward neural networks taking values in high-dimensional Euclidean spaces.

    \subsubsection*{Embedding into Euclidean Space}

    Despite negative theoretical results, Euclidean representations are attractive in practice thanks to the wealth of available machine learning and numerical tools. It is thus desirable to compare them with our proposed representations in the space of measures. To this end we modify the output of our probabilistic transformer to output vectors in $\rr^d$, by swapping its last layer for a trainable linear readout map (see Figure~\ref{fig:transformer_R}) and modifying the embedding objective to minimize the distortion with respect to the Euclidean distance $\|\cdot - \cdot \|_{\ell^2_d}$.  
    Naively, one may expect a well-performing Euclidean representation for sufficiently high embedding dimension $d$. However, as shown below, even for very large $d$ there is a hard limit to representation quality, reflecting the introductory theoretical notes.
    
    We let $\mathcal{S}^N\eqdef\{z\in \rr^{N+1}:\,\|z\|\leq 1\}$ denote the $N$-dimensional sphere and equip it with the geodesic distance  
    $
        d_{\mathcal{S}^N}(x,\tilde{x}) = \cos^{-1}\left( x^{\top}\tilde{x}\right)
    $, 
    $x,\tilde{x}\in \mathcal{S}^N$ \cite{MR3852654,fletcher2003statistics}.
    We randomly sample data points $\{x_i\}_{i=1}^n$ from the uniform probability measure on $\mathcal{S}^N$.
    \begin{figure*}[b!]
    		\centering
        	\includegraphics[width=0.3\linewidth]{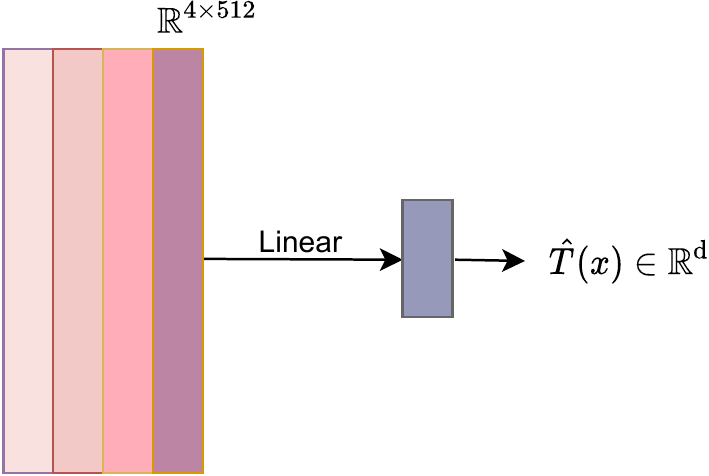}
    	\caption{Replacing the PT's final layer by a linear readout layer yields our feedforward network benchmark.
    	}
    	\label{fig:transformer_R}
   \end{figure*}
    The resulting Euclidean model is a feedforward network with a suitable Euclidean loss function,
	\begin{equation}
    	\label{eq_objective_Euclidean}
    	\mathcal{L}_{\mathrm{Euclid}}(\theta) = \sum_{\mathclap{x,y\in\mathcal{X}_{\mathrm{train}}}} \left( \| T_\theta(x) - T_\theta(y) \|_{\ell^2_d}^2 - d_\mathcal{X}(x,y)^{2}  \right)^2.
    \end{equation}

    \subsubsection{2-dimensional sphere \texorpdfstring{$\mathcal{S}^2$}{S$^2$}}
    \label{s_Numerics__ss_Manifold___sss_2Sphere}

	We begin by visualizing the $\mathcal{GM}_2(\rr)$ embeddings of points on the 2-sphere to develop intuition. A good representation should achieve comparably low distortion across the entire manifold and vary continuously with respect to the geodesic distance on $\mathcal{S}^2$. The learned feature map should generalize well from a set of landmarks to arbitrary points on the sphere.

    We train a PT for $160$ iterations with the Adam optimizer \cite{DBLP:journals/corr/KingmaB14}. In each iteration, we use random batch of 32 points among the 10,000 fixed training points chosen from a uniform measure on the sphere. The $13$ landmarks are not updated during training. The network contains about 200,000 parameters and returns a Gaussian mixture with $K=5$ components.
    
	The results are shown in Figure~\ref{fig:sphere_3d}. For each point $x_k$, colors in Figures~\ref{fig:sphere_3d_MEAN} and~\ref{fig:sphere_3d_MAX} represent the quantities
    \[
    \begin{aligned}
        &\frac{1}{K}\sum_{j=1}^K \Big \vert  d_\M(x_{k},x_{j})^2 - \mw_2^2( \hat T_\theta(x_{k}), \hat T_\theta(x_{j})) \Big \vert  
        &\text{and\quad} 
        & \sup_{1\leq i,j \leq K} 
        \Big \vert 
            d_\M(x_{i},x_{j})^2 - \mw_2^2( \hat T_\theta(x_{i}), \hat T_\theta(x_{j})) 
        \Big \vert .    
    \end{aligned}
    \]
    We can see from the figure that the distances in the representation space faithfully represent those on the sphere. This is further detailed in Figure~\ref{fig:sphere_3d_WorstPoints} where color of a point encodes the embedding error between the point and the red cross. Notice that even if the maximum error between points may seem large (about 0.35), it is reached only for a very small number of point pairs; most pairs have a relative error bellow $10^{-2}$.

    \begin{figure}[!htb]
    	\begin{subfigure}[t]{0.29\textwidth}
    		\centering
    		\includegraphics[width=\linewidth]{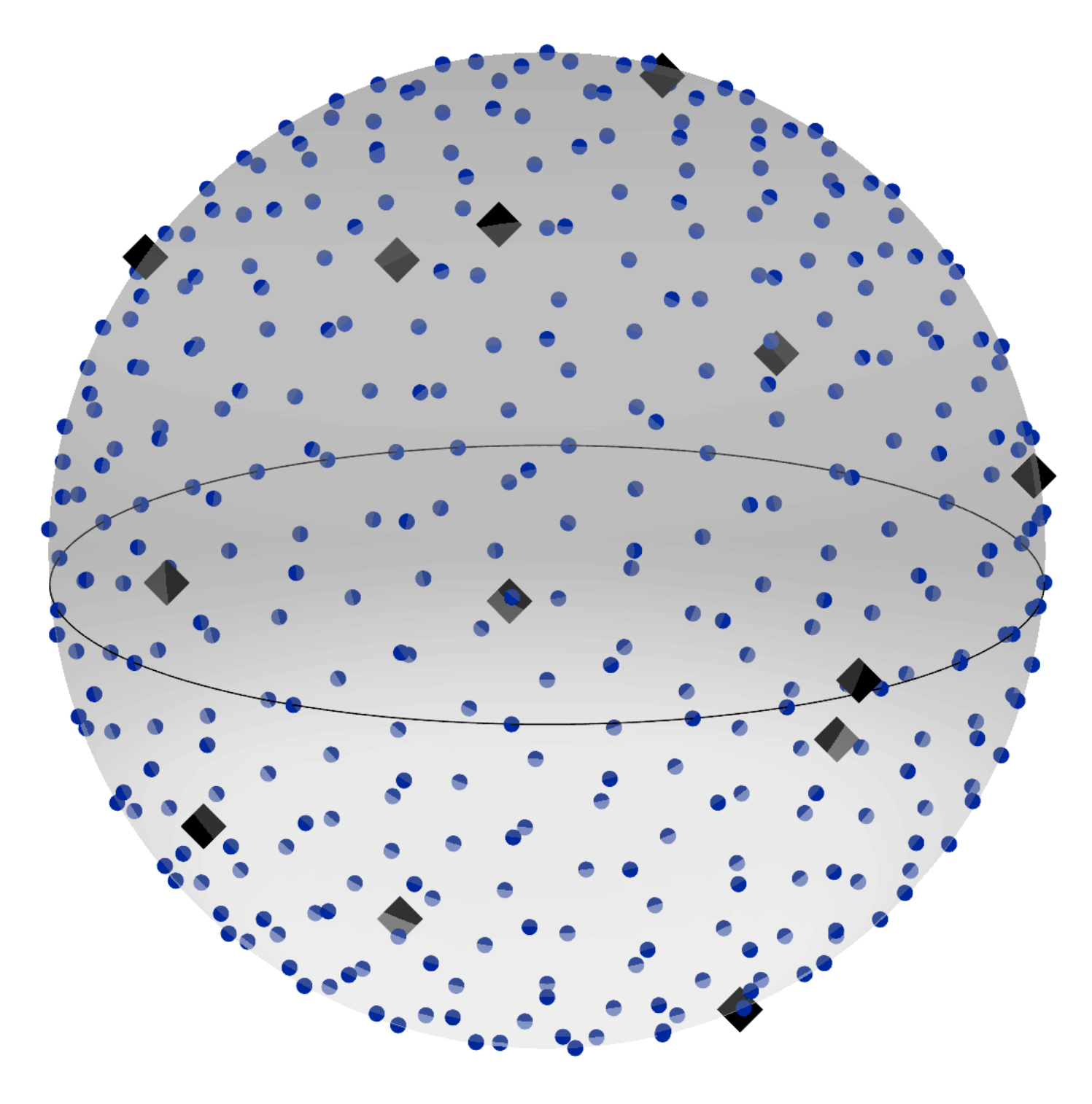}
    		\caption{}
    		\label{fig:sphere_3d_MEAN}
    	\end{subfigure}
    	\begin{subfigure}[t]{0.29\textwidth}
    		\centering
    		\includegraphics[width=\linewidth]{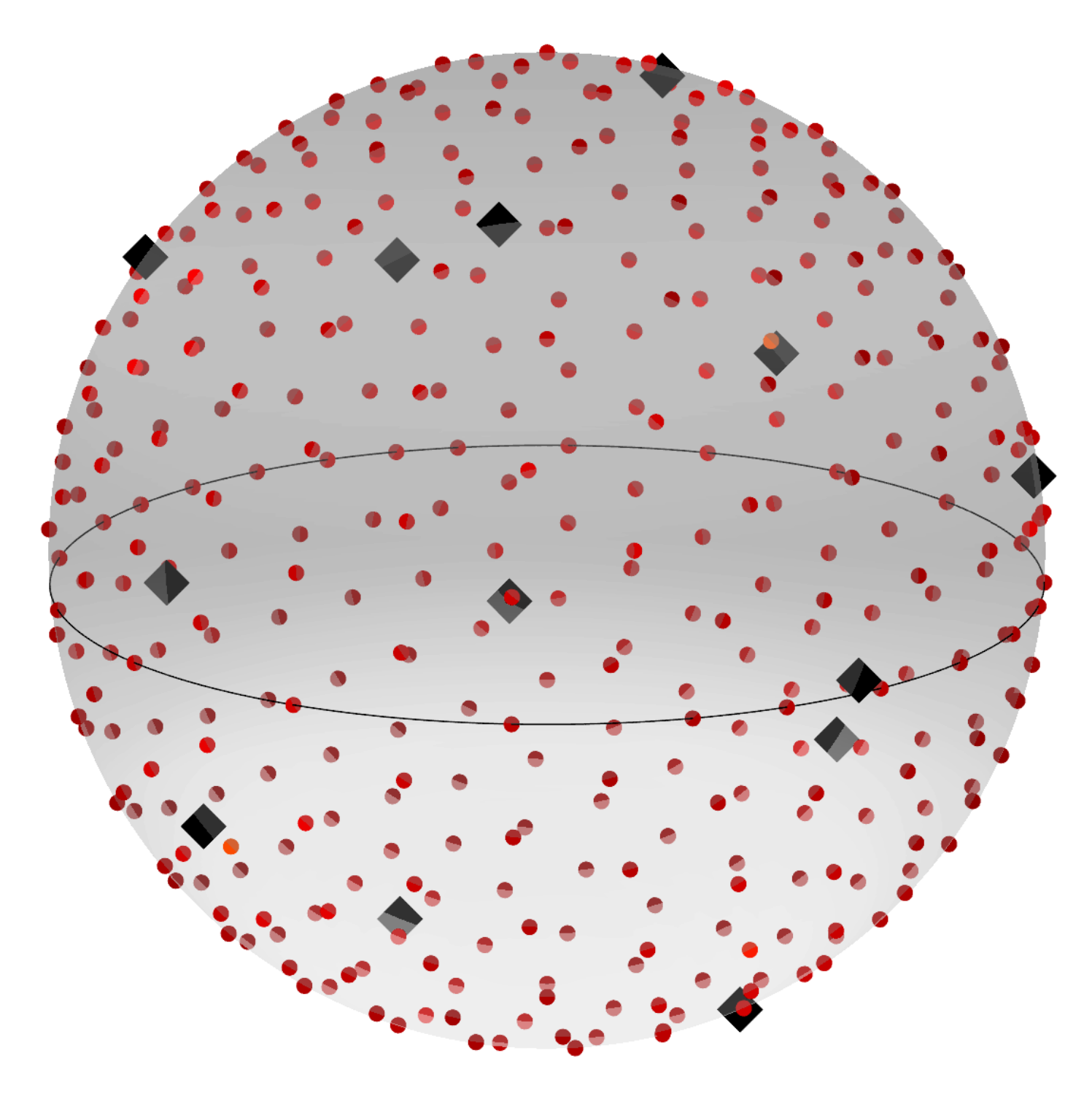}
    		\caption{}
    		\label{fig:sphere_3d_MAX}
    	\end{subfigure}
    	\hspace{0.1cm}
    	\centering
    	\begin{subfigure}[t]{0.39\textwidth}
    		\centering
    		\includegraphics[width=\linewidth]{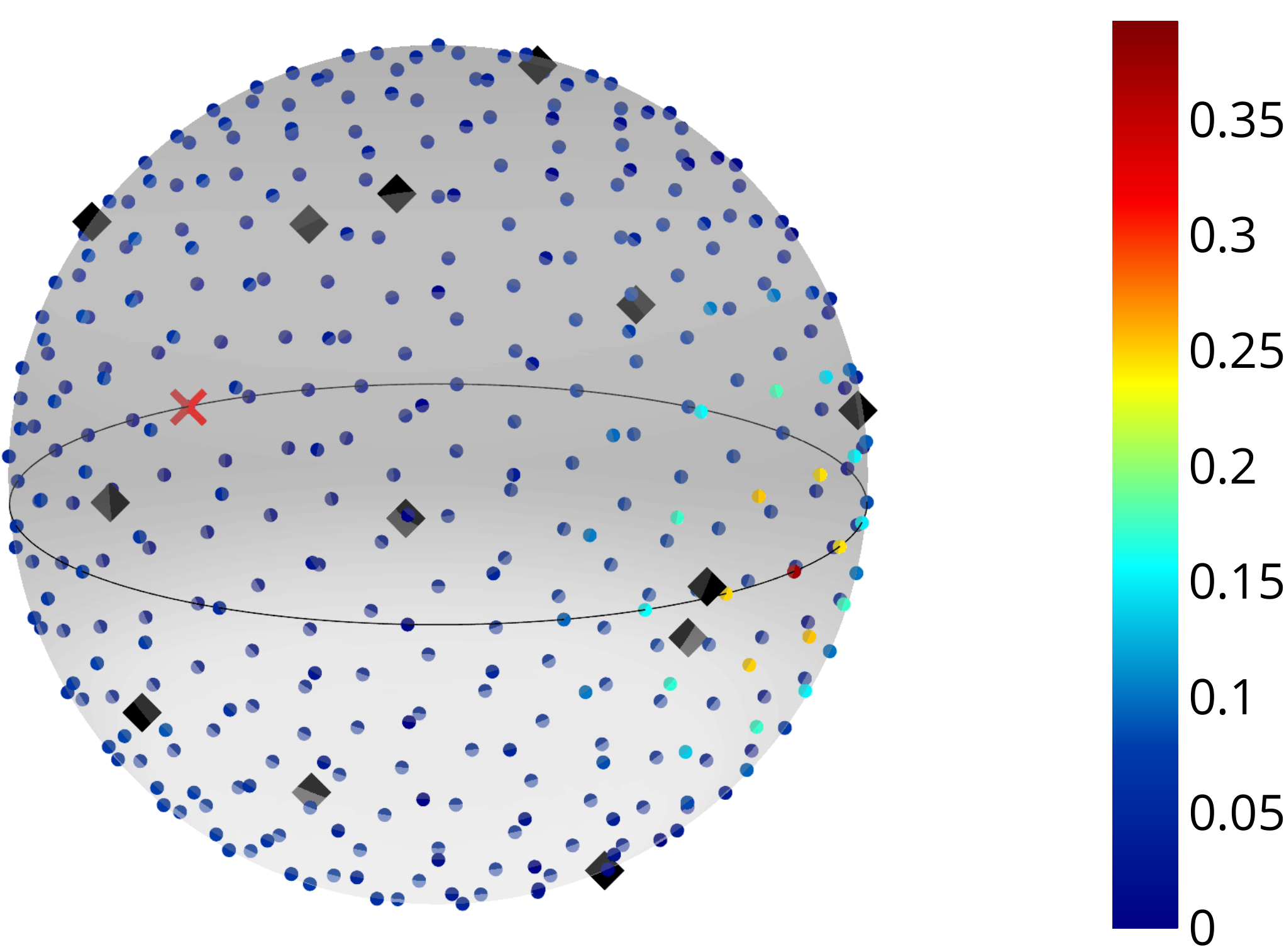}
    		\caption{\hspace{1.7cm}\textcolor{white}{.}}
    		\label{fig:sphere_3d_WorstPoints}
    	\end{subfigure}
    	\caption{Embedding quality of 2-dimensional sphere $\mathcal{S}^2$ in $\mathcal{GM}_2(\rr)$.
    	a) Average error of a point and all the other elements.
    	b) Maximum error of a point and all the other elements.
    	c) Error between the point marked by a red cross and all other points.
    	\label{fig:sphere_3d}} 
    \end{figure}
    
    We then assess the continuity of the embedding in Figure~\ref{fig:sphere_3d2}. We plot the probability density function of a sequence of points spiralling up from the bottom of the sphere to the top, which we colour-code as progressing from red to blue. As expected, the Gaussian mixtures representing points on the sphere progressively change. We also see that only a few significant mixture components are required to perform the embedding.  To aid visualization, we normalize the support of the densities and transform the standard deviation of the Gaussian components using the mapping $\sigma \mapsto \log_{10}(\sigma+2.1)$.
    
    \def\size{5.5cm}
    \begin{figure}[!htb]%
    		\centering
    		\includegraphics[height=\size]{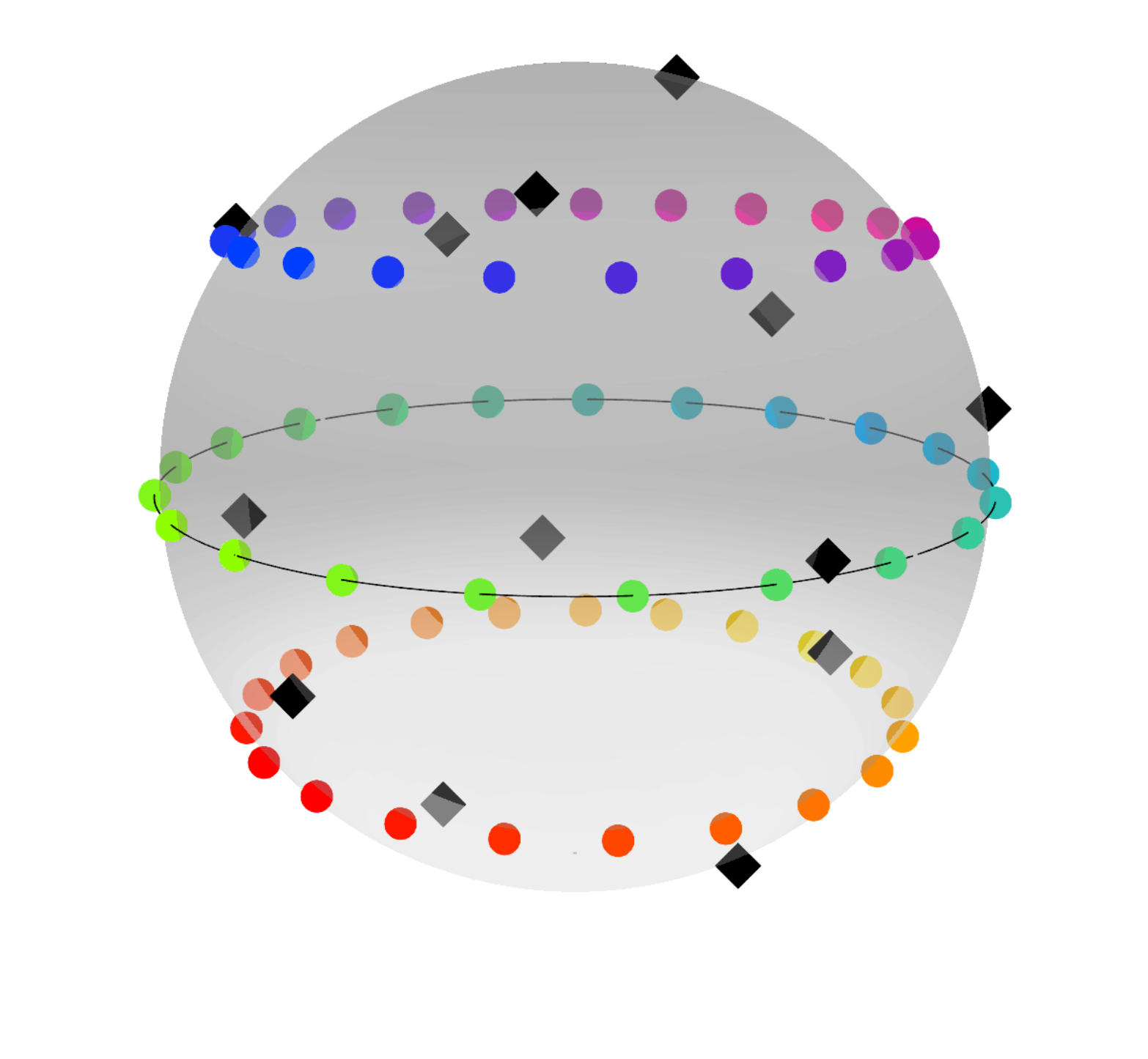}
    		\includegraphics[height=\size]{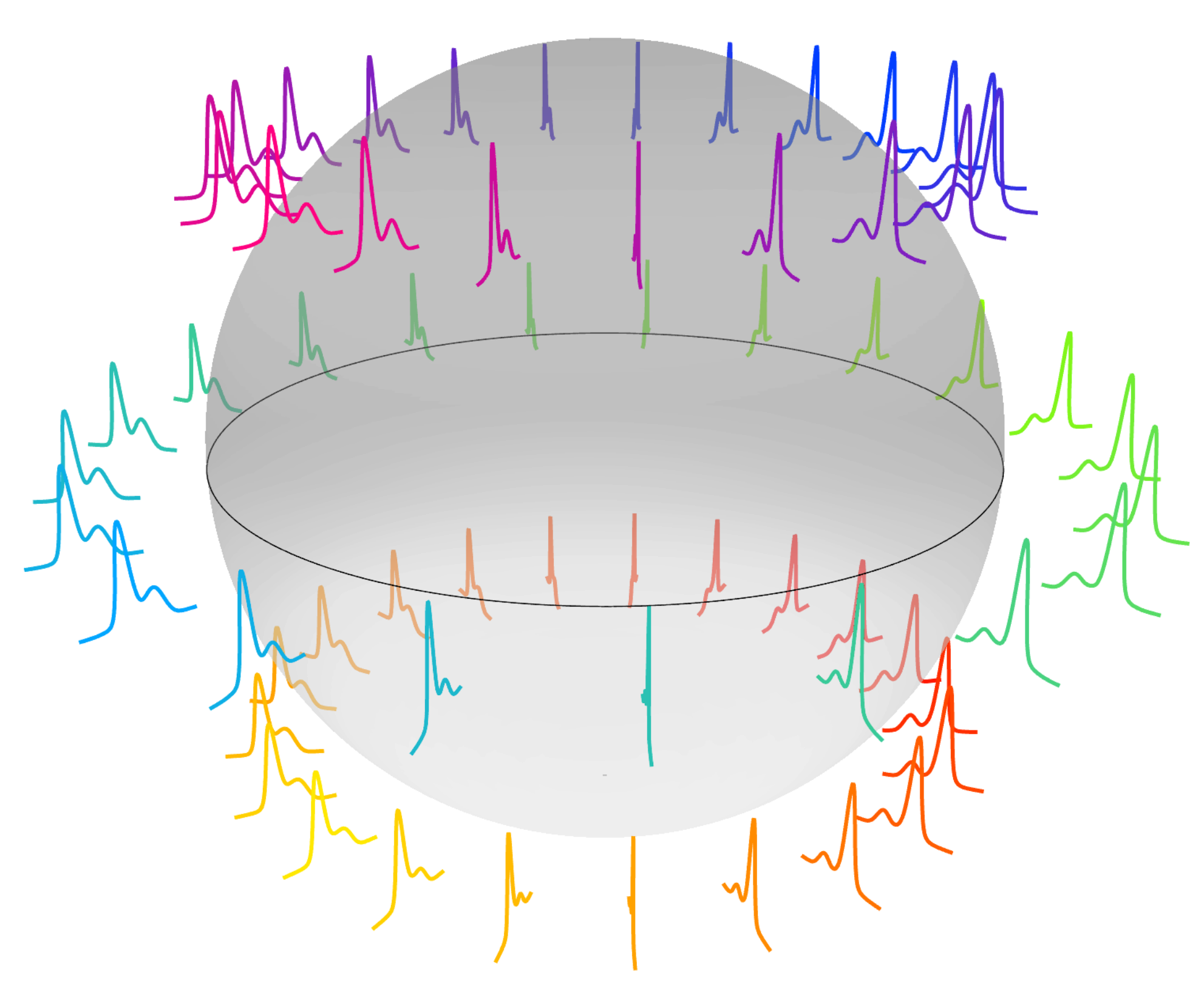}
    	\caption{Probability density functions of the embeddings of points on $\mathcal{S}^2$ into $\mathcal{GM}_2(\rr)$.} \label{fig:sphere_3d2}
    \end{figure}
    
    \subsubsection{Embedding the $N$-Dimensional Sphere}
    The motivation to use Gaussian mixtures instead of Euclidean embeddings is that they easily capture complex geometries; achieving comparable distortion would either require a very high-dimensional Euclidean embedding or is simply impossible due to curvature. To illustrate this, we set up an experiment on $N$-spheres where we vary $N$. We fix the number of mixture components at the output of the PT to $K = 5$. The Euclidean network outputs vectors in $\mathbb{R}^{15}$, thus having the same number of degrees of freedom.  We also compare these two networks with embedding in the hyperbolic space $\mathbb{H}^{15}$. All networks contain about 200,000 parameters. We set the number of landmarks to $L = N + 10$ (note that we need at least $N$ landmarks to uniquely localize a point on $\mathcal{S}^N$). We distribute the landmarks quasi-uniformly on the $N$-sphere adapting a procedure from \cite{debarnot2022deep}.
    
    \begin{figure}[H]
        \centering
        \includegraphics[width=0.4\textwidth]{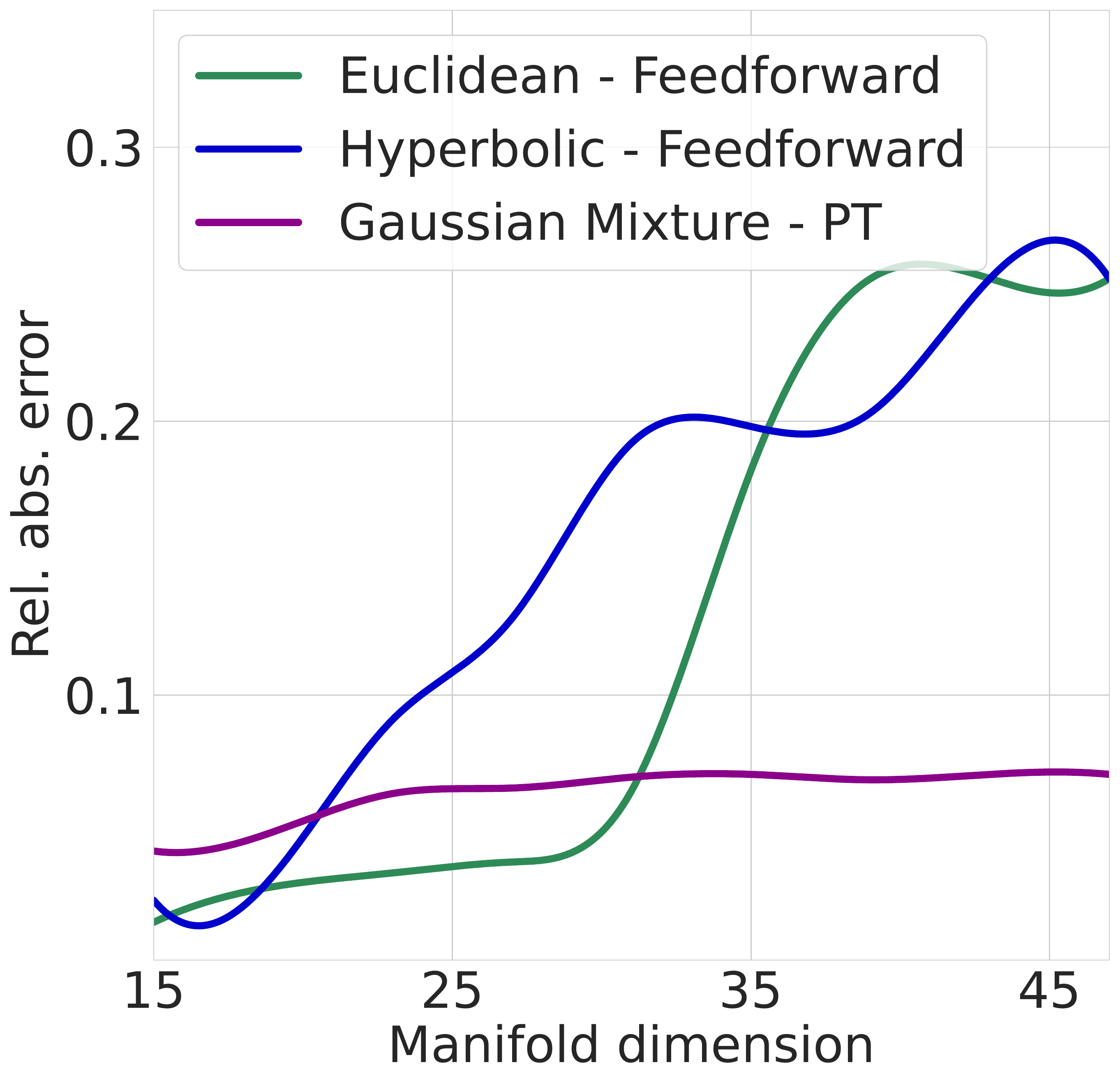}
        \caption{Impact of dimension.\label{fig:sphere_n}}
        \vspace{-0cm}
    \end{figure}

    Figure~\ref{fig:sphere_n} shows the results. As the manifold dimension $N$ increases, the quality of the Euclidean  and Hyperbolic embeddings implemented by the deep feedforward networks rapidly deteriorate.  In contrast, the distortion of the embedding implemented by the probabilistic transformer into the space $(\mathcal{GM}_2(\rr),\mw_2)$ of \cite{WassersteinGaussianMixtures} remains stable.

	\section{Outline of the Proofs of Main Results}
	\label{s_Main_Proof_Outline}
	This section overviews the main steps toward proving Theorem~\ref{MAINTHEOREM_DETERMINISTIC}.  Similar techniques are used in deriving the probabilistic Theorem~\ref{MAINTHEOREM_PROBABILISTIC}.  The latter proof is more technical as it relies on the non-linear Dvoretzky theorem \cite{NaorTao_2012_Dvoretzky_ISJM} which is a metric version of the classical result of Dvoretzky \cite{Dvoretzky_OG_1960}; we defer it to Appendix~\ref{a_Proofs}.   
	
	Our proofs use a novel two-step approach. In \textit{step one}, we use a key lemma (Lemma~\ref{lemma_embedding_neuralification} below) to show that certain probability-measure-valued functions can be implemented by metric transformers with few parameters. In \textit{step two}, given any of the three aforementioned ``geometric priors'', we demonstrate the existence of ``optimal'' embeddings of the required form.  Each of the proofs then concludes by applying the key lemma to show that the embedding can be \textit{exactly} implemented (memorized) by the probabilistic transformer.  The proof is completed by remarking that the PT we have built is defined on the entire metric space $\xxx$.  
	
    We note that the $d$-dimensional Gaussian measure with mean $\mu\in \mathbb{R}^d$ and $d\times d$ covariance matrix $\Sigma$ is be denoted by $N_d(\mu,\Sigma)$.  
	\begin{lemma}[Key Lemma: Probabilistic Transformers Memorize Empirical Markov Kernels]
	\label{lemma_embedding_neuralification}
		Let $d,K$ be a positive integers, let $1\leq p<\infty$, $\#\xxx_n=n>2$, and let $\sigma\eqdef \operatorname{ReLU}$.
		Let $(\xxx_n,d_n)$ be an $n$-point metric space and for $k=1,\dots,K$, let $\mu_k:\xxx_n\rightarrow \rr^d$.  
		Define the map $\varphi:\xxx_n \rightarrow \mathcal{GM}_2(\rr^d)$ by
		\[
		\varphi(x)
		\eqdef 
		\frac1{K}
		\sum_{k=1}^K
		N_d(\mu_k(x),0)
		.
		\]
		Then, there exists a probabilistic transformer with unnormalized graph attention $T:\xxx_n\rightarrow \mathcal{GM}_2(\rr^d)$, which exactly implements $\varphi$; i.e.:
		\[
		\varphi(x) = T(x)
		\qquad
		\mbox{ for every $x \in \xxx_n$}
		\]
		such that
		\begin{align}
		\mathrm{width}(T) &= 
		\max\{
                n
            , 
                K
            , 
                d^2
            ,
                n(n-1) + \max\{d,12\}
        \},
			\\
		\mathrm{depth}(T) &= 
    	\mathcal{O}\left(
	n\left\{
		1+
		    \sqrt{n\log(n)}
		        \,
		     \left[
		        1
		            +
		       \frac{\log(2)}{\log(n)}\,
		       \left(
		            C_n
		                +
		         \frac{
			         \log\big(
			            n^{5/2}\, \operatorname{aspect}(\xxx_n,d_n)
			          \big)
		         }{
		            \log(2)
		         }
		       \right)_+
		     \right]
			\right\}
		\right) \\
		\mathrm{effdim}(T) &\leq K\, (d+d^2) \\
		\mathrm{par}(T) 
		&= \mathcal{O}\Bigg(
                K n \Big( \tfrac{11}{4} \max\{n, d\} \, n^2 -  1 \Big) 
		        \Bigg\{d + \max\{d, 12\}^2 \sqrt{n\log(n)}    \\
    				&\hspace{10mm}\times
    			     \left[
    			        1
    			            +
    			       \frac{\log(2)}{\log(n)}\,
    			       \left(
    			            C_n
    			                +
    			         \frac{
        			         \log\big(
        			            n^{5/2}\,
        			            \operatorname{aspect}(\xxx_n,d_n)
        			          \big)
    			         }{
    			            \log(2)
    			         }
    			       \right)_+
    			 \right]
    		\Bigg\}
		\Bigg).\nonumber
		\end{align}
    Moreover, the ``dimensional constant'' $C_n>0$ is defined by $
    	C_n
    \eqdef
    	\frac{
                2\log(5 \sqrt{2\pi})
            + 
                \frac{3}{2}
                \log(n)
            -
                \frac1{2}\log(n+1)
        }{
            2\log(2)
        }
    $.
	\end{lemma}
	Before stating the next lemma, let us recall the notion of a \textit{doubling metric space}.  Briefly, these are metric spaces for which a maximal ratio determines the number of open balls of finite radius required to cover an open ball of twice their radius.  
	More precisely, a metric space $(\xxx,d)$ is said to be doubling if and only if, there is a constant $C_d>0$ for every point $x\in \xxx$ and every radius $r>0$ the open metric ball about $x$ of radius $r$, i.e. the set $B(x,r)\eqdef \{u\in \xxx:\, d(u,x)<r\}$, there are at most $C_d$ points $x_1,\dots,x_n\in \xxx$ (i.e. $n\leq C_d$) such that the collection of open metric balls $\{B(x_i,r/2)\}_{i=1}^n$ cover $B(x,r)$.  
	The smallest constant $C_d>0$ for which the above this relation holds is called the \textit{doubling constant} of $(\xxx,d)$.  
	Examples of doubling metric spaces are any finite metric space, Carnot Groups such as the Heisenberg group \cite{PansuHeisenbergGroup_1989}, and more generally, any metric space that can be bi-H\"{o}lder embedded into some Euclidean space (see \citep[Theorem 12.2]{heinonen2001lectures} and \cite{naor2012assouad}).  
	
	The doubling constant $C_d$ of a doubling metric space $(\xxx,d)$ plays the role of dimension; for example, for a compact $n$-dimensional Riemannian manifold $\log_2(C_d)$ is of the order $\mathcal{O}(n)$ and the same is true of the $n$-dimensional Euclidean space.  
	However, in general it can be difficult to compute the doubling constant of $(\xxx,d)$.  Thus, we turn to the notion of \textit{metric capacity} as defined in \cite{Brue2021Extension}.  
    The reason is that, as shown in \citep[Proposition 1.7]{Brue2021Extension} the doubling constant and the metric capacity are proportional (we refer the reader to that paper or the proof of Lemma~\ref{lem_doubling_to_capacity_control_lemma} for details).  
	
	Our second main embedding lemma concerns the case where $(\xxx_n,d_n)$ has its points drawn from a doubling metric space.  Since the logarithm of the doubling constant of such spaces plays the role of topological dimension for manifolds (see \citep[Chapter 12]{heinonen2001lectures}), it is intuitive that the number of (parameterized) point masses required to embed $(\xxx_n,d_n)$ into the Wasserstein space depends on the logarithm of the doubling constant.  
	\begin{lemma}[{Datasets in Doubling Metric Spaces bi-H\"{o}lder Embed into $(\mathcal{MG}_2(\rr),\mw_2)$}]
		\label{lem_embedding_doubling_metric_spaces}
		Let $(\xxx,d)$ be a doubling metric space with doubling constant $C_d>0$, $\xxx_n\subseteq \xxx$, and $d_n\eqdef d_{\xxx_n}$.  There are constants $c,C>0$, such that for any $\frac{1}{2}<\alpha<1$, there exists an injective map $\varphi:(\xxx_n,d_n)\hookrightarrow (\mathcal{GM}_2(\rr),\mw_2)$ of the form
		\[
		\varphi(x) = \frac1{
			\lceil
			c
			\alpha^{-1}
			\log(C_d)
			\rceil
		} \sum_{k=1}^{
			\lceil 
			c
			\alpha^{-1}
			\log(C_d)
			\rceil
		}\, 
		N_1(\phi(x)_k,0)
		,
		\]
		satisfying: for every $x,\tilde{x}\in \xxx$ the following holds
		\[
		    d_{\xxx}^{\alpha}(x,\tilde{x})
		\leq
		    \mathcal{W}_2(\varphi(x),\varphi(\tilde{x}))
		\leq 
    	    \mw_2(\varphi(x),\varphi(\tilde{x}))
		\leq 
    		\left\lceil
    		C
    		\left(
    		\frac{\log(C_d)}{(1-\alpha)}
    		\right)^{1+\alpha}
    		\right\rceil
    		d_{\xxx}^{\alpha}(x,\tilde{x})
		,
		\]
		where $\phi:(\xxx_n,d_n)\rightarrow (\rr^{\lceil c\log(C_d)\rceil},\ell_2)$ is a $\alpha$-H\"{o}lder with $\alpha$-H\"{o}lder constant 
		$
		\left\lceil
		C
		\left(
		\frac{\log(C_d)^2}{(1-\alpha)}
		\right)^{1+\alpha}
		\right\rceil
		$
		.
	\end{lemma}
	
	Since every \textit{finite} metric space is a doubling space, then Lemma~\ref{lem_embedding_doubling_metric_spaces} applies to any finite subset $\xxx_n$ of any general metric space $\xxx$.  Consequentially, Theorem~\ref{MAINTHEOREM_DETERMINISTIC} is obtained by applying Lemma~\ref{lemma_embedding_neuralification} to the embedding in Lemma~\ref{lem_embedding_doubling_metric_spaces}.  
	
	\section{\textbf{Limitations and Possible Improvements}}
	\label{s_discussion}
	
	As mentioned in the introduction, our results are the first theoretical guarantees for deep neural representations of finite datasets from arbitrary metric spaces. We prove several embedding guarantees which pleasingly reflect the impact of the ambient geometry of space on the complexity of the representations and the number of parameters of the feature map. These results can be seen as deep neural approximation theorems for embeddings of finite metric spaces which are usually studied in metric geometry non-constructively. To obtain these results we developed a new proof technique which combines elements of metric embedding theory and memorization theory for deep neural networks; this technique may by a tool of independent interest.

	At the same time, the preludial nature of our work makes it likely that a number of improvements are possible. One natural question is whether we can extend the $n$-point guarantees to the whole $\mathcal{X}$ when $\mathcal{X}$ has a regular geometry (for example, when it is a compact Riemannian manifold). While this may require rephrasing some of our questions and using different mathematical tools, it does seem plausible. 

	Further, there are several aspects of the proofs, for embeddings into the space of Univariate Gaussian mixtures, which suggest improvements are possible: 1) a key step passes through a high-dimensional Euclidean space, but we know experimentally that Euclidean embeddings under-perform and theoretically (Naor, Andoni, Neiman \cite{NaorNeimanAndoniSnowflake2018AnnalsSENSQ}) that this is not necessary; 2) the Gaussian mixtures are often degenerated into empirical measures with equal atomic weights, but numerical experiments suggest a clear advantage of using proper Gaussian mixtures. While at the moment we do not know how to leverage these observations, they remain a subject of ongoing work.

	\subsubsection*{Funding}
	This research was supported by the European Research Council (ERC) Starting Grant 852821-SWING.  Anastasis was also supported by his McMaster University Startup Grant.  
	
	\subsubsection*{Acknowledgments}
	The authors would like to thank Giulia Livieri for her very helpful feedback.

\begin{appendices}

	\section{Explicit Network Complexities: Theorems and Examples}
	\label{s_explicit_network_complexities}
	The next table summarizes the precise complexities of the probabilistic transformer representing the metric spaces with generic $n$-point embedding guarantees, which are covered in any one of the theoretical results in this paper.  This table is an explicit analogue of Table~\ref{tab:Complexities_paperversion}.
	
	\begin{table}[ht!]
    \centering
    \caption{Complexity of the probabilistic transformers performing $n$-point embeddings with explicit constants, with $c =
		\frac{
            2\log(5 \sqrt{2\pi})
        + 
            \frac{3}{2}
            \log(n)
        -
            \frac1{2}\log(n+1)
    }{
        2\log(2)
    }$.}
        \resizebox{\columnwidth}{!}{%
    		\begin{tabular}{@{}lllll@{}}
    			\cmidrule[0.3ex](){1-5}
    			\textbf{Latent Geometry} & \textbf{Effective Dimension} $d_{T}$ & \textbf{Depth} & \textbf{Width} & \textbf{Result}
    			\\
    			\midrule
    		    General
    		    & 
    		        $
    		        2\lceil
        			12
        			C
        			\alpha^{-1}
        			\left(
        			\log(\mathrm{cap}(\xxx_n,d_n))
        			\right)
        			\rceil
    		        $
    		    &
    		        $
    		        \mathcal{O}\left(
    			n\left(
    		        1+
    		    \sqrt{n\log(n)}
    		        \,
    		     \Big(
    		        1
    		            +
    		       \frac{\log(2)}{\log(n)}\,
    		       \left[
    		            c
    		                +
    		         \frac{
    			         \log\big(
    			            n^{5/2}\, \operatorname{aspect}(\xxx_n,d_n)
    			          \big)
    		         }{
    		            \log(2)
    		         }
    		       \right]_+
    		     \Big)
            		\right)
            	\right)
    		        $
    		    & 
    		        $\max\{
    		        d_T
                ,
                    n(n-1) + 12
                \}$
    		   &
    		   Theorem~\ref{MAINTHEOREM_DETERMINISTIC}
    		\\
    			    General
    			& 
    			    $\mathcal{O}\big(
                    (D-2)
        			^{-2}
                    \theta_D
        			\log_2(n)
        			\big)
    			$
    			&
    			    $
    			    \mathcal{O}\left(
                	n^{\theta_{D}}
                	    \left(
                		1+
                		    \sqrt{\theta_D\,\log(n)}
                		        \,
                		     \Big(
                		        1
                		            +
                		       \frac{\log(2)}{\theta_D\,\log(n)}\,
                		       \left[
                		            c
                		                +
                		         \frac{
                			         3 \theta_D\,
                			         \log(
                			            n)
                			       +
                			       \log\big(
                			       \operatorname{aspect}(\xxx_n,d_n)
                			          \big)
                		         }{
                		            \log(2)
                		         }
                		       \right]_+
                		     \Big)
                		\right)
                	\right)
    			    $
    			& 
    			$\mathcal{O}\Big(
        			\max\{
        			    d_T
                    ,
                        n^{2\theta_{D}}  + 12
                    \}
                \Big)$
                & 
                Theorem~\ref{MAINTHEOREM_PROBABILISTIC}
    		\\
    		General - Multivariate Mixtures & 
    		$
    		12\,n(n-1)/2\, (5\frac{5n^4}{2(D-1)} \, \operatorname{aspect}(\xxx_n,d)^2 +2)
    		$
    		& 
    		 $\mathcal{O}\left(
    	    n\left\{
    		1+
    		    \sqrt{n\log(n)}
    		        \,
    		     \left[
    		        1
    		            +
    		       \frac{\log(2)}{\log(n)}\,
    		       \left(
    		            C_n
    		                +
    		         \frac{
    			         \log\big(
    			            n^{5/2}\, \operatorname{aspect}(\xxx_n,d_n)
    			          \big)
    		         }{
    		            \log(2)
    		         }
    		       \right)_+
    		     \right]
    			\right\}
    		\right)
    		$
    		& 
            $
            \max\biggl\{
                n
            , 
                \frac{25n^5(n-1)}{4(D-1)} \, (\operatorname{aspect}(\xxx_n,d)^2 +2)
            , 
                9
            ,
                n(n-1) + 12
            \biggr\}
            $ 
            & 
            Theorem~\ref{theorem:high_dimension_asymptotic_no_distortion}
            \\
    		\midrule
    		Discrete: Trees
    		    & 
    		        $2M$
    		    &
    		        $
    		        \mathcal{O}\left(
    	n\left(
    		1+
    		    \sqrt{n\log(n)}
    		        \,
    		     \Big(
    		        1
    		            +
    		       \frac{\log(2)}{\log(n)}\,
    		       \left[
    		            c
    		                +
    		         \frac{
    			         \log\big(
    			            n^{5/2}\, 
    			            \operatorname{diam}(\xxx_n,d_n)
    			          \big)
    		         }{
    		            \log(2)
    		         }
    		       \right]_+
    		     \Big)
    		\right)
    	\right)
    		        $
    		    & 
    		        $\max\{
                    d_T
                ,
                    n(n-1) + 12
                \}$
    		   &
    		   Proposition~\ref{prop_combinatorial_tree_embedding}
    		\\
    		    Discrete: $2$-Hop Graphs
    		    & 
    		        $2\lceil 
    		            12\,C\alpha^{-1}\log(1+\rho(A_{G}))
    		        \rceil$
    		    &
    		        $
    		        \mathcal{O}\left(
    			n\left(
    		        1+
    		    \sqrt{n\log(n)}
    		        \,
    		     \Big(
    		        1
    		            +
    		       \frac{\log(2)}{\log(n)}\,
    		       \left[
    		            c
    		                +
    		         \frac{
    			         \log\big(
    			            n^{5/2}\, 
    			            \operatorname{diam}(\xxx_n,d_n)
    			          \big)
    		         }{
    		            \log(2)
    		         }
    		       \right]_+
    		     \Big)
            		\right)
            	\right)
    		        $
    		    & 
    		        $\max\{
                    d_T
                ,
                    n(n-1) + 12
                \}$
    		   &
    		   Corollary~\ref{cor_combinatorial_graphs}
    		\\
    		\midrule
    		    Manifold: Riemannian, Bounded Curvature
    		    & 
    		        $
    		        2\big\lceil\tilde{C}\frac{m^{1+\alpha}}{\alpha (1-\alpha)^{1+\alpha}}\big\rceil
    		        $
    		    &
    		        $
    		        \mathcal{O}\left(
                    	n\left(
                    		1+
                    		    \sqrt{n\log(n)}
                    		        \,
                    		     \Big(
                    		        1
                    		            +
                    		       \frac{\log(2)}{\log(n)}\,
                    		       \left[
                    		            c
                    		                +
                    		         \frac{
                    			         \log\big(
                    			            n^{5/2}\, \operatorname{aspect}(\xxx_n,d_n)
                    			          \big)
                    		         }{
                    		            \log(2)
                    		         }
                    		       \right]_+
                    		     \Big)
                    		\right)
                    	\right)
    		        $
    		    & 
    		        $
    		            \max\big\{
                    d_T
                ,
                    n(n-1) + 12
                    \big\}
    		        $
    		   &
    		   Corollary~\ref{cor_Ricci_Version}
    		\\
    		    Manifold: Riemannian $\&$ Compact
    		    & 
    		        $4d$
    		    &
    		        $
    		          \mathcal{O}\left(
                	n\left(
                		1+
                		    \sqrt{n\log(n)}
                		        \,
                		     \Big(
                		        1
                		            +
                		       \frac{\log(2)}{\log(n)}\,
                		       \left[
                		            c
                		                +
                		         \frac{
                			         \log\big(
                			            n^{5/2}\, \operatorname{aspect}(\xxx_n,d_n)
                			          \big)
                		         }{
                		            \log(2)
                		         }
                		       \right]_+
                		     \Big)
                		\right)
                	\right)
    		        $
    		    & 
    		        $\max\{
    		        d_T
                ,
                    n(n-1) + 12
                    \}$
    		   &
    		   Proposition~\ref{proposition_manifoldprior}\\
    			\bottomrule
    		\end{tabular}
    		}
        \label{tab:Complexities_explicit}
    \end{table}
	
	
	\begin{remark}[Approximate Complexity of $T$ For Low-Distortion Embeddings]
	\label{rem_breakdown_n_param}
	The proof of Theorem~\ref{MAINTHEOREM_PROBABILISTIC} and a remark on \citep[page 492]{NaorTao_2012_Dvoretzky_ISJM} shows if $D\approx 2$ (but $D>2$) then, the probabilistic transformer
    $T$ in Theorem~\ref{MAINTHEOREM_PROBABILISTIC} 
    has depth about the order%
    \footnote{We use the notation $\tilde{\mathcal{O}}$ to hide logarithmic factors.}%
    $
    \tilde{\mathcal{O}}\big(
        n^{
            \frac{11(D-2)}{\log(1/(D-2)^4)} 
                + 
            \mathcal{O}\big(
            	\frac{
            		11(D-2)
            		\log\log(1/(D-2))
            	}{
            		2 \log(1/(D-2))^2
            	}
            \big)
        }
    \big)
    $.  When $n$ is also large then it has width of the order $
    \tilde{\mathcal{O}}\big(
        n^{
            \frac{D-2}{\log(D-2)} 
            +
            \mathcal{O}\big(
            	\frac{
            		(2D-4)
            		\log\log(1/(D-2))
            	}{
            		\log(1/(D-2))^2
            	}
            \big)
        }
    \big)
    $. 
	\end{remark}

	\section{Proof Details}
	\label{a_Proofs}
	This section contains the paper's main proofs.  
	
	We begin with the following memorization lemma, on which the proof of Lemma~\ref{lemma_embedding_neuralification} is developed.  This result extends the quantitative ``class memorization'' result of \cite{vardi2022on} to a quantitative vector-valued interpolation result.  The next Lemma can be contrasted with the VC-bounds of \cite{JLMLR_BartlessHAveyLiawMagrabian_2019_VCBoundsReLUffNN} which imply that a ReLU feedforward network of depth $D$ memorizing how to match $N$ inputs in $\rr^n$ with $C$ classes must have at-least $\Omega\big(\frac{D}{\ln(D)}\big)$ parameters.  
    We recall that the following convenient notation is used in the next lemma.  For each $u\in \rr$ we denote $[u]_+\eqdef \max\{u,1\}=\operatorname{ReLU}(u-1)$.
	\begin{lemma}[Memory Capacity of Deep ReLU Regressors]
	\label{lem_memory_capacity_deep_ReLU_regressors}
		Let $n,d,N\in \nn_+$, let $f:\rr^n \rightarrow \rr^d$ be some function, and consider distinct $x_1,\dots,x_N\in \rr^n$.  
		There exists a deep ReLU network $\mathcal{NN}:\rr^n\rightarrow \rr^d$ satisfying
		\[
		\mathcal{NN}(x_i) 
		= 
		f(x_i)
		,
		\]
		for every $i=1,\dots,N$.  
		Furthermore, we have the following quantitative model complexity estimates:
		\begin{align*}
			\mathrm{width}(\mathcal{NN}) 
			&= 
			n(N - 1) + \max \{ d, 12\}, \\
			\mathrm{depth}(\mathcal{NN})
			&=
			\mathcal{O}\left(
			N\left\{
    			1+
    			    \sqrt{N\log(N)}
    			        \,
    			     \left[
    			        1
    			            +
    			       \frac{\log(2)}{\log(N)}\,
    			       \left(
    			            C_n
    			                +
    			         \frac{
        			         \log\big(
        			            N^2\,
        			            \operatorname{aspect}(\xxx_N,\|\cdot\|_2)
        			          \big)
    			         }{
    			            \log(2)
    			         }
    			       \right)_+
    			     \right]
    			\right\}
			\right), \\
			\mathrm{par}(\mathcal{NN})
			&=
			\mathcal{O}\left(
                N
		        \Big(
		        \frac{11}{4} 
		        \max\{
		            n
		        ,
		            d
		        \}	 
		        \,
		        N^2
		        -
		        1
		        \Big) 
		        \,
		        \left\{
    				d+
    			    \sqrt{N\log(N)} 
    			        \,
    			     \left[
    			        1 +
    			       \frac{\log(2)}{\log(N)} \,
    			       \right. \right. \right. \\
    			       &\hspace{10mm} \times
    			       \left. \left. \left.
    			       \left(
    			            C_n
    			                +
    			         \frac{
        			         \log\big(
        			            N^2\,
        			            \operatorname{aspect}(\xxx_N,\|\cdot\|_2)
        			          \big)
    			         }{
    			            \log(2)
    			         }
    			       \right)_+
    			 \right]
    			 \,
    			    \max\{d,12\}
    			    \Big[1 + \max\{d,12\}\Big]
    		\right\}
		\right).
		\end{align*}
	The ``dimensional constant'' is defined by $
		    C_n
		        \eqdef 
		\frac{
		            2\log(5 \sqrt{2\pi})
		        + 
		            \frac{3}{2}
		            \log(n)
		        -
		            \frac1{2}\log(n+1)
		    }{
		        2\log(2)
		    }
	>
	    0
	$
	.
	\end{lemma}
	\begin{proof}[{Proof of Lemma~\ref{lem_memory_capacity_deep_ReLU_regressors}}]
	\textbf{Step 1: Centering Data}\hfill\\
	Let $\xxx_N \eqdef \{x_i\}_{i=1}^N$ and for any $\bar{x}\in \rr^n$ and any $r>0$ denote $\overline{B_2(\bar{x},r)}
	    \eqdef 
	    \big\{
	        u\in \rr^n:\,
	        \|u-\bar{x}\|_2
	            \leq 
	        r
	    \big\}
	$.  Since $\xxx_N$ is finite it is compact and therefore by Jung's Theorem (see \cite{Jung1901}) there exists some $\bar{x}\in \rr^n$ such that
	\[
	    \xxx_N
	        \eqdef
	    \overline{
    	    B_2\Big(
    	        \bar{x},
    	        r
    	    \Big)
	    }
	    \mbox{ and }
	    r
	        \eqdef 
	    \operatorname{diam}(\xxx_N,\|\cdot\|_2)
	    \frac{n^{1/2}}{(2(n+1))^{1/2}}
	   ,
	\]
	where $\operatorname{diam}(\xxx_N,\|\cdot\|_2)\eqdef \max_{1\leq i,j\leq N}\, \|x_i-x_j\|_2$.  
	Thus, the (affine) isometry $W:\rr^n\rightarrow\rr^n$ sending any $x\in \rr^n$ to $x-\bar{x}$, bijectively maps $\xxx_N$ to the subset $\tilde{\xxx}_N\eqdef \{u\in \rr^n:\, (\exists x_i\in \xxx_N)\, u=x_i-\bar{x}\}$ of $\overline{B_2(0,r)}$.  In particular, since $W$ is an isometry then 
	\[
	   \operatorname{aspect}(\xxx_N,\|\cdot\|_2)
	        =
	   \operatorname{aspect}(\tilde{\xxx}_N,\|\cdot\|_2) 
	  .
	\]
	Since $W$ is an affine map then, the pre-composition of any deep feedforward network by $W$ yields a deep feedforward network of the same depth and width.  
	
	\textbf{Step 2: Independently Memorizing Classes}\hfill\\
	By \citep[Theorem 3.1]{vardi2022on}, for every $x\in \xxx_N$ there exists a feedforward network $\tilde{\mathcal{NN}}_{x}:\rr^n\rightarrow \rr$ with ReLU activation function satisfying the ``class memorization property''
	\begin{equation}
			\label{PROOF_6_individual_memorizers}
			\tilde{\mathcal{NN}}_{x}(u) 
			=     
			\begin{cases}
				1 & : \, \mbox{if } u =x-\bar{x}\\
				0 & : \, \mbox{if } u \neq x-\bar{x}
			\end{cases}
	\end{equation}
	for every $u\in \tilde{\xxx}_N$.  
	Furthermore, $\tilde{\mathcal{NN}}_{x}$ has width $12$ and depth
		\begin{equation}
			\label{PROOF_5_individual_memorizers_depth__directquote}
			\mathcal{O}\left(
			    \Big(
			        N\log(N)
			    \Big)^{1/2}
			        +
			    \Big(
			        \frac{N}{\log(N)}
			    \Big)^{1/2}
			    \max\big\{
			        \log(R)
			            ,
			        \log(C)
			    \big\}
			\right)
			,
		\end{equation}
		where $C=2$ and $R=10 \pi^{1/2}n^{1/2}\, N^2 \, r \delta^{-1}$ and where $\delta
		    \eqdef 
		\min_{1\leq i,j\leq N;\,i\neq j}\, \|(x_i-\bar{x})-(x_j-\bar{x})\|_2
		    =
		\min_{1\leq i,j\leq N;\,i\neq j}\,
		\|x_i-x_j\|_2
		$.
		We may therefore rewrite~\eqref{PROOF_5_individual_memorizers_depth__directquote} as
		\begin{equation}
		\label{PROOF_5_individual_memorizers_depth}
			\mathcal{O}\left(
			    \sqrt{N\log(N)}
			        \,
			     \Big(
			        1
			            +
			       \frac{\log(2)}{\log(N)}\,
			       \left[
			            C_n
			                +
			         \frac{
    			         \log\big(
    			            N^2\,
    			            \operatorname{aspect}(\xxx_N,\|\cdot\|_2)
    			          \big)
			         }{
			            \log(2)
			         }
			       \right]_+
			     \Big)
			\right)
		,
		\end{equation}
		where the ``dimensional'' constant $C_n>0$ is defined by 
		$
		    C_n
		        \eqdef 
		    \frac{
		            2\log(5 \sqrt{2\pi})
		        + 
		            \frac{3}{2}
		            \log(n)
		        -
		            \frac1{2}\log(n+1)
		    }{
		        2\log(2)
		    }
		$ and where, for any $u\in \rr$ we define $[u]_+\eqdef \max\{u,1\}$.
		
		\textbf{Step 3: Parallel Interpolating ReLU Network for each $x$ in $\xxx_N$}\hfill\\
		For every $x\in \xxx_N$, we view the vector $f(x)\in \rr^d$ as a $d\times 1$ matrix in $\rr^{d\times 1}$.  Now, since the composition of affine maps is again affine and since the first and last layers of each $\tilde{\mathcal{NN}}_x$ are affine maps then, for every $x\in \xxx_N$, the map $\mathcal{NN}_x:\rr^n\rightarrow\rr^d$ defined by
		\begin{equation}
			\label{PROOF_6a_individual_memorizers_modified_memorization_vectorialoutput}
			\mathcal{NN}_{x} (u)\eqdef  f(x) \, [\tilde{\mathcal{NN}}_{x}\circ W^{-1}(u)],
		\end{equation}
		is a ReLU feedforward network satisfying
		\begin{enumerate}
			\item $\mathrm{depth}(\mathcal{NN}_{x}) = \mathrm{depth}(\tilde{\mathcal{NN}}_{x})$ where each $\mathrm{depth}(\tilde{\mathcal{NN}}_{x})$ is as in~\eqref{PROOF_5_individual_memorizers_depth},
			\item $\mathrm{width}(\mathcal{NN}_{x}) = \max\{d,\mathrm{width}(\tilde{\mathcal{NN}}_{x})\} = \max\{d,12\}$,
			\item $\mathrm{par}(\mathcal{NN}_{x}) \leq d + \mathrm{par}(\tilde{\mathcal{NN}_{x}})
			.
			$
		\end{enumerate}
	Moreover, by construction, for each $x\in \xxx_N$ it holds that
	\begin{equation}
			\label{PROOF_6_individual_memorizers__memorizer_no_tilde}
			\mathcal{NN}_x(u)
			=     
			\begin{cases}
				f(x) & : \, \mbox{if } u = x\\
				0 & : \, \mbox{if } u \neq x
				,
			\end{cases}
	\end{equation}
	for every $u\in \xxx_N$.
	Furthermore, we can bound the number of parameters $\mathrm{par}(\mathcal{NN}_{x})$ defining each of the ReLU networks $\mathcal{NN}_x$ above by
		\begin{equation}
			\label{lem_memory_capacity_deep_ReLU_regressors_eq1}
			\begin{aligned}
				\mathrm{par}(\mathcal{NN}_{x}) 
				& \leq 
				d+ 
				\mathrm{par}(\tilde{\mathcal{NN}_{x}}) \\
				& 
				\leq 
				\mathrm{depth}(\tilde{\mathcal{NN}}_{x_1}) 
				\mathrm{width}(\tilde{\mathcal{NN}}_{x_1})(\mathrm{width}(\mathcal{NN}_{x_1})+1) + d
				\\
				& 
				\leq
				\mathcal{O}\Biggl(
				d+
			    \sqrt{N\log(N)}
			        \,
			     \Big(
			        1
			            +
			       \frac{\log(2)}{\log(N)}\,
			       \left[
			            C_n
			                +
			         \frac{
    			         \log\big(
    			            N^2\,
    			            \operatorname{aspect}(\xxx_N,\|\cdot\|_2)
    			          \big)
			         }{
			            \log(2)
			         }
			       \right]_+
			 \Big)
			 \times 
		\\
		& 
			\Big(
			    \max\{d,12\}
			    (\max\{d,12\}+1)
			\Big)
			\Biggr)
				.
			\end{aligned}
		\end{equation}
        
        \textbf{Step 4: Assembling the Parallel Interpolating ReLU Networks}\hfill\\
        Next, we assemble each of the ReLU networks $\{\mathcal{NN}_x\}_{x\in \xxx_N}$ into a single ReLU network interpolating $f$ on $\xxx_N$.   
		Since $\operatorname{ReLU}(x)\eqdef \max\{0,x\}$ has the $2$-identity property \textit{(as defined in \citep[Definition 4]{Florian2_2021} and discussed on page 3 of that manuscript)} then, \citep[Proposition 5]{Florian2_2021} implies that there exists a ``deep diagonalized parallelization''; that is, a map $\|\mathcal{NN}:\mathbb{R}^n\rightarrow \rr^d$ satisfying the following for every $u\in \rr^n$
        \begin{equation}
        \label{PROOF_definition_parallelization}
    		\|\mathcal{NN}(u) 
		= 
		\begin{pmatrix}
			\mathcal{NN}_{x_1} (u)\\
			\vdots\\
			\mathcal{NN}_{x_N} (u)
		\end{pmatrix}
		.
        \end{equation}
		Furthermore, \citep[Proposition 5]{Florian2_2021} shows that $\mathcal{NN}$ has width, depth, and number of non-zero parameters are estimated by
		\begin{enumerate}
			\item \textbf{Width:} $\|\mathcal{NN}(u)$ has width  
			$
			n(N-1) + \max\{d,12\}
			$
			\item \textbf{Depth:} $\|\mathcal{NN}(u) $ has depth
			\small
			\begin{equation}
			\label{PROOF_5_individual_memorizers_depth_1}
			\mathcal{O}\left(
			N\left(
    			1+
    			    \sqrt{N\log(N)}
    			        \,
    			     \Big(
    			        1
    			            +
    			       \frac{\log(2)}{\log(N)}\,
    			       \left[
    			            C_n
    			                +
    			         \frac{
        			         \log\big(
        			            N^2\,
        			            \operatorname{aspect}(\xxx_N,\|\cdot\|_2)
        			          \big)
    			         }{
    			            \log(2)
    			         }
    			       \right]_+
    			     \Big)
    			\right)
			\right)
				,
			\end{equation}
			\normalsize
			\item \textbf{Number of non-zero parameters:}  The number of non-zero parameters determining $\|\mathcal{NN}$ is $
		        \Big(
		        \frac{11}{4} 
		        \max\{
		            n
		        ,
		            d
		        \}	 
		        \,
		        N^2
		        -
		        1
		        \Big)
		        \,
		        \sum_{i=1}^N\,
		            \mathrm{par}(\mathcal{NN}_x)
			.
			$
			When expanded, this can be rewritten as
			\[\resizebox{1\textwidth}{!}{$
			\begin{aligned}
            \mathcal{O}\left(
                N
		        \Big(
		        \frac{11}{4} 
		        \max\{
		            n
		        ,
		            d
		        \}	 
		        \,
		        N^2
		        -
		        1
		        \Big)
		        \,
		        \left(
    				d+
    			    \sqrt{N\log(N)}
    			        \,
    			     \left(
    			        1
    			            +
    			       \frac{\log(2)}{\log(N)}\,
    			       \left[
    			            C_n
    			                +
    			         \frac{
        			         \log\big(
        			            N^2\,
        			            \operatorname{aspect}(\xxx_N,\|\cdot\|_2)
        			          \big)
    			         }{
    			            \log(2)
    			         }
    			       \right]_+
    			 \right)
    			 \,
    			\Big(
    			    \max\{d,12\}
    			    (\max\{d,12\}+1)
    			\Big)
    		\right)
		\right)
			.
			\end{aligned}
			$}\]
		\end{enumerate}
		
		Consider the $d\times (N d)$ block-matrix $A\eqdef (I_d,\dots,I_d)$.  Multiplying the outputs of $\|\mathcal{NN}(u)$ on the left by the matrix $A$ defines a map $\mathcal{NN}^{\star}:\rr^n\rightarrow \rr^d$; i.e.\ for every $u\in \rr^n$ the map $\mathcal{NN}^{\star}$ is defined by
		\begin{equation}
		\label{PROOF_summingparallelizednetworkup}
    		\mathcal{NN}^{\star}(u)
    		\eqdef  
    		A
    		\|\mathcal{NN}(u)
    	.
		\end{equation}
		Furthermore, since the post-composition of ReLU feedforward networks is again a ReLU feedforward network then $\mathcal{NN}^{\star}$ is a ReLU feedforward network.  Moreover, by construction, the width, depth, and number of parameters defining $\mathcal{NN}^{\star}$ satisfy the same estimates as $\|\mathcal{NN}$, respectively.  
		
		\textbf{Step 5: Verifying that $\mathcal{NN}^{\star}$ Implements $f$ on $\xxx_N$}\hfill\\
		It remains only to verify that $\mathcal{NN}^{\star}$ equals to $f$ on $\xxx_N$.  Upon combining the identities in~\eqref{PROOF_6_individual_memorizers__memorizer_no_tilde},~\eqref{PROOF_summingparallelizednetworkup}, and~\eqref{PROOF_definition_parallelization} we conclude that: for each $x \in \xxx$ the following holds
		\begin{equation}
			\label{PROOF_8_efficient_parallelization_memorization}
			\mathcal{NN}^{\star}(x)
			=
			\sum_{i=1}^N
			\mathcal{NN}_{x_i}
			= 
			\sum_{i=1}^N\,
			    f(x_i)I_{x_i=x}\,
			=
			f(x)
			.
		\end{equation}
		Relabeling ``$\mathcal{NN}\eqdef \mathcal{NN}^{\star}$'' yields the lemma's statement.  
	\end{proof}
	
	We now move to the proof of the main memorization lemma, namely Lemma~\ref{lemma_embedding_neuralification}, which guarantees that probabilistic transformer networks can memorize ``conditionally Gaussian'' Markov kernels for any finite number of inputs.  
	In the following, we make use of the finite-dimensional Fr\'{e}chet embedding
	\begin{equation}
	\label{eq_Frechet_embedding}
	    \tilde{\iota}: (\xxx_n,d_n)\ni x \mapsto (d(x,\tilde{x}))_{\tilde{x}\in \xxx_n}.
	\end{equation}

    \begin{proof}[{Proof of Lemma~\ref{lemma_embedding_neuralification}}]
    Set $L\eqdef n$ (where $L$ is as in~\eqref{eq:feature_map}) and identify the space of $d\times d$-matrices with Fr\"{o}benius norm $\rr^{d\times d}$ with $\rr^{d^2}$ with Euclidean norm \textit{(NB, these are isometric metric spaces)}.
    Then, $\tilde{\iota}= \Phi_{\{x_l\}_{l=1}^N}$, where $\tilde{\iota}$ is the Fr\'{e}chet isometric embedding defined in~\eqref{eq_Frechet_embedding}.  
    
    For each $k=1,\dots,K$ define the maps $\tilde{\mu}_k\eqdef \mu_k\circ \iota^{-1}:\iota(\xxx_n)\rightarrow \rr^d$.  By Lemma~\ref{lem_memory_capacity_deep_ReLU_regressors}, for every $k=1,\dots,K$  there exists a ReLU network $\hat{\mu}_k: \rr^n \rightarrow \rr^d$ satisfying the following
    \begin{equation}
    \label{PROOF_lemma_embedding_neuralification__eq_mu_memorization}
        \hat{\mu}_k (x) = \tilde{\mu}_k(x) = \mu_k\circ \iota^{-1}(x)
        ,
    \end{equation}
    for every $x\in \xxx_n$.  
    The following equivalence of the norms $\|\cdot\|_{\infty}\le \|\cdot\|_2 \le \sqrt{n} \|\cdot\|_{\infty}$ and the fact that $\iota$ is an isometric embedding of $(\xxx_n,d_n)$ into $(\mathbb{R}^n,\|\cdot\|_{\infty})$ implies
    \begin{equation}
    \label{PROOF_lemma_embedding_neuralification__complexity_estimate_muk_interpolator}
    \begin{aligned}
        \operatorname{aspect}(\xxx_n,d_n)
            \eqdef &
        \frac{
            \max_{1\leq i,j\leq n}\,
                d_n(x_i,x_j)
        }{
            \min_{1\leq k,s\leq n;\,k\neq s}\,
                d_n(x_k,x_s)
        }
            \\
            = &
        \frac{
            \max_{1\leq i,j\leq n}\,
                \|\iota(x_i) - \iota(x_j)\|_{\infty}
        }{
            \min_{1\leq k,s\leq n;\,k\neq s}\,
                \|\iota(x_k)-\iota(x_s)\|_{\infty}
        }
            \\
            \leq &
        \sqrt{n}\,
        \frac{
            \max_{1\leq i,j\leq n}\,
                \|\iota(x_i) - \iota(x_j)\|_{2}
        }{
            \min_{1\leq k,s\leq n;\,k\neq s}\,
                \|\iota(x_k)-\iota(x_s)\|_{2}
        }
            \\
            = &
        \sqrt{n}\,
        \frac{
            \max_{u,\tilde{u} \in \iota(\xxx_N)}\,
                \|u-\tilde{u}\|_{2}
        }{
            \min_{v,\tilde{v}\in \iota(\xxx_N);\, v\neq \tilde{v}}\,
                \|v-\tilde{v}\|_{2}
        }
            \\
            \eqdef &
        \sqrt{n}\,
        \operatorname{aspect}(\iota(\xxx_N),\|\cdot\|_2)
        .
    \end{aligned}
    \end{equation}
    Thus, Lemma~\ref{lem_memory_capacity_deep_ReLU_regressors} (i)-(iii) and~\eqref{PROOF_lemma_embedding_neuralification__complexity_estimate_muk_interpolator} imply that for each $k=1,\dots,K$ the ReLU network $\hat{\mu}_k$ satisfies
    \begin{enumerate}
	\item[(i)] \textbf{Width:} $\mathcal{NN}$ has width $
	    n(n-1) + \max\{d,12\}
	,
	$
	\item[(ii)] \textbf{Depth:} $\mathcal{NN}$ has depth of the order of
	\[
	\mathcal{O}\left(
	n\left(
		1+
		    \sqrt{n\log(n)}
		        \,
		     \Big(
		        1
		            +
		       \frac{\log(2)}{\log(n)}\,
		       \left[
		            C_n
		                +
		         \frac{
			         \log\big(
			            n^{5/2}\, \operatorname{aspect}(\xxx_n,d_n)
			          \big)
		         }{
		            \log(2)
		         }
		       \right]_+
		     \Big)
		\right)
	\right)
	,
	\]
	\item[(iii)] \textbf{Number of non-zero parameters:} The number of non-zero parameters in $\mathcal{NN}$ is at most
	\[\resizebox{0.99\linewidth}{!}{$
	\mathcal{O}\left(
        n
        \Big(
        \frac{11}{4} 
        \max\{
            n
        ,
            d
        \}	 
        \,
        n^2
        -
        1
        \Big)
        \,
        \left(
			d+
		    \sqrt{n\log(n)}
		        \,
		     \left(
		        1
		            +
		       \frac{\log(2)}{\log(n)}\,
		       \left[
		            C_n
		                +
		         \frac{
			         \log\big(
			            n^{5/2}
			            \,
			            \operatorname{aspect}(\xxx_n,d_n)
			          \big)
		         }{
		            \log(2)
		         }
		       \right]_+
		 \right)
		 \,
		\Big(
		    \max\{d,12\}
		    (\max\{d,12\}+1)
		\Big)
	\right)
\right)
		,
		$}\]
\end{enumerate}
and where the ``dimensional constant'' $C_d>0$ is defined by 
	$
	    C_d 
	        \eqdef 
	    \frac{
	            2\log(5 \sqrt{2\pi})
	        + 
	            3
	            \log(d)
	        -
	            \log(d+1)
	    }{
	        2\log(2)
	    }
	$.
    
    Next, for every $k=1,\dots,K$, define the ReLU network $\hat{\Sigma}_k:\rr^n\rightarrow \rr^{d^2}$ by
    \begin{equation}
    \label{PROOF_lemma_embedding_neuralification__eq_Sigma_memorization}
        \hat{\Sigma}_k(x) 
            \eqdef 
        \bar{0}_{d^2}
        \operatorname{ReLU}\bullet
        \Big(
            \bar{0}_{d^2 \times n}
                x
            +
            \bar{0}_{d^2}
        \Big)
            +
            \bar{0}_{d^2}
    ,
    \end{equation}
    where $\bar{0}_{d^2 \times n}$ is the $d^2\times n$-matrix with all entries $0$, $\bar{0}_{d^2}\in \rr^{d^2}$ is the vector with all entries $0$.  NB, by construction 
    \begin{equation}
        \label{eq_parameter_count_Sigma_nets}
        \max_{k=1,\dots,K}\,
        \mathrm{par}(\hat{\Sigma}_k) 
	    =
	    0
	 .
    \end{equation}
    Similarly, define the ReLU network $\hat{w}:\rr^n\rightarrow \rr^K$ by
    \begin{equation}
    \label{PROOF_lemma_embedding_neuralification__eq_w_memorization}
        \hat{w}_k(x) 
            \eqdef 
        \bar{0}_{K}
        \operatorname{ReLU}\bullet
        \Big(
            \bar{0}_{K \times n}
                x
            +
            \bar{0}_{K}
        \Big)
            +
            \bar{0}_{K}
    .
    \end{equation}
    Combining~\eqref{PROOF_lemma_embedding_neuralification__eq_mu_memorization},~\eqref{PROOF_lemma_embedding_neuralification__eq_Sigma_memorization}, and~\eqref{PROOF_lemma_embedding_neuralification__eq_w_memorization} with the fact that $\Phi_{\{x_l\}_{l=1}^N}=\tilde{\iota}=\tilde{\iota}\Big \vert _{\tilde{\iota}(\xxx_n)}^{-1}$ we have that: for every $x\in \xxx_n$ the following holds
    \allowdisplaybreaks
    \[
    \begin{aligned}
        T(x) 
    \eqdef &
        \sum_{k=1}^K
        [\softmax_K\circ \hat{w}\circ \Phi_{\{x_l\}_{l=1}^N}(x)]_k \times \\
        \qquad &N_d\left(
                \hat{\mu}_k\circ \Phi_{\{x_l\}_{l=1}^N}(x)
            ,
                \hat{\Sigma}_k\circ \Phi_{\{x_l\}_{l=1}^N}(x)
        \right)
    \\
        = &
        \sum_{k=1}^K
        \softmax_K(0)_k\,
        N_d\left(
                \mu_k\circ \tilde{\iota}^{-1}\circ \tilde{\iota}(x)
            ,
                0
        \right)
    \\
        = &
        \sum_{k=1}^K
        \frac{
        e^{0}
        }{
        \sum_{j=1}^K \, e^{0}
        }
        \,
        N_d\left(
                \mu_k\circ \tilde{\iota}^{-1}\circ \tilde{\iota}(x)
            ,
                0
        \right)
    \\
        = &
        \sum_{k=1}^K
        \frac1{K}
        \,
        N_d\left(
                \mu_k\circ \tilde{\iota}^{-1}\circ \tilde{\iota}(x)
            ,
                0
        \right)
    \\  
        = & 
        \phi(x)
    .
    \end{aligned}
    \]
    
    It remains only to compute the width, depth, and number of non-zero parameters determining $T$.  Since $T$ has width 
    \[\max\left\{
        \mathrm{width}(\hat{w}), 
        \max_{k=1,\dots,K}\, \mathrm{width}(\Sigma_k),
        \max_{k=1,\dots,K}\,
        \mathrm{width}(\mu_k)
        \}
    \right\}.
    \]
    Therefore, we deduce that
    \[%
    \begin{aligned}
            \mathrm{width}(T)
    =  &
        \max\{
                \max\{n,K\}
            , 
                \max\{n,d^2\}
            ,
                n(n-1) + \max\{d,12\}
        \}
    \\ 
    = &
        \max\{
                n, K
            , 
                d^2
            ,
                n(n-1) + \max\{d,12\}
        \}
    ,
    \end{aligned}
    \]
    and $T$ has depth $
            \max\left\{
                \mathrm{depth}(\hat{w})
            , 
                \max_{k=1,\dots,K}\,
                \mathrm{depth}(\Sigma_k)
            , 
                \max_{k=1,\dots,K}\,
                \mathrm{depth}(\mu_k)
        \right\}
    $.  Since $\mathrm{depth}(\hat{w})=\max_{k=1,\dots,K}\,\mathrm{depth}(\Sigma_k)=1$ then, $T$ has depth of the following order
    \[
    	\mathcal{O}\left(
	n\left(
		1+
		    \sqrt{n\log(n)}
		        \,
		     \Big(
		        1
		            +
		       \frac{\log(2)}{\log(n)}\,
		       \left[
		            C_n
		                +
		         \frac{
			         \log\big(
			            n^{5/2}\, \operatorname{aspect}(\xxx_n,d_n)
			          \big)
		         }{
		            \log(2)
		         }
		       \right]_+
		     \Big)
		\right)
	\right)
	.
	\]
    Furthermore, the number of parameters defining $T$ equals
    \[
    \begin{aligned}
        \mathrm{par}\Big(
                        T
                    \Big)
                \eqdef &
            \mathrm{par}\Big(
                                \hat{w}
                                \Big)
            + 
            \sum_{k=1}^K
            \left[
                \mathrm{par}\Big(
                        \hat{\mu}_k
                    \Big)
                +
                \mathrm{par}\Big(
                        \hat{\Sigma}_k
                    \Big)
            \right]
        \\
        = &
            0
                + 
            K
            \left[
                0
                +
                \mathrm{par}\Big(
                        \hat{\Sigma}_1
                    \Big)
            \right]
        .
        \end{aligned}
        \]
        Thus, the number of non-zero parameters defining $T$ is 
        \[
        \resizebox{0.99\linewidth}{!}{$
        \begin{aligned}
		\mathcal{O}\left(
                Kn
		        \Big(
		        \frac{11}{4} 
		        \max\{
		            n
		        ,
		            d
		        \}	 
		        \,
		        n^2
		        -
		        1
		        \Big)
		        \,
		        \left(
    				d+
    			    \sqrt{n\log(n)}
    			        \,
    			     \left(
    			        1
    			            +
    			       \frac{\log(2)}{\log(n)}\,
    			       \left[
    			            C_n
    			                +
    			         \frac{
        			         \log\big(
        			            n^{5/2}\,
        			            \operatorname{aspect}(\xxx_n,d_n)
        			          \big)
    			         }{
    			            \log(2)
    			         }
    			       \right]_+
    			 \right)
    			 \,
    			    \max\{d,12\}^2
    		\right)
		\right)
		.
		\end{aligned}
	    $}\]
    \end{proof}	
	We now turn our attention to the proof of Theorem~\ref{MAINTHEOREM_PROBABILISTIC}. 
	\subsection{Proof of PAC-Type Representation Theorem for ``Unstructured Case''}
	\label{s_Proofs__ss___UnstructuredEmbeddingTheorem}
	The following argument uses the notation of an \textit{ultrametric space}; briefly, an ultrametric space $(\mathcal{Z},d_{\mathcal{Z}})$ is a metric space which satisfies the \textit{strong triangle inequality}: for all $z_1,z_2,z_3\in \mathcal{Z}$ one has $d_{\zzz}(z_1,z_2)\leq \max\{d_{\zzz}(z_1,z_2),d_{\zzz}(z_2,z_3)\}$.  Note that, $\max\{d_{\zzz}(z_1,z_2),d_{\zzz}(z_2,z_3)\} \leq d_{\zzz}(z_1,z_2)+d_{\zzz}(z_2,z_3)$ so the strong triangle inequality implies the familiar triangle inequality.  
	\begin{proof}[{Proof of Theorem~\ref{MAINTHEOREM_PROBABILISTIC}}]
	Let $n\ge 2$ and fix an $n$-point metric subset $\xxx_n$ of $\xxx$.  Denote $(\xxx_n,d_n)\eqdef (\xxx_n,d_{\xxx})$.  
		\hfill\\
		\textbf{Step 1 - Embedding of a Dvoretzky-Type Subspace into Low-Dimensional Euclidean Space}
		\hfill\\
		Fix $\epsilon\in (0,\infty]$ and let $(\xxx_n,d_n)$ be an $n$-point metric space.  By \citep[Theorem 1.2]{NaorTao_2012_Dvoretzky_ISJM} there exists a $\theta_{\epsilon} \in (0,1)$ (depending only on $\epsilon$) and subset $\tilde{\xxx}_n\subseteq \xxx_n$ of cardinality
		\begin{equation}
			\label{PROOF_eq__lemma_Dvoretzy_Type___cardinality_bound_DvoretzkySubset}
			    \#\tilde{\xxx}_n 
			\ge
			    n^{\theta_{\epsilon}} 
			\ge 
			    n^{1-\frac{2e}{2+\epsilon}}
			,
		\end{equation}
		with the following property: there exists a separable ultrametric space $(\mathcal{Z},d_{\mathcal{Z}})$ and a bi-Lipschitz map $\phi_1:(\tilde{\xxx}_n,d_n) \hookrightarrow (\mathcal{Z},d_{\mathcal{Z}})$ such that for every $x,\tilde{x}\in \tilde{\xxx}_n$ it holds that
		\begin{equation}
			\label{PROOF_eq__lemma_Dvoretzy_Type___TaoNaorDvoretzky}
			d_n(x,\tilde{x})
			\le\,
			d_{\mathcal{Z}}(\phi_1(x),\phi_1(\tilde{x}))
			\le\,
			(2+\epsilon)\,
			d_n(x,\tilde{x})
			.
		\end{equation}
		Since $(\mathcal{Z},d_{\mathcal{Z}})$ is a separable ultrametric space then \citep{Vestfrid_SeparableUltrametricSpacesEmbedinHilbertSpace_1979} (or \cite{shkarin2004isometric}) implies that there exists an isometric embedding $\varphi_2:(\mathcal{Z},d_{\mathcal{Z}})\hookrightarrow (\ell_2^{\infty},\ell_2)$, where the norm $\|\cdot \|_{\ell^2}$ is defined for any $x\in \rr^{\nn}$ by $\|x\|_{\ell^2}^2\eqdef \sum_{k=1}^{\infty}\, x_k^2$, $\ell_2^{\infty}\eqdef \{x\in \rr^{\nn}:\, \|x\|_{\ell^2}<\infty\}$, and thus $\big(\ell_2^{\infty},\ell_2\big)$ the infinite dimensional separable Hilbert space.  Define $\varphi_3\eqdef\varphi_2\circ \varphi_1$.  It follows from~\eqref{PROOF_eq__lemma_Dvoretzy_Type___TaoNaorDvoretzky} that: for all $x,\tilde{x}\in \tilde{\xxx}_n$
		\begin{equation}
			\label{PROOF_eq__lemma_Dvoretzy_Type___TaoNaorDvoretzky_1}
			d_n(x,\tilde{x})
			\le\,
			\left\|
			\phi_3(x)-\phi_3(\tilde{x})
			\right\|_{\ell_2}
			\le\,
			(2+\epsilon)
			d_n(x,\tilde{x})
			.
		\end{equation}
		By the Johnson-Lindenstrauss Flattening Theorem \citep[Theorem 15.2.1]{Matouvsek_2002_LecturesDiscreteGeo} there exist an $s>0$ and a map $\phi_4:(\ell_2^{\infty},\ell_2)\rightarrow (\ell_2^{n_{\epsilon}},\ell_2)$, where $n_{\epsilon}\in \mathcal{O}(\epsilon^{-2}\log_2(\#\tilde{\xxx}_n))
		= 
		\mathcal{O}(
		\epsilon^{-2}\theta_{\epsilon}\log_2(n)))
		$ (thus $n_{\epsilon}$ is at-most $\mathcal{O}(\epsilon^{-2}\log_2(n) $ since $\theta_{\epsilon}<1$) satisfying
		\begin{equation}
			\label{PROOF_eq__lemma_Dvoretzy_Type___JL_Lemma}
			s
			\left\|
			z-\tilde{z}
			\right\|_{\ell_2}
			\leq 
			\left\|
			\phi_4(z)-\phi_4(\tilde{z})
			\right\|_{\ell_2}
			\leq 
			s
			(1+\epsilon)
			\left\|
			z-\tilde{z}
			\right\|_{\ell_2}
			,
		\end{equation}
		for all $z,\tilde{z}\in \phi_3(\tilde{\xxx}_n)$.  Set $\phi\eqdef\phi_4\circ \phi_3:(\tilde{\xxx}_n,d_n) \hookrightarrow (\rr^{n_{\epsilon}},\ell_2)$.  Together~\eqref{PROOF_eq__lemma_Dvoretzy_Type___TaoNaorDvoretzky_1} and~\eqref{PROOF_eq__lemma_Dvoretzy_Type___JL_Lemma} imply that: for every $x,\tilde{x}\in \tilde{\xxx}_n$ it holds that
		\begin{equation}
			\label{PROOF_eq__lemma_Dvoretzy_Type___JL_Lemma_1}
			s
			d_n(x,\tilde{x})
			\leq 
			\left\|
			\phi_5(x)-\phi_5(\tilde{x})
			\right\|_{\ell_2}
			\leq 
			s
			(1+\epsilon)(2+\epsilon)
			d_n(x,\tilde{x})
			.
		\end{equation}
		
		\hfill\\
		\textbf{Step 2 - Constructing The Embedding into $(\mathcal{MG}(\rr)_2,\mw_2)$}
		\hfill\\
		Since $\tilde{\xxx}_n$ is finite, then $(\tilde{\xxx}_n,d_n)$ is compact.  Since $\phi$ is an isometry then it is continuous and therefore by \citep[Theorem 26.5]{munkres2014topology} $\phi(\tilde{\xxx}_n)$ is a compact subset of $(\rr^{n_{\epsilon}},\ell_2)$.  By the Heine-Borel theorem \citep[Theorem 27.3]{munkres2014topology}, there exists an $r>0$ such that $\phi(\tilde{\xxx}_n)\subseteq \overline{\operatorname{Ball}_{\rr^{n_{\epsilon}},\ell_2}(0,r)}$; where, $\overline{\operatorname{Ball}_{\rr^{n_{\epsilon}},\ell_2}(0,r)}\eqdef \{u\in \rr^{n_{\epsilon}}:\,\|x\|_{\ell_2}\leq r\}$.  
		Since $\phi(\tilde{\xxx}_n) \subseteq \overline{\operatorname{Ball}_{(\rr^{n_{\epsilon}},\ell_2)}(0,r)}$ then \citep[Proposition 7.4]{BenoitBertrand_2010_AnnNormSupPisa} applies.  More specifically, as formulated in \citep[Example 5.5]{BenoitBertrand_2012_JTopAnal})
		there exists some $b\in \rr^{n_{\epsilon}}$ (which can be computed explicitly using Algorithm~\ref{Algorithm_RandomPartition}) such that the map
		\[
		\iota: 
		(\overline{\operatorname{Ball}_{(\rr^{n_{\epsilon}},\ell_2)}(0,r)},\ell_2)
		\ni 
		x
		\mapsto 
		\frac1{n_{\epsilon}} \sum_{k=1}^{n_{\epsilon}}
		\,
		\delta_{
		    \sqrt{n_{\epsilon}}
		(x +b)_k} 
		\in (\mathcal{P}_2(\mathbb{R}),\mathcal{W}_2)
		,
		\]
		is an isometry. Since the composition of isometries is itself an isometry then, the map $\varphi\eqdef \iota\circ \phi:(\tilde{\xxx}_n,d_n)\hookrightarrow (\mathcal{P}_2(\mathbb{R}),\mathcal{W}_2)$ is an isometry and is of the required form.  
		Lastly, since $\iota\circ \phi$ is has range in the set of ``degenerate Gaussian mixtures'' $\mathcal{Z}\eqdef \{\pp\in \mathcal{GM}_2(\rr):\,\exists N\in \nn_+\,\exists w\in \Delta_N\,\exists \mu_1,\dots,\mu_N\in \rr\,s.t.\, \pp=\sum_{n=1}^N\,w_n \delta_{\mu_n}\}$ then, we obtain the conclusion by applying the comparison result in \citep[Proposition 6]{WassersteinGaussianMixtures} which states that the ``identity map'' on $(\mathcal{W}_2(\rr),\mathcal{W}_2)\ni \pp\mapsto \pp \in (\mathcal{GM}_2(\rr),\mw_2)$ is an isometry between on the collection of finitely supported probability measures (i.e.\ ``degenerate Gaussian mixtures'').  
		Applying Lemma~\ref{lemma_embedding_neuralification} to the map $\phi\circ [x\mapsto (d(x,\tilde{x})_{\tilde{x}\in \xxx_n})]$ yields the desired tranformer $T$ satisfying
		\[
		s
		d_{n}(x,\tilde{x})
		\leq
		\mw_2(T(x),T(\tilde{x}))
		\leq 
		s
		(2+\epsilon)^2
		d_n(x,\tilde{x})
		,
		\]
		for all $x,\tilde{x}\in \tilde{\xxx}_n$.
		
		\textbf{Comment:} \textit{From here, it now only remains to estimate the probability that a uniformly and independently chosen pair of random points in $\xxx_n$ lie in $\tilde{\xxx}_n$.}
		
		\hfill\\
		\textbf{Step 3 - Estimating Probability of Independently and Uniformly Sampling two Points in the Dvoretzky-type Subspace}
		\hfill\\
		Let $\pp$ be the uniform probability distribution on $\xxx_n$ and let $X,\tilde{X}\sim \pp$ be independent $\xxx_n$-valued random elements.  By independence and the fact that $X$ and $\tilde{X}$ have the same law we can compute the following 
		\begin{equation}
			\label{PROOF_eq__lemma_Dvoretzy_Type___P_inPACestimate}
			\pp\left(
			X \in \tilde{\xxx}_n
			\mbox{ and } 
			\tilde{X} \in \tilde{\xxx}_n
			\right)
			= 
			\pp\left(
			X \in \tilde{\xxx}_n
			\right)
			\pp\left(
			\tilde{X} \in \tilde{\xxx}_n
			\right)
			=
			\pp\left(
			X \in \tilde{\xxx}_n
			\right)^2
			.
		\end{equation}
		Together, the lower-bound in~\eqref{PROOF_eq__lemma_Dvoretzy_Type___cardinality_bound_DvoretzkySubset}, the definition of $\pp$, and~\eqref{PROOF_eq__lemma_Dvoretzy_Type___P_inPACestimate} imply that
		\begin{equation}
			\label{PROOF_eq__lemma_Dvoretzy_Type___P_inPACestimate____completed}
			\begin{aligned}
				\pp\left(
				X \in \tilde{\xxx}_n
				\mbox{ and } 
				\tilde{X} \in \tilde{\xxx}_n
				\right)
				&=
				\pp\left(
				X \in \tilde{\xxx}_n
				\right)^2
				\\
				&=   \left(\frac{\#\tilde{\xxx}_n}{\#\xxx_n}\right)^2
				\\
				&\ge  \left(\frac{n^{\theta_{\epsilon}}}{n}\right)^2
				\\
				&\ge 
				\left(\frac{n^{1-(2 e/2+\epsilon)}}{n}\right)^2
				\\
				&=  n^{-4 e/(2+\epsilon)}.
			\end{aligned}
		\end{equation}
		Moreover, by \citep[Theorem 1.2]{NaorTao_2012_Dvoretzky_ISJM} the precise value of $\theta_{\epsilon}$ is given by the unique solution $\theta\in (0,1)$ to $\frac{2}{2+\epsilon} = (1-\theta)\theta^{\theta/(1-\theta)}$; nb,
		$\theta_{\epsilon} \in [1-2 e/(2+\epsilon),1)
		.
		$
		Labeling $D\eqdef 2+\epsilon$, $\theta_D\eqdef \theta_{\epsilon}$, observing that $T$ is defined on all of $\xxx$, and noting that by construction $d_T=k$ yields the theorem's statement and concludes our proof.
		
	\textbf{Step 4 - Inverting the relation between $D$ and $\theta_D$:} \hfill\\
	Fix $\delta\in (0,1)$ such that $n^{-2+2\theta_D} = \delta$.  Solving for $\theta_D$ yields
	\begin{equation}
	\label{eq_PROOOF___MAINTHEOREM_PROBABILISTIC___Sovling_forThetaD}
	    \theta_D 
	        = 
	    1 + \log_n(\sqrt{\delta})
	    .
	\end{equation}
	Since $\theta_D$ must belong to $(0,1)$ then~\eqref{eq_PROOOF___MAINTHEOREM_PROBABILISTIC___Sovling_forThetaD} only gives a valid $\theta_D$ if $\delta \in \big(1/e^2,1\big)$; thus, we now impose that constraint.  
	Since $\theta_D$ solves $\frac{2}{D} = (1-\theta)\, \theta^{\theta/(1-\theta)}$ where $\theta \in [1-\frac{2e}{D},1)$ then~\eqref{eq_PROOOF___MAINTHEOREM_PROBABILISTIC___Sovling_forThetaD} implies that
	\begin{equation}
	\label{eq_PROOOF___MAINTHEOREM_PROBABILISTIC___Sovling_forThetaD__end_thresholdingwithmax}
	    D 
	        =
	    - 2\frac{
    	    \big(
    	        1 + \log_n(\sqrt{\delta})
    	    \big)^{
    	        \frac{
    	            1 + \log_n(\sqrt{\delta})
    	        }{
    	            \log_n(\sqrt{\delta})
    	        }
    	    }
    	 }{ 
    	    \log_n(\sqrt{\delta})
    	}
	.
	\end{equation}
	In order for $D$ in~\eqref{eq_PROOOF___MAINTHEOREM_PROBABILISTIC___Sovling_forThetaD__end_thresholdingwithmax} to be a valid distortion (compatible with our argument) we need that $D>2$.  Therefore, we require that $\delta>\delta_n$; where $\tilde{\delta}_n\in [\frac{1}{e^2},1)$ is the unique solution to 
	\[
	        \log_n(\sqrt{\delta}) 
	    = 
    	    \big(
    	        1 + \log_n(\sqrt{\delta})
    	    \big)^{
    	        \frac{
    	            1 + \log_n(\sqrt{\delta})
    	        }{
    	            \log_n(\sqrt{\delta})
    	        }
    	    }
    .
	\]
	Set $\delta_n \eqdef \max\{\frac1{e^{2}},\tilde{\delta}_n\}$.  
	Restricting $\delta \in (\delta_n,1)$ yields a $\delta$ for which $\theta_D$ and $D$ both satisfy the required assumptions of our argument.  This complete the proof.
	\end{proof}
	
	\subsection{Proof of Distortionless Multivariate Embeddings}
	\label{a_Proofs__ss_Embeddings_Distortionless}
	\begin{proof}[{Proof of Theorem~\ref{theorem:high_dimension_asymptotic_no_distortion}}]
	    Fix a positive integer $n$ and let $\xxx_n$ be a $n$-point subset of a metric space $(\xxx,d)$ which we view as an $n$-point metric space $(\xxx_n,d)$. 
	    Let $\iota:(\xxx_n,d)\rightarrow (\mathbb{R}^n,\|\cdot\|_{\infty})$ be the Fr\'{e}chet embedding defined by $\iota(x)\eqdef  (d(x,x_i))_{i=1}^n$ and denote $\tilde{\xxx}_n\eqdef \iota(\xxx_n)$.  Clearly, $(\xxx_n,d)$ and $(\tilde{\xxx}_n,\|\cdot\|_{\infty})$ are isometric with isometry given by $\iota$.  It is therefore sufficient to bi-H\"{o}lder embed $(\title{\xxx}_n,\|\cdot\|_{\infty})$ into $(\mathcal{P}_2(\mathbb{R}^3),\mathcal{W}_2)$; we do this now.  
	    In the proof of \citep[Theorem 1]{NaorNeimanAndoniSnowflake2018AnnalsSENSQ} from \citep[Equation (12) to Equation(22)]{NaorNeimanAndoniSnowflake2018AnnalsSENSQ}, the authors define a map $\tilde{\phi}:\tilde{\xxx}_n \rightarrow (\mathcal{P}_2(\mathbb{R}^3),\mathcal{W}_2)$ with the property that: there is some positive integer $N$ such that for every $x\in \tilde{\xxx}_n$
	    \begin{equation}
	    \label{eq:Naor_representation}
	        \tilde{\phi}(x) 
	       = 
	        \frac1{N}\, \sum_{n=1}^N\, \delta_{u_n(x)}
	    \end{equation}
	    for some set of points $\{u_n(x):\,n=1,\dots,N,\,x\in \tilde{\xxx}_n\}$ in $\mathbb{R}^3$.  Moreover, a careful reading of the construction show that the positive integer $N$ is such bounded above by\footnote{As a brief reference, we note that, in the notation of \cite{NaorNeimanAndoniSnowflake2018AnnalsSENSQ}, $N\le \#\mathcal{C}$, $\mathcal{C}\subseteq \cup_{i,j=1,\dots,n;\,i<j}\, \mathcal{B}_{i,j}$ and $\#\mathcal{B}_{i,j} = 5K +2$; for some hyperparameter $K\in \mathbb{N}_+$ independent of the construction thus far.  }
	    \begin{equation}
	    \label{eq:upper_bound_N__zerodistortionversion}
	        N \le n(n-1)/2\, (5K +2)
	    \end{equation}
	    for some hyperparameter $K\in \mathbb{N}_+$ to be decided upon shortly.  The quantitative version of the result in question, namely \citep[Lemma 15]{NaorNeimanAndoniSnowflake2018AnnalsSENSQ}, shows that for every distortion $D>1$, taking $\frac{5n^4}{2(D-1)} \, \operatorname{aspect}(\tilde{\xxx}_n,d)^2 \le K$ implies that: there is some scale $s>0$ such that for every $x,\tilde{x}\in \tilde{\xxx}_n$ it holds that
	    \begin{equation}
	        \label{eq:distortion_control}
    	        s\,\|x-\tilde{x}\|_{\infty}^{1/2}
	        \le 
    	        \mathcal{W}_2\big(
    	            \tilde{\phi}(x)
    	        ,
    	           \tilde{\phi}(\tilde{x})
    	        \big)
    	    \le 
    	        D\,s\,\|x-\tilde{x}\|_{\infty}^{1/2}
    	   .
	    \end{equation}
	Since $\iota$ is an isometry and every pointmass on $\mathbb{R}^3$ is a Gaussian measure with mean at that point and zero variance then the map $\phi\eqdef \tilde{\phi}\circ \iota^{-1}$ sends $(\xxx_n,d)$ to $(\mathcal{GM}_2(\mathbb{R}^3),\mathcal{MW}_2)$ and it has representation
	\[
	\phi(x) =  \sum_{n=1}^N\, \frac1{N}\, N_3(u_n(x),\boldsymbol{0}),
	\]
	where $\boldsymbol{0}$ is the $3\times 3$-dimensional zero matrix.  Moreover, since $\iota$ is an isometry and by \citep[Proposition 6]{WassersteinGaussianMixtures} the inclusion of the set of finitely supported probability measures in $(\mathcal{P}_2(\mathbb{R}^3),\mathcal{W}_2)$ into $(\mathcal{GM}_2(\mathbb{R}^3),\mathcal{MW}_2)$ is an isometry then~\eqref{eq:distortion_control} implies that: there is a scale $s>0$ such that for every $x,\tilde{x}\in \xxx_n$ it holds that
	\begin{equation}
	        \label{eq:distortion_control__translated}
    	        s\,d(x,\tilde{x})^{1/2}
	        \le 
    	        \mathcal{MW}_2\big(
    	            \phi(x)
    	        ,
    	           \phi(\tilde{x})
    	        \big)
    	    \le 
    	        D\,s\,d(x,\tilde{x})^{1/2}
    	   .
	    \end{equation}
	That is, $\phi$ is a $\frac1{2}$-bi-H\"{o}lder embedding of $(\xxx_n,d)$ into $(\mathcal{GM}_2(\mathbb{R}^3),\mathcal{MW}_2)$.  Applying the Key Lemma, Lemma~\ref{lemma_embedding_neuralification}, and to the map $\phi$ yields the conclusion; upon noting that the probabilistic transformer $T$ is defined on all of $\xxx$.
	\end{proof}
	
	\subsection{Proof of Embedding Lemmata}
	\label{a_Proofs__ss_Embedding_Lemmas}
	
	\begin{proof}[Proof of Lemma~\ref{lem_embedding_doubling_metric_spaces}]
		Since $(\xxx,d)$ is a doubling metric space then, by \citep[Theorem 1.2]{naor2012assouad} there exist constant $c,C>0$ such that for any $0<\epsilon<\frac{1}{2}$ there exists an injective function $\phi:(\xxx,d)\hookrightarrow (\rr^{D},\ell_2)$ satisfying
		\begin{equation}
			\label{PROOF_lem_embedding_doubling_metric_spaces_embedding__Quantitative_Assoud_quote}
			d_{\xxx}^{1-\epsilon}(x,\tilde{x})
			\leq
			\|
			\phi(x)
			-
			\phi(\tilde{x})
			\|_{\ell_2}
			\leq 
			\left\lceil
			\frac{C\log(C_d)^2}{(1-(1-\epsilon))^2}
			\right\rceil
			d_{\xxx}^{1-\epsilon}(x,\tilde{x})
			;
		\end{equation}
		where $D\leq c \log(C_d)$.  Since $\xxx_n\subseteq \xxx$ is finite, then $\xxx_n$ is compact.  By~\eqref{PROOF_lem_embedding_doubling_metric_spaces_embedding__Quantitative_Assoud_quote}, $\phi$ is $(\frac{1}{2},1)\ni (1-\epsilon)$-H\"{o}lder continuous and therefore, by \citep[Theorem 26.5]{munkres2014topology} $\phi(\xxx_n)$ must be a compact subset of $(\rr^D,\ell_2)$.  As before, by appealing to the Heine-Borel theorem (\citep[Theorem 27.3]{munkres2014topology}) we conclude that there must exist some $r>0$ such that $\phi(\xxx_n)\subseteq \overline{\operatorname{Ball}_{\rr^D,\ell_2}(0,r)}$; where $\overline{\operatorname{Ball}_{\rr^D,\ell_2}(0,r)}\eqdef \{u\in \rr^D:\,\|x\|_{\ell_2}\leq r\}$.  Thus, \citep[Proposition 7.4]{BenoitBertrand_2010_AnnNormSupPisa} implies that there exists some $b\in \rr^{D}$ \textit{(which can be constructed using Algorithm~\ref{Algorithm_RandomPartition})} such that the map
		\[
		\iota: 
		(\overline{\operatorname{Ball}_{(\rr^{D},\ell_2)}(0,r)},\ell_2)
		\ni 
		x
		\mapsto 
		\frac1{D} \sum_{k=1}^{D}
		\,
		\delta_{
		\sqrt{D}
		(x +b)_k
		} 
		\in (\mathcal{P}_2(\mathbb{R}),\mathcal{W}_2)
		,
		\]
		is an isometry.  
		As in the proof of Lemma~\ref{lem_embedding_manifold_sampled}, since $\iota\circ \phi$ is has range in the set of ``degenerate Gaussian mixtures''; i.e. belonging to $\mathcal{Z}\eqdef \{\pp\in \mathcal{GM}_2(\rr^d):\,\exists N\in \nn_+\,\exists w\in \Delta_N\,\exists \mu_1,\dots,\mu_N\in \rr^d\,s.t.\, \pp=\sum_{n=1}^N\,w_n \delta_{\mu_n}\}$ then, we obtain the conclusion by applying the comparison result in \citep[Proposition 6]{WassersteinGaussianMixtures} which states that the identity map on $(\mathcal{GM}_2(\rr^d),\mathcal{W}_2)\ni \pp\mapsto \pp \in (\mathcal{GM}_2(\rr^d),\mw_2)$ is an isometry between on the collection of degenerate mixtures of in $\mathcal{Z}$.
		Setting $\phi\eqdef \iota\circ \phi$ and $\alpha\eqdef 1-\epsilon$ concludes the proof.  
	\end{proof}
	
	\subsection{Proof of Main Results}
	\label{a_Proofs__ss_Main_Results}
	
	The proof of Theorem~\ref{MAINTHEOREM_DETERMINISTIC} relies on the following lemma relating the rate at which the \textit{volume of a ball} doubles as a function of its radius to $\xxx_n$'s metric capacity.  To formulate the required lemma, we must first revisit the notion of a \textit{doubling measure} \citep[Page 3 and Chapter 13]{heinonen2001lectures}.   
	
	Let us begin by recalling that, given any $x\in \xxx_n$ and any $r>0$, the metric ball about $x$ of radius $r$ is defined by $B(x,r)\eqdef \{u\in \xxx_n:\,d_{\xxx}(u,z)<r\}$.  
	Briefly, we say that a (finite Borel) measure $\mu$ on $(\xxx_n,d_n)$ is a \textit{doubling measure} with doubling constant $c_{\mu}$ if for every $x\in \xxx_n$ and every $r>0$ it holds that
	\begin{equation}
		\label{eq_ball__measure}
		\mu(B(x,2r)) \leq c_{\mu}\mu(B(x,r))
		.
	\end{equation}
	In \cite{VolbergKonyagin1987Conversedoublingmeasures} it was shown that the existence of a doubling measure is equivalent to $d_n$ being \textit{doubling as a metric space}; this means that there is a (metric) doubling constant $c_{d_n}>0$ such that: for each $r>0$ every metric ball in $\xxx_n$ of radius $2r$ can be covered by at most $c_{d_n}$ metric balls of radius $r$.  In \citep[Proposition 1.7]{Brue2021Extension}, it is shown that the existence of a bound on $\mathrm{cap}(\xxx_n)$ is equivalent to the existence and a bound for $c_{d_n}$. 
	Thus, a bound on the constant of a doubling measure $c_{\mu}$ should, in principle, imply a bound on the $\mathrm{cap}(\xxx_n)$.  The next explicitly spells out the relationship between these ``metric quantities''.  
	\begin{lemma}[Bounds of Metric Capacity via (Metric and Measure) Doubling Constants]
		\label{lem_doubling_to_capacity_control_lemma}
		Let $\mu$ be a doubling measure on $(\xxx_n,d_n)$ with doubling constant $c_{\mu}$, as in~\eqref{eq_ball__measure}.  Then, the metric capacity of $(\xxx_n,d_n)$ is bounded above by
		\[
		    \log(c_{d_n})
		\leq 
		    \log(\mathrm{cap}(\xxx_n) )
		\leq 
		    12 \log(c_{\mu})
		.
		\]
	\end{lemma}
	\begin{proof}[Proof of Lemma~\ref{lem_doubling_to_capacity_control_lemma}]
		Suppose that $N$ is a positive integer for which there does not exist points $x_1,\dots,x_N\in B(x,r)$ satisfying
		\[
		\operatorname{Ball}_g(x,r) 
		\subseteq 
		\bigcup_{i=1}^N \operatorname{Ball}_g(x_i,r/2)
		,
		\]
		then, $N$ is an upper-bound for $c_{d_n}$.  Moreover, there must exist $N$ points $x_1,\dots,x_n$ in $B(x,r)$ for which 
		\[
		0
		<
		\min_{i,j=1,\dots,N;\, i\neq j}\,
		\min_{u \in \xxx_n,\, d_n(u,x_i)\leq 2^{-2}r}
		\min_{\tilde{u} \in \xxx_n,\, d_n(\tilde{u},x_j)\leq 2^{-2}r}
		\,
		d_n(u,\tilde{u})
		.
		\]
		Pick $N^{\star} \in \underset{i=1,\dots,N}{\operatorname{argmin}}\, \mu(B(x_i,r))$.  Thus, we have that
		\[
		\begin{aligned}
			\mu(B(x_{N^{\star}},4r)) 
			& \geq \sum_{i=1}^{N} \, \mu(B(x_i,2^{-2}r)) 
			\\
			& 
			\geq n \min_{i=1,\dots,N}\,\mu(B(x_i,r)) 
			& = n \mu(B(x_{N^{\star}},r))
			.
		\end{aligned}
		\]
		Therefore, $c_d\leq c_{\mu}^4$.  Now, by \citep[Proposition 1.7 (i)]{Brue2021Extension}, we have that 
		\begin{equation}
			\label{lem_doubling_to_capacity_control_lemma__cd_to_capX}
			\mathrm{cap}(\xxx_n)
			\leq 
			C_d^3
			.
		\end{equation}
		Combining the estimate~\eqref{lem_doubling_to_capacity_control_lemma__cd_to_capX} with the estimate $c_d\leq c_{\mu}^4$ yields the desired upper-bound on $\mathrm{cap}(\xxx_n)$.  The lower-bound on $\mathrm{cap}(\xxx_n)$ is given in \citep[Proposition 1.7 (i)]{Brue2021Extension}.
	\end{proof}
	\begin{proof}[Proof of Theorem~\ref{MAINTHEOREM_DETERMINISTIC}]
	Fix an $n\ge 2$-point metric subspace $(\xxx_n,d_n)$ of $(\xxx,d_{\xxx})$.  
	Restricting to $(\xxx_n,d_n)$, the result for $\xxx=\xxx_n$ follows from 
		Lemma~\ref{lemma_embedding_neuralification} applied to Lemma~\ref{lem_embedding_doubling_metric_spaces} and then using Lemma~\ref{lem_doubling_to_capacity_control_lemma} to upper-bound $\log_2(c_{d_n})$ by $12\log_2(\mathrm{cap}(\xxx_n,d_{n}))$.
	The general statement, follows by simply noting that $T$ constructed this way is defined on all of $\xxx$.  
	\end{proof}
	\begin{remark}
		\label{remark_sharper__MAINTHEOREM_DETERMINISTIC}
		The proof of Theorem~\ref{MAINTHEOREM_DETERMINISTIC} implies that the result holds with $\mathrm{cap}(\xxx_n,d_{n})$ replaced by $\xxx_n$'s doubling constant.  Nevertheless, one can still derive the result with $C_d^{3}$ in place of $\mathrm{cap}(\xxx_n)$ as a consequence of the present formulation of Theorem~\ref{MAINTHEOREM_DETERMINISTIC} and of \citep[Proposition 1.7 (ii)]{Brue2021Extension}.  Thus, we only make the remark to state that the result can be sharpened if one prefers to consider doubling constants rather than metric capacity.  
	\end{remark}
	\subsection{Proofs of Secondary Results}
	\label{a_Proofs__aa__secondary}
	\subsubsection{Proofs of Deterministic Embeddings}
	\label{a_Proofs__aa__secondary___aaa_deterministic_phasetransition_stuff}
	Let us consider the case where $(\xxx_n,d_n)$ is a tree.  
	\begin{lemma}[{Bi-Lipschitz Embedding of (Finite) Combinatorial Trees into $(\mathcal{MG}_2(\rr),\mw_2)$}]
		\label{lem_embedding_trees}
		Let $G=(V,E)$ be an $n$-vertex tree and let $(V,d_G)$ be its associated metric space (as in Example~\ref{s_Intro_Background__ssCombinatorialGraphs}). There is an absolute constant $C>0$ such that, for every $M\in \nn_+$ there exists a map of the form
		\[
		\varphi(x)
		\eqdef 
		\frac1{M}
		\sum_{k=1}^M
		N_d(\phi(x)_k,0)
		,
		\]
		such that the following holds: for every $x,\tilde{x}\in \xxx_n$ we have that
		\[
		d_{G}(x,\tilde{x})
		\leq
		\mathcal{W}_2(\varphi(x),\varphi(\tilde{x}))
		\leq 
		\mw_2(
		\varphi(x)
		,
		\varphi(\tilde{x})
		)
		\leq 
		C n^{1/(d-1)}
		d_{G}(x,\tilde{x})
		,
		\]
		where $\phi:V\rightarrow \rr^M$ and $C>0$ is an absolute constant independent of $n$, $d$, and of $(\xxx_n,d_n)$.
	\end{lemma}
	\begin{proof}[Proof of Lemma~\ref{lem_embedding_trees}]
		Fix $\epsilon>0$.  Since $d_n$ is a tree metric on $\xxx_n$ then \citep[Theorem 1]{Gupta_Trees_Quantitiativefinitedimembeddings_ACM}
		implies there is an absolute constant $C>0$ (independent of $n$ and of $(\xxx_n,d_n)$) and bi-Lipschitz embedding $\phi:V\rightarrow \rr^d$ satisfying: for every $x,\tilde{x}\in V$
		\begin{equation}
			\label{PROOF_lem_embedding_trees_matousek}
			    d_n(x,\tilde{x})
			\le 
			    \|\phi(x)-\phi(\tilde{x})\|
			\le
    			C n^{1/(d-1)}
    			d_n(x,\tilde{x})
			.
		\end{equation}
		
		Since $\xxx_n$ is finite, then $(\xxx_n,d_n)$ is compact.  Since $\phi$ is a bi-Lipschitz map then, it is continuous and therefore by \citep[Theorem 26.5]{munkres2014topology}, $\phi(\xxx_n)$ is a compact subset of $(\rr^5,\ell_2)$.  By the Heine-Borel theorem (\citep[Theorem 27.3]{munkres2014topology}), there exists an $r>0$ such that $\phi(\xxx_n)\subseteq \overline{\operatorname{Ball}_{\rr^5,\ell_2}(0,r)}$; where $\overline{\operatorname{Ball}_{\rr^5,\ell_2}(0,r)}\eqdef \{u\in \rr^5:\,\|x\|_{\ell_2}\leq r\}$.  
		Since $\phi(\xxx_n) \subseteq \overline{\operatorname{Ball}_{(\rr^{5},\ell_2)}(0,r)}$ then \citep[Proposition 7.4]{BenoitBertrand_2010_AnnNormSupPisa} applies.  More specifically, the proof of \citep[Proposition 7.4]{BenoitBertrand_2010_AnnNormSupPisa} (see the comment in \citep[Example 5.5]{BenoitBertrand_2012_JTopAnal}) guarantee that there exists some $b\in \rr^{d}$ \textit{(which can be constructed using Algorithm~\ref{Algorithm_RandomPartition} in Section~\ref{a_construction_of_bias_term})} such that the map
		\[
		\iota: 
		(\overline{\operatorname{Ball}_{(\rr^{5},\ell_2)}(0,r)},\ell_2)
		\ni 
		x
		\mapsto 
		\frac1{d} \sum_{k=1}^{d}
		\,
		\delta_{
		\sqrt{d}
		(x +b)_k
		} 
		\in (\mathcal{P}_2(\mathbb{R}),\mathcal{W}_2)
		,
		\]
		is an isometry. Since the composition of isometries is itself an isometry, then the map $\varphi\eqdef \iota\circ \phi:(\xxx_n,d_n)\hookrightarrow (\mathcal{P}_2(\mathbb{R}),\mathcal{W}_2)$ is an isometry and is of the required form.  
		Lastly, since $\iota\circ \phi$ is has range in the set of ``degenerate Gaussian mixtures'' $\mathcal{Z}\eqdef \{\pp\in \mathcal{GM}_2(\rr^d):\,\exists N\in \nn_+\,\exists w\in \Delta_N\,\exists \mu_1,\dots,\mu_N\in \rr^d\,s.t.\, \pp=\sum_{n=1}^N\,w_n \delta_{\mu_n}\}$ then, we obtain the conclusion by applying the comparison result in \citep[Proposition 6]{WassersteinGaussianMixtures} which states that the ``identity map'' on $(\mathcal{W}_2(\rr),\mathcal{W}_2)\ni \pp\mapsto \pp \in (\mathcal{GM}_2(\rr),\mw_2)$ is an isometry on the collection of finitely supported probability measures (i.e. ``degenerate Gaussian mixtures'').
	\end{proof}
	\begin{proof}[Proof of Proposition~\ref{prop_combinatorial_tree_embedding}]
	Since $(\xxx_n,d_n)$ is a combinatorial (metric) graph then $\operatorname{aspect}(\xxx_n,d_n)\leq \operatorname{diam}(\xxx_n,d_n)$.  
	The result now follows directly from Lemma~\ref{lemma_embedding_neuralification} applied to Lemma~\ref{lem_embedding_trees}.  
	\end{proof}
	Our next case concerns the structure of a low-distortion embedding of a graph whose points are drawn from a Riemannian manifold.  In this case, we see that the number of pointmasses required scales quadratically in the dimension of the latent Riemannian manifold.  
	
	\begin{lemma}[{Bi-Lipschitz Embeddings of Riemannian Datasets - For Compact Manifolds}]
	\label{lem_embedding_manifold_sampled}
	\hfill\\
		Let $n,d\in \nn_+$.  
		Let $(M,g)$ be a $d$-dimensional Riemannian manifold with geodesic distance $d_g$ and let $K\subseteq M$ be compact.  
		There exists a constant $C_K$ (depending only on $K$) such that for any $\xxx_n\subseteq M$ if $d_n\eqdef d_g\vert _{\xxx_n}$, then exists a map $\varphi:(\xxx_n,d_n) \rightarrow (\mathcal{GM}_2(\rr),\mw_2)$ of the form
		\[
		\varphi(x) = \frac1{2d} \sum_{k=1}^{2d}\, 
		N_1(\phi(x)_k,0)
		;
		\]
		satisfying: for every $x,\tilde{x}\in \xxx$ the following holds
		\[
		d_{\xxx}(x,\tilde{x})
		\leq
		\mathcal{W}_2(\varphi(x),\varphi(\tilde{x}))
		\leq 
		\mw_2(\varphi(x),\varphi(\tilde{x}))
		\leq 
		C\, C_K 
		d_{\xxx}(x,\tilde{x})
		;
		\]
		where $\phi:(\xxx_n,d_n)\rightarrow (\rr^{2d},\ell_2)$ is a class-$C^1$-(Riemannian) isometric embedding.  
	\end{lemma}
	\begin{proof}[Proof of Lemma~\ref{lem_embedding_manifold_sampled}]
		Since $(M,g)$ is a Riemannian manifold then, \citep[Theorem 2]{Nash_1954_c1_Embeddings} implies that there exists a class-$C^1$ (Riemannian isometric embedding) $\phi:(M,d_g)\hookrightarrow (\rr^{2d},\ell_2)$.  Since $\phi$ is a class $C^1$ Riemannian isometric embedding, it is a diffeomorphism onto its image and therefore, $\phi$ has a class $C^1$ inverse on $\phi(M)$ which we denote by $\phi^{-1}$.  Since $K$ is compact then and $\phi$ is continuous then $\phi(K)$ is compact and therefore the following is finite
		\[
		C_K
		\eqdef
		\max\{
		\max_{x \in K}\, \|\nabla \phi(x)\|
		,
		\max_{x \in \phi(K)}\, \|\nabla \phi^{-1}(x)\|
		\}
		.
		\]
		Therefore, by the Rademacher-Stepanov Theorem \citep[Theorem 3.1.6]{Federer_GeometricMeasureTheory_1978} we have that: for every $x,\tilde{x}\in K$
		\begin{equation}
			\label{eq_proof_general_manifold__biLipschitz_guarantee}
			d_g(x,\tilde{x})
			\leq 
			\|\phi(x)-\phi(\tilde{x})\|
			\leq 
			C_K
			d_g(x,\tilde{x})
			.
		\end{equation}
		
		Since $\phi$ is class-$C^1$, it is continuous and since $\xxx_n$ is a compact subset of $M$ (since it is finite) then by \citep[Theorem 26.5]{munkres2014topology} $\phi(\xxx_n)$ is a non-empty compact subset of $(\rr^{2d},\ell_2)$.  By the Heine-Borel Theorem \citep[Theorem 27.3]{munkres2014topology}, there exists some $r>0$ such that $\phi(\xxx_n)\subseteq \overline{\operatorname{Ball}_{(\rr^{2d},\ell_2)}(0,r)}$ where $\overline{\operatorname{Ball}_{(\rr^{2d},\ell_2)}(0,r)}\eqdef 
		\{z \in \rr^{2d}:\, \|z\|_{\ell_2}\leq r\}.$  
		
		Since $\phi(\xxx_n) \subseteq \overline{\operatorname{Ball}_{(\rr^{2d},\ell_2)}(0,r)}$ then \citep[Proposition 7.4]{BenoitBertrand_2010_AnnNormSupPisa} applies.  More specifically, the proof of \citep[Proposition 7.4]{BenoitBertrand_2010_AnnNormSupPisa} (see the comment in \citep[Example 5.5]{BenoitBertrand_2012_JTopAnal}) guarantee that there exists some $b\in \rr^{2d}$ \textit{(that can be built explicitly via Algorithm~\ref{Algorithm_RandomPartition} in Section~\ref{a_construction_of_bias_term})} such that the map
		\[
		\iota: 
		(\overline{\operatorname{Ball}_{(\rr^{2d},\ell_2)}(0,r)},\ell_2)
		\ni 
		x
		\mapsto 
		\frac1{2d} \sum_{k=1}^{2d}
		\,
		\delta_{
    		\sqrt{2d}
    		(x +b)_k
		} 
		\in (\mathcal{P}_2(\mathbb{R}),\mathcal{W}_2)
		,
		\]
		is an isometry. Combining the two maps, we find that the composite map
		\[
		\varphi: (\xxx_n,d_g) \ni 
		x
		\mapsto 
		\frac1{2d} \sum_{k=1}^{2d}
		\,
		\delta_{
    		\sqrt{2d}
    		(\phi(x) +b)_k
    	} 
		\in (\mathcal{P}_2(\mathbb{R}),\mathcal{W}_2)
		,
		\]
		is an isometry.  
		As in the proof of Lemma~\ref{lem_embedding_trees}, since $\iota\circ \phi$ is has range in the set of ``degenerate Gaussian mixtures'' $\mathcal{Z}\eqdef \{\pp\in \mathcal{GM}_2(\rr^d):\,\exists N\in \nn_+\,\exists w\in \Delta_N\,\exists \mu_1,\dots,\mu_N\in \rr^d\,s.t.\, \pp=\sum_{n=1}^N\,w_n \delta_{\mu_n}\}$ then, we obtain the conclusion by applying the comparison result in \citep[Proposition 6]{WassersteinGaussianMixtures} which states that the ``identity map'' on $(\mathcal{W}_2(\rr),\mathcal{W}_2)\ni \pp\mapsto \pp \in (\mathcal{GM}_2(\rr),\mw_2)$ is an isometry between on the collection of finitely supported probability measures in $\mathcal{Z}$.  
		Relabeling $\phi$ as $\phi(\cdot)+b$ yields the conclusion.  
	\end{proof}
	\begin{proof}[Proof of Proposition~\ref{proposition_manifoldprior}]
		The result is directly follows from Lemma~\ref{lemma_embedding_neuralification} applied to the map $\varphi$ of Lemma~\ref{lem_embedding_manifold_sampled} and relabeling $C\,C_K$ as $C_K$.
	\end{proof}
	\subsubsection{Proof of Corollaries to Main Results}
	\label{a_Proofs__aa__secondary___aaa_Corollaries}
	\begin{proof}[Proof of Corollary~\ref{cor_combinatorial_graphs}]
	We only upper-bound $\mathrm{cap}(\xxx_n)$ in order to apply Theorem~\ref{MAINTHEOREM_DETERMINISTIC}.  By Lemma~\ref{lem_doubling_to_capacity_control_lemma}
	$\log(\mathrm{cap}(\xxx_n,d_n))\leq 12 \log(c_{\mu})$ for any doubling measure $\mu$ on $(\xxx_n,d_n)$.  Since $(\xxx_n,d_n)$ is the metric space associated to a (finite) combinatorial graph in which every two vertices are connected by at most $2$ edges then, \citep[Theorem 5, Proposition 18, and Proposition 19]{Lestdoublingconstant_graph_durand2021doubling} we know that $c_{\mu}=1+\rho(A_{V})$.  
    Note also that $\operatorname{aspect}(\xxx_n,d_n)= \operatorname{diam}(\xxx_n,d_n)$ since the shortest distance between any two points in a combinatorial (metric) graph is $1$.  	
	Thus, $\log(\mathrm{cap}(\xxx_n,d_n))\leq 12 \log(1+\rho(A_{V}))$ and the result now follows from Theorem~\ref{MAINTHEOREM_DETERMINISTIC}.
	\end{proof}
	\begin{proof}[{Proof of Corollary~\ref{cor_Ricci_Version}}]
		As argued in \citep[pages 1121-1122]{BaudoinGarofalo_2011_DoublingRiemannian}, the complete $d$-dimensional Riemannian manifold $(M,g)$ satisfies the \textit{curvature-dimension inequality} of~\eqref{eq_cdi_beaudoin} for some $r\in \rr$ if and only if $(M,g)$'s Ricci curvature tensor (\citep[Definition 4.3.3]{jost2008riemannian}) is uniformly lower-bounded by $r$.  As demonstrated in \citep[Equation (1.1)]{BaudoinGarofalo_2011_DoublingRiemannian} the lower Ricci-cuvature bound on $(M,g)$ together with the Bishop-Gromov Volume Comparison Theorem (e.g. see \citep[Theorem 3.10]{Chavel_RiemannianIntro_1993}) implies that $c_{d_n}\leq d$.  Therefore, \citep[Proposition 1.7 (ii)]{Brue2021Extension} implies that $\log(\mathrm{cap}(\xxx_n,d_n))\leq 3 \log (c_{d_n})$.  The result therefore follows from Theorem~\ref{MAINTHEOREM_DETERMINISTIC}.   
	\end{proof}
	\subsubsection{Proof of Impossibility Results}
	\label{a_Proofs__aa__secondary___aaa_impossibility_result}

	\begin{proof}[Proof of Proposition~\ref{prop_impossibility_theorem_negativelyCurved}]
		\hfill\\
		\textbf{Step 1 - A lower Bound when Embedding into Spaces of Markov type $p\ge 2$}\hfill\\
		On \citep[page 173]{naor2006markov} the authors discuss that the proof of a result in \cite{LinialMagenNaor_2002_GraphEmbedding} implies the following for any metric space $\mathcal{R}$ of Markov type $2$ (see \citep[Definition 1.1]{naor2006markov}).  If $\Gamma_{g,\delta}\eqdef (X_{g,\delta},d_{g,\delta})$ is a graph (equipped with the usual graph geodesic distance) in which the shortest loop has length at least $g>0$ (called $\Gamma_{g,\delta}$'s \textit{girth}) and its average degree is $\delta>2$ then, any bi-Lipschitz embedding $f:\Gamma_{g,\delta}\rightarrow (\mathcal{R},d_{\mathcal{R}})$ (here $s=1$) has distortion $\operatorname{Dist}(f)$ at-least
		\begin{equation}
		\label{eq_PROOF__prop_impossibility_theorem_negativelyCurved___NaorsLowerBoundMarkovType2}
		    \operatorname{Dist}(f) 
		        \ge 
		    \frac{
		        \delta-2
		    }{
		        2\,M_p(\mathcal{R})
		    }
		        g^{(p-1)/p}
		,
		\end{equation}
		where $p$ is the Markov-type of $(\mathcal{R},d_{\mathcal{R}})$ as defined by in \citep[Definition 1.1]{ball1992markov} and $M_p(\mathcal{R})>0$ is a constant depending only on $\mathcal{R}$'s geometry. 
		
		\textbf{Step 2 - Markov Type of $(M,g)$}\hfill\\
		Let $(\mathcal{R},d_{\mathcal{R}})\eqdef (M,d_g)$ where $(M,g)$ is a complete Riemannian manifold with sectional curvature pinched between $[-C,-c]$ for some $0<c\le C<\infty$.  Therefore, \citep[Corollary 6.5]{naor2006markov} applies implying that $(M,d_g)$ has Markov type $p=2$.  
		
		\textbf{Step 3 - Realizing the Diverging Lower-Bound via Sextet Graphs}\hfill\\
		The main result of \cite{weiss1984girths} shows that, for every prime integer $n>2$ satisfying $n \bmod 16 \in \{\pm 3, \pm 5, \pm 7\}$, the sextet graph with $n$ vertices is a \textit{cubic graph} (i.e.\ each vertex has exactly $3$ edges connected to it and therefore $\delta=3>2$) and its girth $g$ is at-least $\frac{4}{3} \log_2(n) - 2$.  Therefore, setting $\Gamma_{g,\delta}$ to be the sextet graph with $n$-vertices we deduce from~\eqref{eq_PROOF__prop_impossibility_theorem_negativelyCurved___NaorsLowerBoundMarkovType2} and step 2 that any bi-Lipschitz embedding $f:\Gamma_{g,\delta}\rightarrow \mathcal{R}$ satisfies
		\[
    		\operatorname{Dist}(f) 
		        \ge 
		    \frac{
		        1
		    }{
		        2\,M_p(\mathcal{R})
		    }
		        \sqrt{
		        \frac{4}{3} \log_2(n) - 2
		        }
		.
		\]
		Setting $
		    c_{\mathcal{R}} \eqdef 
		    \frac{
		        1
		    }{
		        2\,M_p(\mathcal{R})
		    }$ yields the conclusion.  
	\end{proof}
	
	\begin{proof}[Proof of Proposition~\ref{prop_impossibility_theorem_positivecurvature}]
		We argue by contradiction.  Assume that $\mathcal{R}$ is a Riemannian manifold which is bi-Lipschitz embedded in the Hilbert space $\ell^2$ with distortion $D_1$; denote this bi-Lipschitz embedding by $\psi:(\mathcal{R},g) \hookrightarrow (\mathbb{R}^{2d},\ell_2)$
		and with the property that: for every $n\in \nn_+$ there is no $n$-point metric space $(\xxx_n,d_n)$ which can be bi-Lipschitz embedded in $(\mathcal{R},d_g)$ with distortion $D_2\left(\frac{\log(n)}{\log(\log(n))}\right)$ (here $D_2$ is a constant depending only on $\mathcal{R}$).  
		
		Therefore, for every $n$-point metric space $(\xxx_n,d_n)$ there exists a bi-Lipschitz embedding $\phi:(\xxx_n,d_n)\hookrightarrow (\mathcal{R},d_g)$ whose distortion $D$ is less than $D_2\frac{\log(n)}{\log(\log(n))}$.  
	    Thus $
		\psi\circ \phi:\xxx_n \rightarrow \ell^2,
		$
		is a bi-Lipschitz embedding with at-least distortion $D_1 D_2$. Thus, for $n$ large enough there is an $n$-point metric space which bi-Lipschitz embeds into $(\ell^2,\ell_2)$ with distortion less than $\mathcal{O}\left(\frac{\log(n)}{\log(\log(n))}\right)$.  This is a contradiction of a main result of \cite{Bourgain_1985_HilbertSpaceisaNogo}.  
		Whence, $(\mathcal{R},g)$ does not exist.  
	\end{proof}

\section{Algorithmic Construction of the ``Bias Term'' in our Embeddings' Proofs}
\label{a_construction_of_bias_term}
This short appendix contains an explicit algorithmic construction of the shift (or ``bias term''), denoted by $b$, used when embedding a bounded Euclidean ball into the Wasserstein space.  The term $b\in \rr^d$ maps the Euclidean ball into the space $\mathbb{R}_{<}^d\eqdef \{x\in \rr^d:\, x_1<\dots<x_d\}$ used in \citep[Example 5.5]{BenoitBertrand_2012_JTopAnal} and initially proposed in \citep[Proposition 7.4]{BenoitBertrand_2010_AnnNormSupPisa} .  
    \begin{algorithm}
    \caption{Initialize Bias}
    \label{Algorithm_RandomPartition}
    \begin{algorithmic}
    \Require Set of $n$-vectors $\mathbb{X}\eqdef \{x^{(1)},\dots,x^{(n)}\}$ in $\rr^K$,
    \State $b_0\eqdef 0$ 
        \Comment{Initialize first shift}
        \For{$n=1,\dots,N$}
            $\tilde{x}^{n}_{1}\eqdef x_1^{(n)} + b_1$
            \Comment{Dummy vectors}
        \EndFor
    \For{$k=1,\dots,K$}
            \Comment{Iteratively build bias components}
        \State $
        b_k \eqdef \max_{n\leq N}\,\operatorname{ReLU}(x_{k-1}^{(n)} - x_{k-1}^{(n)})$
            \For{$n\leq N$}
                $\tilde{x}^{n}_{k}\eqdef x_k^{(n)} + b_k$
                \Comment{Dummy vectors}
            \EndFor
    \EndFor
    \State \Return $b\eqdef (b_1,\dots,b_K)$ 
    \Comment{Return Bias)}
    \end{algorithmic}
    \end{algorithm}

\end{appendices}


\bibliography{References}

\begin{thebibliography}{150}
\providecommand{\natexlab}[1]{#1}
\providecommand{\url}[1]{\texttt{#1}}
\expandafter\ifx\csname urlstyle\endcsname\relax
  \providecommand{\doi}[1]{doi: #1}\else
  \providecommand{\doi}{doi: \begingroup \urlstyle{rm}\Url}\fi

\bibitem[Acciaio et~al.(2020)Acciaio, Backhoff-Veraguas, and
  Zalashko]{AcciaioBackhoffZalashki_CausalOT_2020}
B.~Acciaio, J.~Backhoff-Veraguas, and A.~Zalashko.
\newblock Causal optimal transport and its links to enlargement of filtrations
  and continuous-time stochastic optimization.
\newblock \emph{Stochastic Process. Appl.}, 130\penalty0 (5):\penalty0
  2918--2953, 2020.
\newblock URL \url{https://doi.org/10.1016/j.spa.2019.08.009}.

\bibitem[Acciaio et~al.(2022)Acciaio, Kratsios, and Pammer]{acciaio2022metric}
Beatrice Acciaio, Anastasis Kratsios, and Gudmund Pammer.
\newblock Metric hypertransformers are universal adapted maps.
\newblock \emph{{A}r{X}i{V}}, 2022.

\bibitem[Agarwal and
  Niyogi(2009)]{AgarwalNiyogi_2009_GeneralizationRankingAlgosJMLR}
Shivani Agarwal and Partha Niyogi.
\newblock Generalization bounds for ranking algorithms via algorithmic
  stability.
\newblock \emph{J. Mach. Learn. Res.}, 10:\penalty0 441--474, 2009.
\newblock ISSN 1532-4435.

\bibitem[Amari(2016)]{amari2016information}
Shun-ichi Amari.
\newblock Information geometry and its applications.
\newblock 194:\penalty0 xiii+374, 2016.

\bibitem[Andoni et~al.(2018)Andoni, Naor, and
  Neiman]{NaorNeimanAndoniSnowflake2018AnnalsSENSQ}
Alexandr Andoni, Assaf Naor, and Ofer Neiman.
\newblock Snowflake universality of {W}asserstein spaces.
\newblock \emph{Ann. Sci. \'{E}c. Norm. Sup\'{e}r. (4)}, 51\penalty0
  (3):\penalty0 657--700, 2018.

\bibitem[Assouad(1979)]{Assouad}
Patrice Assouad.
\newblock \'{E}tude d'une dimension m\'{e}trique li\'{e}e \`a la
  possibilit\'{e} de plongements dans {${\bf R}^{n}$}.
\newblock \emph{C. R. Acad. Sci. Paris S\'{e}r. A-B}, 288\penalty0
  (15):\penalty0 A731--A734, 1979.

\bibitem[Backhoff et~al.(2022)Backhoff, Bartl, Beiglb\"{o}ck, and
  Wiesel]{Gudi2022_Ann}
Julio Backhoff, Daniel Bartl, Mathias Beiglb\"{o}ck, and Johannes Wiesel.
\newblock Estimating processes in adapted {W}asserstein distance.
\newblock \emph{Ann. Appl. Probab.}, 32\penalty0 (1):\penalty0 529--550, 2022.

\bibitem[Backhoff-Veraguas and
  Pammer(2022{\natexlab{a}})]{BackhoddPammer_2022_StabilityMartingaleAndWeakOT}
J.~Backhoff-Veraguas and G.~Pammer.
\newblock Stability of martingale optimal transport and weak optimal transport.
\newblock \emph{Ann. Appl. Probab.}, 32\penalty0 (1):\penalty0 721--752,
  2022{\natexlab{a}}.
\newblock URL \url{https://doi.org/10.1214/21-aap1694}.

\bibitem[Backhoff-Veraguas and Pammer(2022{\natexlab{b}})]{GudiJulio_2022_AAP}
J.~Backhoff-Veraguas and G.~Pammer.
\newblock Stability of martingale optimal transport and weak optimal transport.
\newblock \emph{Ann. Appl. Probab.}, 32\penalty0 (1):\penalty0 721--752,
  2022{\natexlab{b}}.

\bibitem[Backhoff-Veraguas et~al.(2020{\natexlab{a}})Backhoff-Veraguas, Bartl,
  Beiglb\"{o}ck, and Eder]{Backhodd_Bartl_Beiglbock__AlladaptedTopequal_2020}
Julio Backhoff-Veraguas, Daniel Bartl, Mathias Beiglb\"{o}ck, and Manu Eder.
\newblock All adapted topologies are equal.
\newblock \emph{Probab. Theory Related Fields}, 178\penalty0 (3-4):\penalty0
  1125--1172, 2020{\natexlab{a}}.
\newblock ISSN 0178-8051.
\newblock URL \url{https://doi.org/10.1007/s00440-020-00993-8}.

\bibitem[Backhoff-Veraguas et~al.(2020{\natexlab{b}})Backhoff-Veraguas, Bartl,
  Beiglb\"{o}ck, and Eder]{allAdaptedTopsareEqual}
Julio Backhoff-Veraguas, Daniel Bartl, Mathias Beiglb\"{o}ck, and Manu Eder.
\newblock All adapted topologies are equal.
\newblock \emph{Probab. Theory Related Fields}, 178\penalty0 (3-4):\penalty0
  1125--1172, 2020{\natexlab{b}}.

\bibitem[Ball(1992)]{ball1992markov}
Keith Ball.
\newblock Markov chains, riesz transforms and lipschitz maps.
\newblock \emph{Geometric \& Functional Analysis GAFA}, 2\penalty0
  (2):\penalty0 137--172, 1992.

\bibitem[Bartal et~al.(2005)Bartal, Linial, Mendel, and
  Naor]{Bartal2005RamseyPhenomena}
Yair Bartal, Nathan Linial, Manor Mendel, and Assaf Naor.
\newblock On metric {R}amsey-type phenomena.
\newblock \emph{Ann. of Math. (2)}, 162\penalty0 (2):\penalty0 643--709, 2005.
\newblock URL \url{https://doi.org/10.4007/annals.2005.162.643}.

\bibitem[Bartlett and Long(2021)]{BartlettFailuresGeneralizationJMLR2021}
Peter~L. Bartlett and Philip~M. Long.
\newblock Failures of model-dependent generalization bounds for least-norm
  interpolation.
\newblock \emph{J. Mach. Learn. Res.}, 22:\penalty0 Paper No. 204, 15, 2021.
\newblock ISSN 1532-4435.

\bibitem[Bartlett et~al.(2017)Bartlett, Foster, and
  Telgarsky]{bartlett2017spectrally}
Peter~L Bartlett, Dylan~J Foster, and Matus~J Telgarsky.
\newblock Spectrally-normalized margin bounds for neural networks.
\newblock \emph{Advances in neural information processing systems}, 30, 2017.

\bibitem[Bartlett et~al.(2019)Bartlett, Harvey, Liaw, and
  Mehrabian]{JLMLR_BartlessHAveyLiawMagrabian_2019_VCBoundsReLUffNN}
Peter~L. Bartlett, Nick Harvey, Christopher Liaw, and Abbas Mehrabian.
\newblock Nearly-tight vc-dimension and pseudodimension bounds for piecewise
  linear neural networks.
\newblock \emph{Journal of Machine Learning Research}, 20\penalty0
  (63):\penalty0 1--17, 2019.
\newblock URL \url{http://jmlr.org/papers/v20/17-612.html}.

\bibitem[Baudoin and Garofalo(2011)]{BaudoinGarofalo_2011_DoublingRiemannian}
Fabrice Baudoin and Nicola Garofalo.
\newblock Perelman's entropy and doubling property on {R}iemannian manifolds.
\newblock \emph{J. Geom. Anal.}, 21\penalty0 (4):\penalty0 1119--1131, 2011.
\newblock URL \url{https://doi.org/10.1007/s12220-010-9180-x}.

\bibitem[Baudoin et~al.(2014)Baudoin, Bonnefont, and
  Garofalo]{BaudoinBonnefontGarofalo_2014_SubRiemannianCurvatureDoubling}
Fabrice Baudoin, Michel Bonnefont, and Nicola Garofalo.
\newblock A sub-{R}iemannian curvature-dimension inequality, volume doubling
  property and the {P}oincar\'{e} inequality.
\newblock \emph{Math. Ann.}, 358\penalty0 (3-4):\penalty0 833--860, 2014.
\newblock URL \url{https://doi.org/10.1007/s00208-013-0961-y}.

\bibitem[Beiglb\"{o}ck et~al.(2017{\natexlab{a}})Beiglb\"{o}ck,
  Henry-Labord\`ere, and Touzi]{MonotoneTransport_BeiglbrockLabodereTouzi_2017}
Mathias Beiglb\"{o}ck, Pierre Henry-Labord\`ere, and Nizar Touzi.
\newblock Monotone martingale transport plans and {S}korokhod embedding.
\newblock \emph{Stochastic Process. Appl.}, 127\penalty0 (9):\penalty0
  3005--3013, 2017{\natexlab{a}}.
\newblock URL \url{https://doi.org/10.1016/j.spa.2017.01.004}.

\bibitem[Beiglb\"{o}ck et~al.(2017{\natexlab{b}})Beiglb\"{o}ck, Nutz, and
  Touzi]{MathiasNizarMarcel_2017_AOP}
Mathias Beiglb\"{o}ck, Marcel Nutz, and Nizar Touzi.
\newblock Complete duality for martingale optimal transport on the line.
\newblock \emph{Ann. Probab.}, 45\penalty0 (5):\penalty0 3038--3074,
  2017{\natexlab{b}}.

\bibitem[Belkin and Niyogi(2003)]{belkin2003laplacian}
Mikhail Belkin and Partha Niyogi.
\newblock Laplacian eigenmaps for dimensionality reduction and data
  representation.
\newblock \emph{Neural computation}, 15\penalty0 (6):\penalty0 1373--1396,
  2003.

\bibitem[Belkin and Niyogi(2006)]{belkin2006convergence}
Mikhail Belkin and Partha Niyogi.
\newblock Convergence of laplacian eigenmaps.
\newblock \emph{Advances in neural information processing systems}, 19, 2006.

\bibitem[Bertrand and Kloeckner(2012)]{BenoitBertrand_2012_JTopAnal}
J\'{e}r\^{o}me Bertrand and Beno\^{\i}t Kloeckner.
\newblock A geometric study of {W}asserstein spaces: {H}adamard spaces.
\newblock \emph{J. Topol. Anal.}, 4\penalty0 (4):\penalty0 515--542, 2012.
\newblock URL \url{https://doi.org/10.1142/S1793525312500227}.

\bibitem[Biggs and Hoare(1983)]{biggs1983sextet}
Norman~L Biggs and MJ~Hoare.
\newblock The sextet construction for cubic graphs.
\newblock \emph{Combinatorica}, 3\penalty0 (2):\penalty0 153--165, 1983.

\bibitem[Bigot et~al.(2017)Bigot, Gouet, Klein, and
  L\'{o}pez]{bigot2017geodesic}
J\'{e}r\'{e}mie Bigot, Ra\'{u}l Gouet, Thierry Klein, and Alfredo L\'{o}pez.
\newblock Geodesic {PCA} in the {W}asserstein space by convex {PCA}.
\newblock \emph{Ann. Inst. Henri Poincar\'{e} Probab. Stat.}, 53\penalty0
  (1):\penalty0 1--26, 2017.
\newblock URL \url{https://doi.org/10.1214/15-AIHP706}.

\bibitem[Bilokopytov and Kratsios(2020)]{kratsios2020non}
Eugene Bilokopytov and Anastasis Kratsios.
\newblock {N}on-{E}uclidean {U}niversal {A}pproximation.
\newblock \emph{NeurIPS}, 33, 2020.

\bibitem[Bion-Nadal and Talay(2019)]{NAdalTalay_Transport_Refined_SDE}
Jocelyne Bion-Nadal and Denis Talay.
\newblock On a {W}asserstein-type distance between solutions to stochastic
  differential equations.
\newblock \emph{Ann. Appl. Probab.}, 29\penalty0 (3):\penalty0 1609--1639,
  2019.
\newblock URL \url{https://doi.org/10.1214/18-AAP1423}.

\bibitem[Bishop(1994)]{bishop1994mixture}
Christopher~M Bishop.
\newblock Mixture density networks.
\newblock \emph{{U}npublished}, 1994.

\bibitem[Bourgain(1985)]{Bourgain_1985_HilbertSpaceisaNogo}
J.~Bourgain.
\newblock On {L}ipschitz embedding of finite metric spaces in {H}ilbert space.
\newblock \emph{Israel J. Math.}, 52\penalty0 (1-2):\penalty0 46--52, 1985.

\bibitem[Bru\`e et~al.(2021)Bru\`e, Di~Marino, and Stra]{Brue2021Extension}
Elia Bru\`e, Simone Di~Marino, and Federico Stra.
\newblock Linear {L}ipschitz and {$C^1$} extension operators through random
  projection.
\newblock \emph{J. Funct. Anal.}, 280\penalty0 (4):\penalty0 Paper No. 108868,
  21, 2021.
\newblock URL \url{https://doi.org/10.1016/j.jfa.2020.108868}.

\bibitem[Bubeck et~al.(2020)Bubeck, Eldan, Lee, and
  Mikulincer]{NEURIPS2020_34609bdc}
Sebastien Bubeck, Ronen Eldan, Yin~Tat Lee, and Dan Mikulincer.
\newblock Network size and size of the weights in memorization with two-layers
  neural networks.
\newblock 33:\penalty0 4977--4986, 2020.

\bibitem[Buckley and Harary(1988)]{BuckleyHarary_WheelGraphs_1988}
Fred Buckley and Frank Harary.
\newblock On the {E}uclidean dimension of a wheel.
\newblock \emph{Graphs Combin.}, 4\penalty0 (1):\penalty0 23--30, 1988.
\newblock URL \url{https://doi.org/10.1007/BF01864150}.

\bibitem[Burger and Neubauer(2001)]{BurgerNeubauer_2001}
Martin Burger and Andreas Neubauer.
\newblock Error bounds for approximation with neural networks.
\newblock \emph{J. Approx. Theory}, 112\penalty0 (2):\penalty0 235--250, 2001.

\bibitem[Chami et~al.(2020)Chami, Gu, Chatziafratis, and
  R{\'e}]{chami2020trees}
Ines Chami, Albert Gu, Vaggos Chatziafratis, and Christopher R{\'e}.
\newblock From trees to continuous embeddings and back: Hyperbolic hierarchical
  clustering.
\newblock \emph{Advances in Neural Information Processing Systems},
  33:\penalty0 15065--15076, 2020.
\newblock URL \url{https://dl.acm.org/doi/abs/10.5555/3495724.3496987}.

\bibitem[Chami et~al.(2021)Chami, Gu, Nguyen, and R{\'e}]{chami2021horopca}
Ines Chami, Albert Gu, Dat~P Nguyen, and Christopher R{\'e}.
\newblock Horopca: Hyperbolic dimensionality reduction via horospherical
  projections.
\newblock In \emph{International Conference on Machine Learning}, pages
  1419--1429. PMLR, 2021.

\bibitem[Chavel(1993)]{Chavel_RiemannianIntro_1993}
Isaac Chavel.
\newblock Riemannian geometry-a modern introduction.
\newblock 108:\penalty0 xii+386, 1993.

\bibitem[Chen et~al.(2021)Chen, Lin, and M{\"u}ller]{chen2021wasserstein}
Yaqing Chen, Zhenhua Lin, and Hans-Georg M{\"u}ller.
\newblock Wasserstein regression.
\newblock \emph{Journal of the American Statistical Association}, pages 1--14,
  2021.

\bibitem[Cheridito et~al.(2021)Cheridito, Jentzen, and
  Rossmannek]{Florian2_2021}
Patrick Cheridito, Arnulf Jentzen, and Florian Rossmannek.
\newblock Efficient approximation of high-dimensional functions with neural
  networks.
\newblock \emph{IEEE Transactions on Neural Networks and Learning Systems},
  pages 1--15, 2021.

\bibitem[Chevallier(2018)]{chevallier2018uniform}
Julien Chevallier.
\newblock Uniform decomposition of probability measures: quantization,
  clustering and rate of convergence.
\newblock \emph{Journal of Applied Probability}, 55\penalty0 (4):\penalty0
  1037--1045, 2018.

\bibitem[Costa et~al.(2015)Costa, Santos, and Strapasson]{COSTA201559}
Sueli~I.R. Costa, Sandra~A. Santos, and João~E. Strapasson.
\newblock Fisher information distance: A geometrical reading.
\newblock \emph{Discrete Applied Mathematics}, 197:\penalty0 59--69, 2015.
\newblock Distance Geometry and Applications.

\bibitem[Cruceru et~al.(2020)Cruceru, Becigneul, and
  Ganea]{cruceru2020computationally}
Calin Cruceru, Gary Becigneul, and Octavian-Eugen Ganea.
\newblock Computationally tractable riemannian manifolds for graph embeddings.
\newblock \emph{arXiv preprint arXiv:2002.08665}, 2020.

\bibitem[Cruceru et~al.(2021)Cruceru, Becigneul, and
  Ganea]{Ganea_embeddin_2021_AAAI}
Calin Cruceru, Gary Becigneul, and Octavian-Eugen Ganea.
\newblock Computationally tractable riemannian manifolds for graph embeddings.
\newblock \emph{Proceedings of the AAAI Conference on Artificial Intelligence},
  35\penalty0 (8):\penalty0 7133--7141, May 2021.
\newblock URL \url{https://ojs.aaai.org/index.php/AAAI/article/view/16877}.

\bibitem[Cuesta and Matran(1989)]{ClosedFormW2}
Juan~Antonio Cuesta and Carlos Matran.
\newblock Notes on the wasserstein metric in hilbert spaces.
\newblock \emph{The Annals of Probability}, 17\penalty0 (3):\penalty0
  1264--1276, 1989.
\newblock ISSN 00911798.
\newblock URL \url{http://www.jstor.org/stable/2244406}.

\bibitem[Dai and M\"{u}ller(2018)]{MR3852654}
Xiongtao Dai and Hans-Georg M\"{u}ller.
\newblock Principal component analysis for functional data on {R}iemannian
  manifolds and spheres.
\newblock \emph{Ann. Statist.}, 46\penalty0 (6B):\penalty0 3334--3361, 2018.

\bibitem[Daniely(2020)]{NEURIPS2020_662a2e96}
Amit Daniely.
\newblock Neural networks learning and memorization with (almost) no
  over-parameterization.
\newblock 33:\penalty0 9007--9016, 2020.

\bibitem[Das and Kumar(2004)]{DasKumar_Spectral_Radii}
Kinkar~Ch. Das and Pawan Kumar.
\newblock Some new bounds on the spectral radius of graphs.
\newblock \emph{Discrete Math.}, 281\penalty0 (1-3):\penalty0 149--161, 2004.
\newblock ISSN 0012-365X.

\bibitem[Debarnot and Weiss(2022)]{debarnot2022deep}
Valentin Debarnot and Pierre Weiss.
\newblock Deep-blur : Blind identification and deblurring with convolutional
  neural networks.
\newblock June 2022.

\bibitem[Delon and Desolneux(2020)]{WassersteinGaussianMixtures}
Julie Delon and Agnès Desolneux.
\newblock A {W}asserstein-type distance in the space of {G}aussian mixture
  models.
\newblock \emph{SIAM Journal on Imaging Sciences}, 13\penalty0 (2):\penalty0
  936--970, 2020.

\bibitem[Delon and Desolneux(2021)]{GaussMixGithub}
Julie Delon and Agnès Desolneux.
\newblock Optimal transport between gaussian mixture models.
\newblock \url{https://github.com/judelo/gmmot}, 2021.

\bibitem[Dhingra et~al.(2018)Dhingra, Shallue, Norouzi, Dai, and
  Dahl]{dhingra2018embedding}
Bhuwan Dhingra, Christopher~J. Shallue, Mohammad Norouzi, Andrew~M. Dai, and
  George~E. Dahl.
\newblock Embedding text in hyperbolic spaces.
\newblock \emph{TextGraphs@NAACL-HLT}, pages 59--69, 2018.

\bibitem[Dolinsky and Soner(2014)]{MeteDolinsky_2014_PTRF}
Yan Dolinsky and H.~Mete Soner.
\newblock Martingale optimal transport and robust hedging in continuous time.
\newblock \emph{Probab. Theory Related Fields}, 160\penalty0 (1-2):\penalty0
  391--427, 2014.
\newblock ISSN 0178-8051.

\bibitem[Dowty(2018)]{ChenstovTheorem_ExponentialFamilies2018}
James~G. Dowty.
\newblock Chentsov's theorem for exponential families.
\newblock \emph{Inf. Geom.}, 1\penalty0 (1):\penalty0 117--135, 2018.

\bibitem[Durand-Cartagena et~al.(2021)Durand-Cartagena, Soria, and
  Tradacete]{Lestdoublingconstant_graph_durand2021doubling}
Estibalitz Durand-Cartagena, Javier Soria, and Pedro Tradacete.
\newblock Doubling constants and spectral theory on graphs.
\newblock \emph{arXiv preprint arXiv:2111.09199}, 2021.

\bibitem[Dvoretzky(1961)]{Dvoretzky_OG_1960}
Aryeh Dvoretzky.
\newblock Some results on convex bodies and {B}anach spaces.
\newblock pages 123--160, 1961.

\bibitem[E. and Wojtowytsch(2022)]{WeinanWojtowytsch_pathnorm_ReLU2022}
Weinan E. and Stephan Wojtowytsch.
\newblock Representation formulas and pointwise properties for {B}arron
  functions.
\newblock \emph{Calc. Var. Partial Differential Equations}, 61\penalty0
  (2):\penalty0 Paper No. 46, 37, 2022.

\bibitem[Eidi and Jost(2020)]{JostNature2}
Marzieh Eidi and J{\"u}rgen Jost.
\newblock Ollivier {R}icci curvature of directed hypergraphs.
\newblock \emph{Scientific Reports}, 10\penalty0 (1):\penalty0 12466, 2020.

\bibitem[Erd\H{o}s(1964/65)]{Erdos_1964TuranTheorem}
P.~Erd\H{o}s.
\newblock On an extremal problem in graph theory.
\newblock \emph{Colloq. Math.}, 13:\penalty0 251--254, 1964/65.
\newblock ISSN 0010-1354.

\bibitem[Erd\H{o}s et~al.(1966)Erd\H{o}s, R\'{e}nyi, and
  S\'{o}s]{ErdosSos_Friendship1966}
P.~Erd\H{o}s, A.~R\'{e}nyi, and V.~T. S\'{o}s.
\newblock On a problem of graph theory.
\newblock \emph{Studia Sci. Math. Hungar.}, 1:\penalty0 215--235, 1966.
\newblock ISSN 0081-6906.

\bibitem[Erd\"{o}s and Sachs(1963)]{erdregukre}
Paul Erd\"{o}s and Horst Sachs.
\newblock Regukre graphen gegebener taillenweite mit minimaler knotenzahl.
\newblock \emph{Wissenschaftliche Zeitschrift der Martin-Luther-Universität
  Halle}, pages 251--258, 1963.

\bibitem[Eriksson-Bique(2018)]{eriksson2018quantitative}
Sylvester Eriksson-Bique.
\newblock Quantitative bi-{L}ipschitz embeddings of bounded-curvature manifolds
  and orbifolds.
\newblock \emph{Geometry \& Topology}, 22\penalty0 (4):\penalty0 1961--2026,
  2018.

\bibitem[Falconer(1986)]{Falconer_1986_GeometryofFractalSets}
K.~J. Falconer.
\newblock The geometry of fractal sets, 1986.

\bibitem[Federer(1978)]{Federer_GeometricMeasureTheory_1978}
Herbert Federer.
\newblock Colloquium lectures on geometric measure theory.
\newblock \emph{Bull. Amer. Math. Soc.}, 84\penalty0 (3):\penalty0 291--338,
  1978.

\bibitem[Feldman and Zhang(2020)]{NEURIPS2020_1e14bfe2}
Vitaly Feldman and Chiyuan Zhang.
\newblock What neural networks memorize and why: Discovering the long tail via
  influence estimation.
\newblock 33:\penalty0 2881--2891, 2020.

\bibitem[Flamary et~al.(2021)Flamary, Courty, Gramfort, Alaya, Boisbunon,
  Chambon, Chapel, Corenflos, Fatras, Fournier, Gautheron, Gayraud, Janati,
  Rakotomamonjy, Redko, Rolet, Schutz, Seguy, Sutherland, Tavenard, Tong, and
  Vayer]{PyOpt_2021_FlamaryCoutryetAl}
Rémi Flamary, Nicolas Courty, Alexandre Gramfort, Mokhtar~Z. Alaya, Aurélie
  Boisbunon, Stanislas Chambon, Laetitia Chapel, Adrien Corenflos, Kilian
  Fatras, Nemo Fournier, Léo Gautheron, Nathalie~T.H. Gayraud, Hicham Janati,
  Alain Rakotomamonjy, Ievgen Redko, Antoine Rolet, Antony Schutz, Vivien
  Seguy, Danica~J. Sutherland, Romain Tavenard, Alexander Tong, and Titouan
  Vayer.
\newblock {POT}: {P}ython optimal transport.
\newblock \emph{Journal of Machine Learning Research}, 22\penalty0
  (78):\penalty0 1--8, 2021.
\newblock URL \url{http://jmlr.org/papers/v22/20-451.html}.

\bibitem[Fletcher et~al.(2009)Fletcher, Venkatasubramanian, and
  Joshi]{fletcher2003statistics}
P.~Thomas Fletcher, Suresh Venkatasubramanian, and Sarang Joshi.
\newblock The geometric median on riemannian manifolds with application to
  robust atlas estimation.
\newblock \emph{NeuroImage}, 45\penalty0 (1):\penalty0 S143 -- S152, 2009.
\newblock Mathematics in Brain Imaging.

\bibitem[Fr{\'e}chet(1906)]{frechet1906quelques}
M~Maurice Fr{\'e}chet.
\newblock Sur quelques points du calcul fonctionnel.
\newblock \emph{Rendiconti del Circolo Matematico di Palermo (1884-1940)},
  22\penalty0 (1):\penalty0 1--72, 1906.

\bibitem[Ganea et~al.(2018)Ganea, Becigneul, and Hofmann]{ganea2018hyperbolic}
Octavian Ganea, Gary Becigneul, and Thomas Hofmann.
\newblock Hyperbolic neural networks.
\newblock \emph{Advances in Neural Information Processing Systems},
  31:\penalty0 5345--5355, 2018.
\newblock URL
  \url{https://proceedings.neurips.cc/paper/2018/file/dbab2adc8f9d078009ee3fa810bea142-Paper.pdf}.

\bibitem[Giovanni et~al.(2022)Giovanni, Luise, and
  Bronstein]{digiovanni2022heterogeneous}
Francesco~Di Giovanni, Giulia Luise, and Michael Bronstein.
\newblock Heterogeneous manifolds for curvature-aware graph embedding.
\newblock \emph{{A}r{X}i{V}}, 2022.

\bibitem[Graf and Luschgy(2000)]{FoundationsofQuantization2000GrafLuschgy}
Siegfried Graf and Harald Luschgy.
\newblock Foundations of quantization for probability distributions.
\newblock 1730:\penalty0 x+230, 2000.

\bibitem[Guo and Ob\l~\'{o}j(2019)]{JanGuo_2019_AAP}
Gaoyue Guo and Jan Ob\l~\'{o}j.
\newblock Computational methods for martingale optimal transport problems.
\newblock \emph{Ann. Appl. Probab.}, 29\penalty0 (6):\penalty0 3311--3347,
  2019.

\bibitem[Gupta(2000)]{Gupta_Trees_Quantitiativefinitedimembeddings_ACM}
A.~Gupta.
\newblock Embedding tree metrics into low-dimensional {E}uclidean spaces.
\newblock \emph{Discrete Comput. Geom.}, 24\penalty0 (1):\penalty0 105--116,
  2000.

\bibitem[Heinonen(2001)]{heinonen2001lectures}
Juha Heinonen.
\newblock Lectures on analysis on metric spaces, 2001.

\bibitem[Heinonen(2003)]{heinonen2003geometric}
Juha Heinonen.
\newblock Geometric embeddings of metric spaces, 2003.

\bibitem[Hopf and Rinow(1931)]{Hopf_Rinow_1931}
H.~Hopf and W.~Rinow.
\newblock Ueber den {B}egriff der vollst\"{a}ndigen differentialgeometrischen
  {F}l\"{a}che.
\newblock \emph{Comment. Math. Helv.}, 3\penalty0 (1):\penalty0 209--225, 1931.

\bibitem[Jost(2017)]{jost2008riemannian}
J\"{u}rgen Jost.
\newblock Riemannian geometry and geometric analysis, 2017.

\bibitem[Jung(1901)]{Jung1901}
Heinrich Jung.
\newblock Ueber die kleinste kugel, die eine räumliche figur einschliesst.
\newblock \emph{Journal für die reine und angewandte Mathematik},
  123:\penalty0 241--257, 1901.
\newblock URL \url{http://eudml.org/doc/149122}.

\bibitem[Jungerman and Ringel(1978)]{Jungerman_cocktailgraph_1978}
Mark Jungerman and Gerhard Ringel.
\newblock The genus of the {$n$}-octahedron: regular cases.
\newblock \emph{J. Graph Theory}, 2\penalty0 (1):\penalty0 69--75, 1978.
\newblock ISSN 0364-9024.

\bibitem[Katz and
  Katz(2011)]{KatzUsadiKatz_GeomDedicata_2011__BiLipEmbeddingFinManifolds}
Karin~Usadi Katz and Mikhail~G. Katz.
\newblock Bi-{L}ipschitz approximation by finite-dimensional imbeddings.
\newblock \emph{Geom. Dedicata}, 150:\penalty0 131--136, 2011.

\bibitem[Kidger and Lyons(2020)]{kidger2019universal}
Patrick Kidger and Terry Lyons.
\newblock Universal approximation with deep narrow networks.
\newblock 125:\penalty0 2306--2327, 09--12 Jul 2020.

\bibitem[Kingma and Ba(2015)]{DBLP:journals/corr/KingmaB14}
Diederik~P. Kingma and Jimmy Ba.
\newblock Adam: {A} method for stochastic optimization.
\newblock In \emph{3rd International Conference on Learning Representations,
  {ICLR} 2015, San Diego, CA, USA, May 7-9, 2015, Conference Track
  Proceedings}, 2015.

\bibitem[Kloeckner(2010)]{BenoitBertrand_2010_AnnNormSupPisa}
Beno\^{\i}t Kloeckner.
\newblock A geometric study of {W}asserstein spaces: {E}uclidean spaces.
\newblock \emph{Ann. Sc. Norm. Super. Pisa Cl. Sci. (5)}, 9\penalty0
  (2):\penalty0 297--323, 2010.
\newblock URL \url{https://doi.org/10.2422/2036-2145.2010.2.03}.

\bibitem[Kloeckner(2012)]{KloecknerQuantizationAhlforsRegular2012}
Beno\^{\i}t Kloeckner.
\newblock Approximation by finitely supported measures.
\newblock \emph{ESAIM Control Optim. Calc. Var.}, 18\penalty0 (2):\penalty0
  343--359, 2012.

\bibitem[Kloeckner and Bertrand(2016)]{kloeckner2016geometric}
Beno{\^\i}t~R Kloeckner and J\'{e}rome Bertrand.
\newblock A geometric study of wasserstein spaces: isometric rigidity in
  negative curvature.
\newblock \emph{International Mathematics Research Notices}, 2016\penalty0
  (5):\penalty0 1368--1386, 2016.

\bibitem[Kochurov et~al.(2020)Kochurov, Karimov, and
  Kozlukov]{kochurov2020geoopt}
Max Kochurov, Rasul Karimov, and Serge Kozlukov.
\newblock Geoopt: Riemannian optimization in pytorch.
\newblock In Hal~Daumé III and Aarti Singh, editors, \emph{Proceedings of the
  37th International Conference on Machine Learning}, volume 119 of
  \emph{Proceedings of Machine Learning Research}, Virtual, 13--18 Jul 2020.
  PMLR.

\bibitem[Kolyvakis et~al.(2020)Kolyvakis, Kalousis, and
  Kiritsis]{kolyvakis2020hyperbolic}
Prodromos Kolyvakis, Alexandros Kalousis, and Dimitris Kiritsis.
\newblock Hyperbolic knowledge graph embeddings for knowledge base completion.
\newblock In \emph{European Semantic Web Conference}, pages 199--214. Springer,
  2020.

\bibitem[Kratsios(2022)]{kratsios2021_GCDs}
Anastasis Kratsios.
\newblock Universal regular conditional distributions.
\newblock \emph{Constructive Approximation (minor revisions)}, 2022.

\bibitem[Kratsios and Papon(2022)]{kratsios2021_GDL}
Anastasis Kratsios and L{\'e}onie Papon.
\newblock Universal approximation theorems for differentiable geometric deep
  learning.
\newblock \emph{Journal of Machine Learning Research}, 23:\penalty0 1--73,
  2022.

\bibitem[Kratsios et~al.(2022)Kratsios, Zamanlooy, Liu, and
  Dokmani{\'c}]{AB_2022}
Anastasis Kratsios, Behnoosh Zamanlooy, Tianlin Liu, and Ivan Dokmani{\'c}.
\newblock Universal approximation under constraints is possible with
  transformers.
\newblock In \emph{International Conference on Learning Representations}, 2022.

\bibitem[Krauthgamer et~al.(2005)Krauthgamer, Lee, Mendel, and
  Naor]{KrauthgameLeeNaor2004}
R.~Krauthgamer, J.~R. Lee, M.~Mendel, and A.~Naor.
\newblock Measured descent: a new embedding method for finite metrics.
\newblock \emph{Geom. Funct. Anal.}, 15\penalty0 (4):\penalty0 839--858, 2005.

\bibitem[Krioukov et~al.(2010)Krioukov, Papadopoulos, Kitsak, Vahdat, and
  Bogun{\'a}]{krioukov2010hyperbolic}
Dmitri Krioukov, Fragkiskos Papadopoulos, Maksim Kitsak, Amin Vahdat, and
  Mari{\'a}n Bogun{\'a}.
\newblock Hyperbolic geometry of complex networks.
\newblock \emph{Physical Review E}, 82\penalty0 (3):\penalty0 036106, 2010.

\bibitem[Le et~al.(2019)Le, Roller, Papaxanthos, Kiela, and
  Nickel]{le2019inferring}
Matt Le, Stephen Roller, Laetitia Papaxanthos, Douwe Kiela, and Maximilian
  Nickel.
\newblock Inferring concept hierarchies from text corpora via hyperbolic
  embeddings.
\newblock \emph{{A}r{X}i{V}}, 2019.

\bibitem[Le~Donne and Rajala(2015)]{le2015assouad}
Enrico Le~Donne and Tapio Rajala.
\newblock Assouad dimension, nagata dimension, and uniformly close metric
  tangents.
\newblock \emph{Indiana University Mathematics Journal}, pages 21--54, 2015.

\bibitem[Linial et~al.(2002)Linial, Magen, and
  Naor]{LinialMagenNaor_2002_GraphEmbedding}
N.~Linial, A.~Magen, and A.~Naor.
\newblock Girth and {E}uclidean distortion.
\newblock \emph{Geom. Funct. Anal.}, 12\penalty0 (2):\penalty0 380--394, 2002.

\bibitem[Liu and Neufeld(2019)]{ChongAriel_2019}
Chong Liu and Ariel Neufeld.
\newblock Compactness criterion for semimartingale laws and semimartingale
  optimal transport.
\newblock \emph{Trans. Amer. Math. Soc.}, 372\penalty0 (1):\penalty0 187--231,
  2019.

\bibitem[Liu and Pag\`{e}s(2020)]{LiuPages2020Quantization}
Yating Liu and Gilles Pag\`{e}s.
\newblock Convergence rate of optimal quantization and application to the
  clustering performance of the empirical measure.
\newblock \emph{Journal of Machine Learning Research}, 21\penalty0
  (86):\penalty0 1--36, 2020.
\newblock URL \url{http://jmlr.org/papers/v21/18-804.html}.

\bibitem[L{\'o}pez and Strube(2020)]{lopez2020fully}
Federico L{\'o}pez and Michael Strube.
\newblock A fully hyperbolic neural model for hierarchical multi-class
  classification.
\newblock \emph{{A}r{X}i{V}}, 2020.

\bibitem[Lopez et~al.(2021)Lopez, Pozzetti, Trettel, Strube, and
  Wienhard]{LopezPozzettiTrettelStrubeWienhard_2021__Symmetric_Spaces_for_Graph_Embeddings}
Federico Lopez, Beatrice Pozzetti, Steve Trettel, Michael Strube, and Anna
  Wienhard.
\newblock Symmetric spaces for graph embeddings: A finsler-riemannian approach.
\newblock 139:\penalty0 7090--7101, 18--24 Jul 2021.

\bibitem[Margulis(1982)]{margulis1982explicit}
Grigorii~A Margulis.
\newblock Explicit constructions of graphs without short cycles and low density
  codes.
\newblock \emph{Combinatorica}, 2\penalty0 (1):\penalty0 71--78, 1982.

\bibitem[Margulis(1988)]{margulis1988explicit}
Grigorii~Aleksandrovich Margulis.
\newblock Explicit group-theoretical constructions of combinatorial schemes and
  their application to the design of expanders and concentrators.
\newblock \emph{Problemy peredachi informatsii}, 24\penalty0 (1):\penalty0
  51--60, 1988.

\bibitem[Matou{\v{s}}ek(2013)]{matouvsek2013lecture}
Jir{\i} Matou{\v{s}}ek.
\newblock Lecture notes on metric embeddings.
\newblock Technical report, Technical report, ETH Z{\"u}rich, 2013.

\bibitem[Matou\v{s}ek(1996)]{Matouvsek_1996_Embeddings}
Ji\v{r}\'{\i} Matou\v{s}ek.
\newblock On the distortion required for embedding finite metric spaces into
  normed spaces.
\newblock \emph{Israel J. Math.}, 93:\penalty0 333--344, 1996.
\newblock ISSN 0021-2172.

\bibitem[Matou\v{s}ek(2002)]{Matouvsek_2002_LecturesDiscreteGeo}
Ji\v{r}\'{\i} Matou\v{s}ek.
\newblock Lectures on discrete geometry.
\newblock 212:\penalty0 xvi+481, 2002.

\bibitem[Mei and Montanari(2022)]{MeiMontanari2022}
Song Mei and Andrea Montanari.
\newblock The generalization error of random features regression: {P}recise
  asymptotics and the double descent curve.
\newblock \emph{Comm. Pure Appl. Math.}, 75\penalty0 (4):\penalty0 667--766,
  2022.

\bibitem[Mei et~al.(2022)Mei, Misiakiewicz, and
  Montanari]{MeiTheodorMontanari_2022_Generalization_RFMs}
Song Mei, Theodor Misiakiewicz, and Andrea Montanari.
\newblock Generalization error of random feature and kernel methods:
  {H}ypercontractivity and kernel matrix concentration.
\newblock \emph{Appl. Comput. Harmon. Anal.}, 59:\penalty0 3--84, 2022.

\bibitem[Mi et~al.(2018)Mi, Zhang, Gu, and Wang]{mi2018variational}
Liang Mi, Wen Zhang, Xianfeng Gu, and Yalin Wang.
\newblock Variational wasserstein clustering.
\newblock In \emph{Proceedings of the European Conference on Computer Vision
  (ECCV)}, pages 322--337, 2018.

\bibitem[Miolane et~al.(2020)Miolane, Guigui, Brigant, Mathe, Hou, Thanwerdas,
  Heyder, Peltre, Koep, Zaatiti, Hajri, Cabanes, Gerald, Chauchat, Shewmake,
  Brooks, Kainz, Donnat, Holmes, and Pennec]{GEOMSTATS_JMLR_2020}
Nina Miolane, Nicolas Guigui, Alice~Le Brigant, Johan Mathe, Benjamin Hou, Yann
  Thanwerdas, Stefan Heyder, Olivier Peltre, Niklas Koep, Hadi Zaatiti, Hatem
  Hajri, Yann Cabanes, Thomas Gerald, Paul Chauchat, Christian Shewmake, Daniel
  Brooks, Bernhard Kainz, Claire Donnat, Susan Holmes, and Xavier Pennec.
\newblock Geomstats: A python package for {R}iemannian geometry in machine
  learning.
\newblock \emph{Journal of Machine Learning Research}, 21\penalty0
  (223):\penalty0 1--9, 2020.
\newblock URL \url{http://jmlr.org/papers/v21/19-027.html}.

\bibitem[Munkres(2000)]{munkres2014topology}
James~R. Munkres.
\newblock Topology, 2000.
\newblock Second edition.

\bibitem[Munzner(1997)]{munzner1997h3}
Tamara Munzner.
\newblock H3: Laying out large directed graphs in 3d hyperbolic space.
\newblock In \emph{Proceedings of VIZ'97: Visualization Conference, Information
  Visualization Symposium and Parallel Rendering Symposium}, pages 2--10. IEEE,
  1997.

\bibitem[Muscoloni et~al.(2017)Muscoloni, Thomas, Ciucci, Bianconi, and
  Cannistraci]{muscoloni2017machine}
Alessandro Muscoloni, Josephine~Maria Thomas, Sara Ciucci, Ginestra Bianconi,
  and Carlo~Vittorio Cannistraci.
\newblock Machine learning meets complex networks via coalescent embedding in
  the hyperbolic space.
\newblock \emph{Nature communications}, 8\penalty0 (1):\penalty0 1--19, 2017.

\bibitem[Naor(2015)]{naor2015metric}
Assaf Naor.
\newblock Metric embeddings and lipschitz extensions.
\newblock \emph{Princeton University, Department of Mathematics}, 2015.

\bibitem[Naor(2018)]{Naor_SnapshotRibe_ICM_2018}
Assaf Naor.
\newblock Metric dimension reduction: a snapshot of the {R}ibe program.
\newblock pages 759--837, 2018.

\bibitem[Naor and Neiman(2012)]{naor2012assouad}
Assaf Naor and Ofer Neiman.
\newblock Assouad’s theorem with dimension independent of the snowflaking.
\newblock \emph{Revista Matematica Iberoamericana}, 28\penalty0 (4):\penalty0
  1123--1142, 2012.

\bibitem[Naor and Tao(2012)]{NaorTao_2012_Dvoretzky_ISJM}
Assaf Naor and Terence Tao.
\newblock Scale-oblivious metric fragmentation and the nonlinear {D}voretzky
  theorem.
\newblock \emph{Israel J. Math.}, 192\penalty0 (1):\penalty0 489--504, 2012.
\newblock ISSN 0021-2172.

\bibitem[Naor et~al.(2006)Naor, Peres, Schramm, and Sheffield]{naor2006markov}
Assaf Naor, Yuval Peres, Oded Schramm, and Scott Sheffield.
\newblock Markov chains in smooth banach spaces and gromov-hyperbolic metric
  spaces.
\newblock \emph{Duke Mathematical Journal}, 134\penalty0 (1):\penalty0
  165--197, 2006.

\bibitem[Nash(1954)]{Nash_1954_c1_Embeddings}
John Nash.
\newblock {$C^1$} isometric imbeddings.
\newblock \emph{Ann. of Math. (2)}, 60:\penalty0 383--396, 1954.

\bibitem[Neyshabur et~al.(2015)Neyshabur, Salakhutdinov, and
  Srebro]{neyshabur2015path}
Behnam Neyshabur, Russ~R Salakhutdinov, and Nati Srebro.
\newblock Path-sgd: Path-normalized optimization in deep neural networks.
\newblock \emph{Advances in neural information processing systems}, 28, 2015.

\bibitem[Nickel and Kiela(2017)]{nickel2017poincare}
Maximillian Nickel and Douwe Kiela.
\newblock Poincar{\'e} embeddings for learning hierarchical representations.
\newblock \emph{Advances in neural information processing systems}, 30, 2017.

\bibitem[Orlin(1988)]{orlin1988faster}
James Orlin.
\newblock A faster strongly polynomial minimum cost flow algorithm.
\newblock In \emph{Proceedings of the Twentieth annual ACM symposium on Theory
  of Computing}, pages 377--387, 1988.

\bibitem[Ostrovskii(2013)]{ostrovskii2013metric}
Mikhail~I. Ostrovskii.
\newblock Metric embeddings.
\newblock 49:\penalty0 xii+372, 2013.
\newblock Bilipschitz and coarse embeddings into Banach spaces.

\bibitem[Pansu(1989)]{PansuHeisenbergGroup_1989}
Pierre Pansu.
\newblock M\'{e}triques de {C}arnot-{C}arath\'{e}odory et quasiisom\'{e}tries
  des espaces sym\'{e}triques de rang un.
\newblock \emph{Ann. of Math. (2)}, 129\penalty0 (1):\penalty0 1--60, 1989.
\newblock ISSN 0003-486X.

\bibitem[Park et~al.(2020)Park, Yun, Lee, and Shin]{Park}
S.~Park, C.~Yun, J.~Lee, and J.~Shin.
\newblock Minimum width for universal approximation, 2020.

\bibitem[Park et~al.(2021)Park, Lee, Yun, and Shin]{Sublinear_Memorization}
Sejun Park, Jaeho Lee, Chulhee Yun, and Jinwoo Shin.
\newblock Provable memorization via deep neural networks using sub-linear
  parameters.
\newblock 134:\penalty0 3627--3661, 15--19 Aug 2021.

\bibitem[Pflug and Pichler(2015)]{pflug2015dynamic}
Georg~Ch. Pflug and Alois Pichler.
\newblock Dynamic generation of scenario trees.
\newblock \emph{Comput. Optim. Appl.}, 62\penalty0 (3):\penalty0 641--668,
  2015.

\bibitem[Puthawala et~al.(2022)Puthawala, Kothari, Lassas, Dokmani{\'c}, and
  de~Hoop]{puthawala2020globally}
Michael Puthawala, Konik Kothari, Matti Lassas, Ivan Dokmani{\'c}, and Maarten
  de~Hoop.
\newblock Globally injective {R}e{LU} networks.
\newblock \emph{Journal of Machine Learning Research}, 23:\penalty0 1--55,
  2022.

\bibitem[Ravasz and Barab{\'a}si(2003)]{ravasz2003hierarchical}
Erzs{\'e}bet Ravasz and Albert-L{\'a}szl{\'o} Barab{\'a}si.
\newblock Hierarchical organization in complex networks.
\newblock \emph{Physical review E}, 67\penalty0 (2):\penalty0 026112, 2003.

\bibitem[Robinson(2006{\natexlab{a}})]{RobinsonSphereNoFlat_2006}
P.~L. Robinson.
\newblock The sphere is not flat.
\newblock \emph{The American Mathematical Monthly}, 113\penalty0 (2):\penalty0
  171--173, 2006{\natexlab{a}}.
\newblock URL \url{http://www.jstor.org/stable/27641870}.

\bibitem[Robinson(2006{\natexlab{b}})]{robinson2006sphere}
PL~Robinson.
\newblock The sphere is not flat.
\newblock \emph{The American Mathematical Monthly}, 113\penalty0 (2):\penalty0
  171--173, 2006{\natexlab{b}}.

\bibitem[Samal et~al.(2018)Samal, Sreejith, Gu, Liu, Saucan, and
  Jost]{JostNature1}
Areejit Samal, R.~P. Sreejith, Jiao Gu, Shiping Liu, Emil Saucan, and
  J{\"u}rgen Jost.
\newblock Comparative analysis of two discretizations of {R}icci curvature for
  complex networks.
\newblock \emph{Scientific Reports}, 8\penalty0 (1):\penalty0 8650, 2018.

\bibitem[Sarkar(2011)]{sarkar2011low}
Rik Sarkar.
\newblock Low distortion delaunay embedding of trees in hyperbolic plane.
\newblock In \emph{International Symposium on Graph Drawing}, pages 355--366.
  Springer, 2011.

\bibitem[Shen et~al.(2021)Shen, Yang, and Zhang]{ShenYangZhang2021}
Zuowei Shen, Haizhao Yang, and Shijun Zhang.
\newblock Deep network with approximation error being reciprocal of width to
  power of square root of depth.
\newblock \emph{Neural Comput.}, 33\penalty0 (4):\penalty0 1005--1036, 2021.
\newblock ISSN 0899-7667.

\bibitem[Shimizu et~al.(2021)Shimizu, Mukuta, and
  Harada]{shimizu2021hyperbolic}
Ryohei Shimizu, YUSUKE Mukuta, and Tatsuya Harada.
\newblock Hyperbolic neural networks++.
\newblock In \emph{International Conference on Learning Representations}, 2021.

\bibitem[Shkarin(2004)]{shkarin2004isometric}
SA~Shkarin.
\newblock Isometric embedding of finite ultrametric spaces in banach spaces.
\newblock \emph{Topology and its Applications}, 142\penalty0 (1-3):\penalty0
  13--17, 2004.

\bibitem[Sontag(1997{\natexlab{a}})]{SontagShattering_1997}
Eduardo~D. Sontag.
\newblock Shattering all sets of `$k$’ points in ``general position”
  requires $(k — 1)/2$ parameters.
\newblock \emph{Neural Computation}, 9\penalty0 (2):\penalty0 337--348,
  1997{\natexlab{a}}.

\bibitem[Sontag(1997{\natexlab{b}})]{sontag1997shattering}
Eduardo~D Sontag.
\newblock Shattering all sets of ‘k’points in “general position”
  requires (k—1)/2 parameters.
\newblock \emph{Neural Computation}, 9\penalty0 (2):\penalty0 337--348,
  1997{\natexlab{b}}.

\bibitem[Suzuki(2019)]{suzuki2018adaptivity}
Taiji Suzuki.
\newblock Adaptivity of deep {R}e{LU} network for learning in besov and mixed
  smooth besov spaces: optimal rate and curse of dimensionality.
\newblock \emph{International Conference on Learning Representations}, 2019.

\bibitem[Tabaghi and Dokmani{\'c}(2020)]{tabaghi2020hyperbolic}
Puoya Tabaghi and Ivan Dokmani{\'c}.
\newblock Hyperbolic distance matrices.
\newblock In \emph{Proceedings of the 26th ACM SIGKDD International Conference
  on Knowledge Discovery \& Data Mining}, pages 1728--1738, 2020.

\bibitem[Tabaghi and Dokmani{\'c}(2021)]{tabaghi2021procrustes}
Puoya Tabaghi and Ivan Dokmani{\'c}.
\newblock On procrustes analysis in hyperbolic space.
\newblock \emph{IEEE Signal Processing Letters}, 28:\penalty0 1120--1124, 2021.

\bibitem[Tabaghi et~al.(2021)Tabaghi, Chien, Pan, Peng, and
  Milenkovi{\'c}]{tabaghi2021linear}
Puoya Tabaghi, Eli Chien, Chao Pan, Jianhao Peng, and Olgica Milenkovi{\'c}.
\newblock Linear classifiers in product space forms.
\newblock \emph{arXiv preprint arXiv:2102.10204}, 2021.

\bibitem[Vardi et~al.(2022)Vardi, Yehudai, and Shamir]{vardi2022on}
Gal Vardi, Gilad Yehudai, and Ohad Shamir.
\newblock On the optimal memorization power of {R}e{LU} neural networks.
\newblock \emph{International Conference on Learning Representations}, 2022.

\bibitem[Vaswani et~al.(2017)Vaswani, Shazeer, Parmar, Uszkoreit, Jones, Gomez,
  Kaiser, and Polosukhin]{vaswani2017attention}
Ashish Vaswani, Noam Shazeer, Niki Parmar, Jakob Uszkoreit, Llion Jones,
  Aidan~N Gomez, {\L}ukasz Kaiser, and Illia Polosukhin.
\newblock Attention is all you need.
\newblock \emph{Advances in neural information processing systems}, 30, 2017.

\bibitem[Veli\v{c}kovi\'{c} et~al.(2018)Veli\v{c}kovi\'{c}, Cucurull, Casanova,
  Romero, Li\`{o}, and Bengio]{velickovic2018graph}
Petar Veli\v{c}kovi\'{c}, Guillem Cucurull, Arantxa Casanova, Adriana Romero,
  Pietro Li\`{o}, and Yoshua Bengio.
\newblock Graph attention networks.
\newblock \emph{International Conference on Learning Representations}, 2018.

\bibitem[Vershynin(2020)]{Memorization_ReluThreshfold_2020_SIAM}
Roman Vershynin.
\newblock Memory capacity of neural networks with threshold and rectified
  linear unit activations.
\newblock \emph{SIAM Journal on Mathematics of Data Science}, 2\penalty0
  (4):\penalty0 1004--1033, 2020.

\bibitem[Vestfrid and
  Timan(1979)]{Vestfrid_SeparableUltrametricSpacesEmbedinHilbertSpace_1979}
I.~A. Vestfrid and A.~F. Timan.
\newblock A universality property of {H}ilbert spaces.
\newblock \emph{Dokl. Akad. Nauk SSSR}, 246\penalty0 (3):\penalty0 528--530,
  1979.
\newblock ISSN 0002-3264.

\bibitem[Vol'berg and
  Konyagin(1987)]{VolbergKonyagin1987Conversedoublingmeasures}
A.~L. Vol'berg and S.~V. Konyagin.
\newblock On measures with the doubling condition.
\newblock \emph{Izv. Akad. Nauk SSSR Ser. Mat.}, 51\penalty0 (3):\penalty0
  666--675, 1987.

\bibitem[Waserstein(1969)]{vaserstein1969markov_OG}
Leonid~Nisonovich Waserstein.
\newblock Markov processes over denumerable products of spaces, describing
  large systems of automata.
\newblock \emph{Problemy Peredachi Informatsii}, 5\penalty0 (3):\penalty0
  64--72, 1969.

\bibitem[Weiss(1984)]{weiss1984girths}
Alfred Weiss.
\newblock Girths of bipartite sextet graphs.
\newblock \emph{Combinatorica}, 4\penalty0 (2):\penalty0 241--245, 1984.

\bibitem[Yarotsky(2017)]{YAROTSKYSobolev}
Dmitry Yarotsky.
\newblock Error bounds for approximations with deep {R}e{LU} networks.
\newblock \emph{Neural Networks}, 94:\penalty0 103 -- 114, 2017.
\newblock ISSN 0893-6080.

\bibitem[Zheng et~al.(2019)Zheng, Meng, Zhang, Chen, Yu, and
  Liu]{Zheng_Meng_Zhang_Chen_Yu_Liu_2019}
Shuxin Zheng, Qi~Meng, Huishuai Zhang, Wei Chen, Nenghai Yu, and Tie-Yan Liu.
\newblock Capacity control of {R}e{LU} neural networks by basis-path norm.
\newblock \emph{Proceedings of the AAAI Conference on Artificial Intelligence},
  33\penalty0 (01):\penalty0 5925--5932, Jul. 2019.

\bibitem[Zhou(2020)]{zhou2020universality}
Ding-Xuan Zhou.
\newblock Universality of deep convolutional neural networks.
\newblock \emph{Applied and computational harmonic analysis}, 48\penalty0
  (2):\penalty0 787--794, 2020.

\bibitem[Zipf(1949)]{Zipf1949}
George~Kingsley Zipf.
\newblock Human behavior and the principle of least effort.
\newblock pages 1--334, 1949.

\end{thebibliography}
\end{document}